%% file: paper.tex
\documentclass[nohyperref]{article}

\usepackage{microtype}
\usepackage{graphicx}
\usepackage{subfigure}
\usepackage{booktabs} %

\usepackage{hyperref}

\usepackage[accepted]{icml2022}

\usepackage{amsmath}
\usepackage{amssymb}
\usepackage{mathtools}
\usepackage{amsthm}

\usepackage{xcolor}
\definecolor{vermillion}{RGB}{213,94,0}
\definecolor{skyblue}{RGB}{86,180,233}
\definecolor{bluishgreen}{RGB}{0,158,115}
\definecolor{yellow}{RGB}{240,228,66}
\definecolor{orange}{RGB}{230,159,0}
\usepackage[capitalize,noabbrev]{cleveref}

\theoremstyle{plain}
\newtheorem{theorem}{Theorem}[section]

\newtheorem{lemma}[theorem]{Lemma}
\newtheorem{corollary}[theorem]{Corollary}
\theoremstyle{definition}
\newtheorem{definition}[theorem]{Definition}
\newtheorem{assumption}[theorem]{Assumption}
\theoremstyle{remark}

\usepackage[textsize=tiny]{todonotes}

\allowdisplaybreaks
\usepackage{microtype}
\usepackage{graphicx}
\usepackage{subfigure}
\usepackage{booktabs} %
\usepackage{amsthm}
\usepackage{amsmath}
\usepackage{enumitem}
\usepackage{flushend}
\usepackage{amssymb}
\usepackage{hyperref}
\usepackage{thmtools,thm-restate}
\newcommand{\vdecomp}{{V_{\textrm{decomp}}}}

\crefname{theorem}{Theorem}{Theorems}
\crefname{cor}{Corollary}{Corollaries}
\crefname{assumption}{Assumption}{Assumptions}
\Crefname{lemma}{Lemma}{Lemmas}
\Crefname{alg}{Algorithm}{Algorithms}
\Crefname{claim}{Claim}{Claims}
\Crefname{observation}{Observation}{Observations}
\Crefname{invariant}{Invariant}{Invariants}
\Crefname{algorithm}{Algorithm}{Algorithms}
\usepackage{algorithm}

\newtheorem{claim}{Claim}

\newcommand{\pa}{\textrm{Pa}}
\newcommand{\mec}{\textrm{MEC}}
\newcommand{\oracle}{\textrm{Oracle}}
\newcommand{\CI}{\textrm{CI}}
\newcommand{\close}{{c_{\textrm{close}}}}
\newcommand{\csup}{{c_{\textrm{support}}}}
\newcommand{\clb}{{c_{\textrm{lb}}}}

\usepackage{todonotes}
\newcommand{\colorroot}{{\textrm{color-root}}}
\newcommand{\vinit}{{v_{\textrm{initial}}}}
\newcommand{\vplat}{{v_{\textrm{plateau}}}}
\newcommand{\vsur}{{v_{\textrm{surplus}}}}

\newcommand{\overvinit}{{\overline{v}_{\textrm{initial}}}}
\newcommand{\overvplat}{{\overline{v}_{\textrm{plateau}}}}

\newcommand{\overvhelp}{{\overline{v}_{\textrm{helpful}}}}
\newcommand{\relate}{{\textrm{related}}}
\newcommand{\colored}{{\textrm{color}}}
\newcommand{\ord}{{\textrm{order}}}

\usepackage{thmtools,thm-restate}

\newcommand{\indep}{\perp \!\!\! \perp}
\newcommand{\Xsrc}{{X_{\textrm{src}}}}

\icmltitlerunning{Entropic Causal Inference: Graph Identifiability}

\begin{document}

\twocolumn[
\icmltitle{Entropic Causal Inference: Graph Identifiability}

\begin{icmlauthorlist}
\icmlauthor{Spencer Compton}{mit,mitibm}
\icmlauthor{Kristjan Greenewald}{mitibm,ibm}
\icmlauthor{Dmitriy Katz}{mitibm,ibm}
\icmlauthor{Murat Kocaoglu}{purdue}
\end{icmlauthorlist}

\icmlaffiliation{mit}{Massachusetts Institute of Technology, Cambridge, USA}
\icmlaffiliation{mitibm}{MIT-IBM Watson AI Lab, Cambridge, USA}
\icmlaffiliation{ibm}{IBM Research, Cambridge, USA}
\icmlaffiliation{purdue}{Purdue University, West Lafayette, USA}
\icmlcorrespondingauthor{Spencer Compton}{scompton@mit.edu}

\icmlkeywords{Information Theory, Causality}

\vskip 0.3in
]

\printAffiliationsAndNotice{} %

\begin{abstract}
Entropic causal inference is a recent framework for learning the causal graph between two variables from observational data by finding the information-theoretically simplest structural explanation of the data, i.e., the model with smallest entropy. In our work, we first extend the causal graph identifiability result in the two-variable setting under relaxed assumptions. We then show the first identifiability result using the entropic approach for learning causal graphs with more than two nodes. Our approach utilizes the property that ancestrality between a source node and its descendants can be determined using the bivariate entropic tests. %
We provide a sound sequential peeling algorithm for general graphs that relies on this property. We also propose a heuristic algorithm for small graphs that shows strong empirical performance. We rigorously evaluate the performance of our algorithms on synthetic data generated from a variety of models, observing improvement over prior work. Finally we test our algorithms on real-world datasets.

\end{abstract}

\section{Introduction}
Causal reasoning is %
essential for high-quality decision-making, as, for instance, it improves interpretability and enables counterfactual reasoning \cite{athey2015machine,morgan2015counterfactuals,moraffah2020causal}. By learning the relationships between causes and effects, we can predict how various interventions would affect a system. %
Advances in causality enable us to better answer questions such as \emph{``Why does this phenomenon occur in the system?''} or \emph{``What could happen if the system were perturbed in this particular way?''}
Moreover, causal inference methods are being utilized to tackle key challenges for reliability of ML systems, such as domain adaptation \cite{magliacane2017domain,zhang2020domain} and generalization (e.g. via causal transportability or  imputation) \cite{bareinboim2012transportability, pearl2014external,squires2020causal}.

Structural causal models (SCMs) represent relationships in a system of random variables \cite{pearl2009causality}. In particular, each variable is modeled with a structural equation that characterizes how the variable is realized. Causal graphs are directed acyclic graphs (DAGs) that are used to represent such systems, where nodes and edges correspond to variables and the causal relations between these variables, respectively. A variable's structural equation is a function of the variable's corresponding node's  parents in the graph. 

Learning such causal graphs can be done through a series of interventions. However, in many settings it is not possible to perform such interventions. 
A large amount of literature has focused on learning the causal graph from observational data with additional ``faithfulness assumptions'' ~\cite{spirtes2000causation},
though in general it is impossible to fully learn the causal graph without stronger assumptions on the generative model. A variety of stronger assumptions and corresponding methodologies exist in the literature~\cite{shimizu2006linear,hoyer2008nonlinear,loh2014high,peters2014identifiability}. Most of these methods, however, are limited to continuous variables %
and thus cannot handle categorical data, especially in the multivariate setting.

A recent framework explicitly designed to handle categorical data is \emph{entropic causal inference} \cite{kocaoglu2017entropic,kocaoglu2017entropicISIT,compton2021entropic}. At a high level, the underlying assumption of this approach is that true causal mechanisms in nature are often ``simple,'' taking inspiration from the Occam's razor principle. The authors adopt an information-theoretic realization of this principle by using ``entropy" to measure the complexity of a causal model. As we further explore in this work, entropic causal inference provides a means to measure the amount of randomness a generative model would require to produce an observed distribution. As Occam's razor prefers simpler explanations, entropic causal inference prefers generative models with small randomness. We do not expect following this preference to always lead to the discovery of true causal relationships (just as one does not expect a simpler explanation to be always be the correct one), but view this as a guiding intuition that mirrors nature and experimental observations. Our experiments on semi-synthetic data demonstrates that the low-entropy assumption indeed holds in certain settings.  

Previously, the framework was applied to discovering the causal direction between two random variables given that the amount of randomness in the true causal relationship is small. %
We focus on extending this framework to learn larger causal graphs instead of just cause-effect pairs. Suppose the observed variables have $n$ states. Our contributions follow:
\begin{enumerate}[parsep=0pt]
    \item We show pairwise identifiability with strictly relaxed assumptions compared to the previously known results. We enable learning the causal graph $X\rightarrow Y$ from observational data even when $(i)$ the cause variable $X$ has low entropy of $o(\log(n))$ and $(ii)$ the exogenous noise has non-constant entropy, i.e., $\mathcal{O}(1)\ll H(E)=o(\log\log(n))$.
    \item We show the first identifiability result for causal graphs with more than two observed variables, with a new peeling algorithm for general graphs. 
    \item  We propose a heuristic algorithm that 
    searches over all DAGs and outputs the one that requires the minimum entropy to fit to the observed distribution.
    \item We experimentally evaluate our algorithms and show that entropic approaches outperform the discrete additive noise models in synthetic data. We also apply our algorithms on semi-synthetic data using the \emph{bnlearn}\footnote{\url{https://www.bnlearn.com/bnrepository/}} repository and demonstrate the applicability of low-entropy assumptions and the proposed method.
\end{enumerate}
\section{Related Work}
Learning causal graphs from observational data has been studied extensively in the case of continuous variables. \citet{loh2014high} proposes an algorithm for learning linear structural causal models when the error variance is known. Similarly, \citet{peters2014identifiability}
show that linear models with Gaussian noise become identifiable if the noise variance is the same for all variables. A more general modeling assumption is the additive noise model (ANM). In \cite{shimizu2006linear}, the authors show that for almost all linear causal models, the causal graph is identifiable if the additive exogenous noise is non-Gaussian. 

In the case of discrete and/or categorical variables, causal discovery literature is much more sparse. \cite{cai2018causal} introduces a method for categorical cause-effect pairs when there exists a hidden intermediate representation that is compact. In \cite{daniusis2010inferring,janzing2015justifying}, the authors propose using an information-geometric approach called IGCI that is based on independence of cause and the causal mechanism. However IGCI can provably recover the causal direction only in the case of deterministic relations. An extension of additive noise models to discrete data is done in \cite{peters2011causal} where identifiability is shown between two variables. The authors also propose using the regression-based algorithm of \cite{mooij2009regression} (which made continuous domain ANM applicable to arbitrary graphs) for the discrete setting as well. Without specific assumptions on the graph and the generative mechanisms, this is a heuristic algorithm, i.e., identifiability in polynomial time is not guaranteed by discrete ANM on graphs with more than two nodes. 

One related idea is to use Kolmogorov complexity to determine the simplest causal model~\cite{janzing2010causal}. Minimum-description length has been used as a substitute for Kolmogorov complexity (which is not computable) in a series of follow-up papers~\cite{budhathoki2017mdl,marx2021formally}. Our extension of entropic causal inference to graphs can be seen as an information-theoretic realization of this promise, where the complexity of the causal model is captured by its entropy. Other information-theoretic concepts such as interaction information \cite{ghassami2017interaction} and directed information \cite{etesami2014directed} have also been studied in the context of causality.

\section{Background and Notation}
\label{section:background}
\textbf{Causal Graphs and Learning:} 
Consider a %
causal system where each variable is generated as a function of a subset of the rest of the observed variables and some additional randomness. Such systems are modeled by structural equations and are called structural causal models (SCMs). Let $X_1, X_2, \dots, X_{\lvert V\rvert}$ be the set of observed variables. Accordingly, there exists functions $f_i$ and exogenous noise terms $E_i$ such that $X_i = f_i(\pa_i, E_i)$. This equation in a causal system should be understood as an assignment operator since changing $\pa_i$ affects $X_i$ whereas changing $X_i$ does not affect $\pa_i$. We say the set of variables $\pa_i$ \emph{cause} $X_i$. A directed acyclic graph (DAG) can be used to summarize these causal relations, which is called the \emph{causal graph}. We denote the causal graph by $G = (V, \mathcal{E})$ where $V$ is the set of observed nodes and $\mathcal{E}$ is the set of directed edges. There are $\lvert V\rvert$ nodes, $X_1, X_2, \dots, X_{\lvert V\rvert}$, where each $X_i$ corresponds to an observable random variable. Edges are constructed by adding a directed arrow from every node in the set $\pa_i$ to $X_i$ for all $i$. $\pa_i$ then becomes the set of parents of $X_i$ in $G$. We assume causal sufficiency, i.e., that there are no unobserved confounders, and that there is no selection bias. Under these assumptions, %
$\pa_i \indep E_i$. %
Additionally, for simplicity of presentation we denote the number of states of all variables as $n$ (i.e. $|X_i|=n$ for all $i$). Note that our proofs do not require each observed variable to strictly have the same number of states; we merely need them scale together, i.e. if $X_1\in [n_1]$ and $X_2\in [n_2]$ then $\frac{n_1}{n_2} = \Theta(1)$. All big-o notation in the paper is relative to $n$.
Our goal is to infer the directed causal graph from the observed joint distribution $p(X_1,X_2,\hdots, X_{\lvert V\rvert})$ using the assumptions of the entropic causality framework as needed. 

Even without making any parametric assumptions, we can learn some properties of the graph  %
from purely observational data. Algorithms relying on conditional independence tests (such as the PC or IC algorithms \cite{spirtes2000causation,pearl2009causality}) can identify the \emph{Markov equivalence class} (MEC) of $G$, i.e. the set of graphs that produce distributions with the exact same set of conditional independence relations. 
Moreover, a graph's Markov equivalence class uniquely determines its \emph{skeleton} (the set of edges, ignoring orientation) and \emph{unshielded colliders} (the induced subgraphs of the form $X \rightarrow Z \leftarrow Y$). A Markov equivalence class is summarized by %
a mixed graph called the \emph{essential graph}, which has the same skeleton and contains a directed edge if   %
all graphs in the equivalence class orient the edge in the same direction. All other edges are undirected. The problem of determining the true causal graph from observational data thus reduces to orienting these remaining undirected edges, given enough samples to perform conditional independence tests reliably. 

\textbf{Entropic Causality Framework: }
Without interventional data, one needs additional assumptions to refine the graph structure further than the equivalence class. %
The key assumption of the entropic causality framework is that, in nature, true causal models are often ``simple.'' In information-theoretic terms, this is formalized as the entropy of exogenous variables often being small. Previous work has shown guarantees for identifying the direction between a causal pair $X,Y$ where $Y=f(X,E), X\indep E$ for some exogenous variable $E$ from observational data. %
The work of \cite{kocaoglu2017entropic} showed that when the support size of the exogenous variable (i.e. the Renyi-0 entropy $H_0(E) = |E|$) is small, with probability $1$ it is impossible to factor the model in the reverse direction (as $X=g(Y,\tilde{E})$) with an exogenous variable with small support size (i.e. $|\tilde{E}|$ must be large). Thus, one can identify the causal pair direction by fitting the smallest cardinality exogenous variable in both directions and checking which direction enables the smaller cardinality. \citet{kocaoglu2017entropic} conjectured that this approach also would work well for Shannon entropy. \citet{compton2021entropic} resolved this conjecture, showing identifiability for causal pairs under particular generative assumptions. 

\begin{definition}[$(\alpha,\beta)$-support]
A discrete random variable $X$ is said to have $(\alpha,\beta)$-support if at least $\alpha$ states of $X$ have probability of at least $\beta$.
\end{definition}
\cite{compton2021entropic} assumes that the cause variable $X$ has $(\Omega(n),\Omega(\frac{1}{n}))$-support and that the Shannon entropy of the exogenous variable (i.e. $H(E) = H_1(E)$) is small. In this paper, we say an event holds ``with high probability'' if the probability of the event not holding is bounded by $O\left(\frac{1}{n^{\alpha}}\right)$ for any constant $\alpha > 0$. \cite{compton2021entropic} showed that when $H(E)=O(1)$, $H(\tilde{E})=\Omega(\log(\log(n)))$ with high probability. The high probability statement is with respect to the selection of the function $f$, i.e., for all but a vanishing (in $n$) fraction of functions $f$, identifiability holds. Moreover, they showed that this approach was robust to only having a polynomial number of samples, whereas the result of \cite{kocaoglu2017entropic} that assumed small $|E|$ required knowing the exact joint distribution, e.g. from an oracle or infinite samples. 

Algorithmically, one can provably orient causal pairs under the assumptions of \cite{compton2021entropic} by comparing the minimum entropy exogenous variable needed to factor the pair in both directions (i.e. comparing the minimum $H(E)$ for which there exists a function $f$ and $E \indep X$ such that $Y=f(X,E)$, and the analogous quantity minimizing $H(\tilde{E})$). Finding this minimum entropy exogenous variable is an optimization problem equivalent to the \emph{minimum-entropy coupling problem} for the conditionals, specifically, the minimum $H(E)$ in the direction $X\rightarrow Y$ is the same as the minimum-entropy coupling for $[(Y|X=i)],  \forall i\in[n]$ \cite{kocaoglu2017entropic,cicalese2017find,painsky2019innovation}. Accordingly, we denote the entropy of the minimum-entropy coupling for a variable $X$ conditioned on a set $S$ as $\mec(X\vert S)$. \citet{compton2021entropic} showed $\mec(Y\vert X) < \mec(X \vert Y)$ with high probability.

\section{Tightening the Entropic Identifiability Result for Cause-Effect Pairs}
In this work, we leverage results for the bivariate entropic causality setting to learn general graphs. Theorem 1 of \cite{compton2021entropic} provides identifiability guarantees in the bivariate setting. 
However, the assumptions of their theorem are not general enough to imply an identifiability result on graphs with more than $2$ nodes. 
Specifically, a fundamental challenge in applying bivariate causality to discover each edge in a larger graph is confounding due to the other variables, i.e., when one considers a pair of variables, the remaining variables act as confounders. These confounders cannot be controlled for since we do not know the causal graph and conditioning on other variables unknowingly creates additional dependencies. One natural approach to handle confounding is to recursively discover source nodes by conditioning on the common causes that are discovered so far in the graph. This idea will form the basis for our peeling algorithm to be proposed in Section \ref{subsection:peeling}. %
We are interested in learning graphs where the exogenous variable for every node has small entropy (in particular, $H(E_i)=o(\log(\log(n)))$). When conditioning on the source nodes, some nodes $X$ (e.g. the children of the source nodes) will thus have conditional entropies of order $H(X|\mathrm{sources})=o(\log(\log(n)))$ since for the children of source nodes, the only remaining randomness on $X$ will be due to the low-entropy exogenous variable. This creates problems when attempting to orient edges connected to these variables conditioned on the source nodes.
Specifically, Theorem 1 of \citet{compton2021entropic} requires the cause variable $X$ to have $(\Omega(n),\Omega(\frac{1}{n}))$-support which enforces $H(X) = \Omega(\log(n))$ -- and this is not satisfied for the above nodes with $o(\log(\log(n)))$ entropy. 

In the following bivariate result, we instead only require $(\Omega(n),\Omega(\frac{1}{n \log (n)}))$-support, and simultaneously relax the exogenous variable constraint from $H(E) = O(1)$ to $H(E) = o(\log(\log(n)))$. This condition can be satisfied for $X$ with $H(X) = O(1)$ as needed.
\begin{theorem}
Consider the SCM $Y=f(X,E), X\indep E$, where $X,Y\!\in\![n], E\!\in\! [m]$. Suppose $E$ is any random variable with entropy $H(E)=o(\log(\log(n)))$. Let $X$ have $(\Omega(n),\Omega(\frac{1}{n \log(n)}))$-support. Let $f$ be sampled uniformly randomly from all mappings $f\!:\![n]\!\times\![m]\!\rightarrow\! [n]$. Suppose $n$ is sufficiently large. Then, with high probability, any $\tilde{E}$ that satisfies $X \!= \!g(Y,\tilde{E}),\tilde{E}\indep Y$ for some $g$, entails $H(\tilde{E})\!\geq\! \Omega( \log(\log(n)))$. 
\label{theorem:pairwise}
\end{theorem}
While interesting in its own right, we apply this tightened bivariate identifiability result to the general graph case in \cref{section:learning-graphs}.
Note that the assumption of a uniformly random $f$ (also used in \cite{compton2021entropic}, \cite{kocaoglu2017entropic}) is not meant as a description of how nature generates causal functions, but as the least-restrictive option for putting a measure on the space of possible functions so that high-probability statements can be made rigorously. Theorem \ref{theorem:pairwise} can be immediately adapted to any alternative distribution on the space of $f$ that does not assign any individual value of $f$ probability mass more than $n^{c'}$ times the probability mass assigned by the uniform distribution, for some constant $c'$. %

\paragraph{Proof overview for Theorem \ref{theorem:pairwise}.} Here we provide the intuition behind the proof strategy, the full proof is given in Appendix \ref{app:pairwise}. It is simple to show that the minimum entropy required to fit the function in the incorrect direction, $H(\tilde{E})$, is lower-bounded as $H(\tilde{E}) \ge \max_{y} H(X | Y=y)$. The overarching goal of our proof method is then to show that there exists a state $y'$ of $Y$ such that $H(X | Y=y') = \Omega(\log(\log(n)))$. 

To accomplish this, we start by showing that the $(\Omega(n),\Omega(\frac{1}{n \log(n) }))$-support of $X$ implies existence of a subset $S$ of $\Omega(n)$ states of $X$ that each have probability $\Omega(\frac{1}{n \log(n)})$ and are all relatively close in probability to each other. We call this subset $S$, the \emph{plateau} states. If one envisions them as adjacent in the PMF of $X$, these states would have similar heights and thus look like a plateau. 

Now, we conceptualize the realization of $f$ as a balls-and-bins game, where each element of $X \times E$ (a ball) is mapped i.i.d. uniformly randomly to a state of $Y$ (a bin). Using balls-and-bins arguments, it is our hope to show that there is a bin that receives $\Omega (\frac{\log(n)}{\log(\log(n))})$ \emph{plateau balls} of the form $(X \in S, E=e_1)$, where $e_1$ is the most probable state of $E$, and that this will cause the corresponding conditional distribution to have large entropy. The primary intuition is that a bin receiving many plateau balls would cause the corresponding conditional distribution to have many plateau states that all have near-uniform probabilities, and this near-uniform subset of the conditional distribution would contribute a significant fraction of the probability mass to guarantee that its entropy is large. With the stronger assumptions on $(\alpha,\beta)$-support by \cite{compton2021entropic}, this proof method suffices. However, as we are assuming a weaker notion of $(\alpha,\beta)$-support, it is not clear that the plateau balls would make up a significant fraction of the conditional's mass to guarantee large entropy.

In a sense, the plateau balls are probability masses that are ``helping'' us make some conditional entropy large. The proof of \cite{compton2021entropic} takes the perspective that all remaining mass from non-plateau states are ``hurting'' our effort to make a conditional distribution with large entropy. 
To accommodate our relaxed assumptions, 
we take a more nuanced perspective on \emph{helpful} and \emph{hurtful} mass. Consider a non-plateau state $x$ of $X$ that contributes a small amount of mass towards the conditional distribution corresponding to a state $y$ of $Y$. With the perspective of \cite{compton2021entropic}, this would be viewed as hurtful mass because it is from a non-plateau state of $X$. But intuitively, in terms of its contribution to $H(X|Y=y)$, it does not matter whether $x$ is a plateau state or not. Through careful analysis, we can show that if $P(X=x | Y=y)$ is small then they are not ``too hurtful.''
We follow this intuition to make a new definition of the good mass, where we set a threshold $\mathcal{T}$, define the first $\mathcal{T}$ mass we receive from a non-plateau state of $X$ as \emph{helpful} mass for the state of $Y$, and the surplus beyond $\mathcal{T}$ from the non-plateau state of $X$ as \emph{hurtful} mass for the state of $Y$. As before, all mass from plateau states will be helpful. With this new perspective and a careful analysis, we show that there is a state $y'$ that receives many plateau balls, and has much more helpful mass than hurtful mass. This then enables us to show that $H(X|Y=y')$ is large, proving the theorem.

\section{Learning Graphs via Entropic Causality}
\label{section:learning-graphs}

Now, we focus on how to leverage the capability of correctly orienting causal pairs to learn causal graphs \emph{exactly}. 
In comparison, traditional structure learning methods only learn the Markov equivalence class of graphs from observational data. For example, given the line graph $X \rightarrow Y \rightarrow Z$, such methods would deduce the true graph is either $X \rightarrow Y \rightarrow Z$ or $X \leftarrow Y \leftarrow Z$, but not that it is exactly $X \rightarrow Y \rightarrow Z$.

As was discussed in \cref{section:background}, learning the entire graph can be reduced to correctly orienting each edge in the skeleton. However, we cannot naively use a pairwise algorithm, as the rest of the observed variables can act as confounders. We examine how different pairwise oracles can enable us to characterize the value of using minimum entropy couplings to learn causal graphs. One example of a natural-feeling oracle is one that can correctly orient any edges that have no active confounding. Such an oracle enables learning of directed trees and complete graphs. 
However, it cannot be used to learn all general graphs (see \cref{section:non-identifiable} for an example). We propose an alternative oracle, that can distinguish between a source node and any node it can reach:

\begin{definition}[Source-pathwise oracle]
A source-pathwise oracle for a DAG $G$ always orients $A \rightarrow B$ if $A$ is a source and there exists a directed path from $A$ to $B$ in $G$.
\end{definition}

Let us formalize our entropic method for causal pairs as the following oracle:

\begin{definition}[MEC oracle]
A minimum entropy coupling (MEC) oracle returns $X\rightarrow Y$ if $\mec(Y\vert X)<\mec(X\vert Y)$ and $X\leftarrow Y$ otherwise, given the joint distribution $p(X,Y)$.
\end{definition}

We aim to show that our MEC oracle is a source-pathwise oracle for graphs with the following assumptions:

\begin{assumption}[Low-entropy assumption]
\label{assumption:graph}
Consider an SCM where $X_i = f_i(\pa_i,E_i), \pa_i \indep E_i, \forall i$, where $X_i \in [n], E_i \in [m]$. Suppose $|V|=O(1)$, $H(E_i)=o(\log(\log(n)))$ and $E_i$ has $(\Omega(n), \Omega(\frac{1}{n \log(n)}))$-support for all $i$, and $f_i$ are sampled uniformly randomly from all mappings $f_i\!:\![m]\!\times\![n]^{|\pa_i|}\!\rightarrow\! [n]$.
\end{assumption}

We are now ready to show the main result of our paper. We show that, under certain generative model assumptions, applying entropic causality on pairs of observed variables acts as a source-pathwise oracle for DAGs:

\begin{theorem}
\label{theorem:dag-oracle}
For any SCM under \cref{assumption:graph}, the MEC oracle is a source-pathwise oracle for the causal graph 
with high probability for sufficiently large $n$.
\end{theorem}
Characterizing entropic causality as a source-pathwise oracle enables us to identify the true causal graph for general graphs. 
We outline the key intuitions of our proof:

\input{tikz_graph}
\paragraph{Proof overview for Theorem \ref{theorem:dag-oracle}.} Suppose $\Xsrc$ is a source and $Y$ is a node such that there is a path from $\Xsrc$ to $Y$. To show the MEC oracle is a source-pathwise oracle, we show that $\mec(\Xsrc|Y) > \mec(Y | \Xsrc)$. As in \cref{theorem:pairwise}, we will accomplish this by showing there is a state $y'$ of $Y$ such that $H(\Xsrc|Y=y') = \Omega(\log(\log(n)))$. 

We begin by conceptualizing the realization of all $f_i$ as a balls-and-bins game. Every node $X_i$ is a uniformly random function $f_i$ of $\pa(X_i) \cup E_i$. Let us define each ball as the concatenation of $X$ and all $E_i$ other than $E_{\textrm{src}}$. More formally, we denote each ball as $(X=x, E_1=e_1, \dots, E_{\textrm{src}-1}=e_{\textrm{src}-1}, E_{\textrm{src}+1}=e_{\textrm{src}+1}, \dots, E_{|V|}=e_{|V|})$, and each ball has a corresponding probability mass of $P(X=x) \times \Pi_{i \ne \textrm{src}} P(E_{i}=e_i)$. To view the realization of $f_i$ as a balls-and-bins game, we consider the nodes in an arbitrary topological order for the graph. When we process a node $X_i$, we group balls according to their configuration of $(\pa(X_i) \cup E_i)$. This is because balls with the same configuration correspond to the same cell of the function $f_i$. For each group of balls that all share the same configuration, we uniformly randomly sample a state of $X_i$ to assign all the balls in the group. This is essentially realizing one cell of $f_i$. Groups are assigned independently of other groups. In this sense, each realization of $f_i$ is a balls-and-bins game where we group balls by their configuration, and throw them together into states of $X_i$ (bins). 

Let us define plateau balls as those who have a plateau state of $\Xsrc$ and have the most probable state of $E_i$ for every $i \ne \textrm{src}$. As was done in \cref{theorem:pairwise}, our goal is to show that there will be a state $y'$ of $Y$ such that $y'$ receives many plateau balls and much more helpful mass than hurtful mass. However, it is not immediately clear how to show this in the graph setting. For intuition, we explore two special cases.

Consider the case of a line graph (Figure \ref{fig:line}). For simplicity of this proof overview, assume all $X_i$ other than $\Xsrc$ are deterministic functions of their parents (i.e., $H(E_i)=0$). Using techniques similar to \cref{theorem:pairwise}, we can show there are many bins of $X_2$ that receive many plateau balls and much more helpful mass than hurtful mass. Moreover, we can then use similar techniques to show a non-negligible proportion of those bins will have their corresponding balls mapped together to a bin of $X_3$ where it does not encounter much hurtful mass. We can repeat this argument again to show some of these desirable bins ``survive'' from $X_3$ to $Y$. While only a scalingly small fraction of these desirable bins ``survive'' each level, this still ensures the survival of at least one bin if the number of vertices is constant. This will accomplish our goal of having a state $y'$ of $Y$ with many plateau balls and much more helpful mass than hurtful mass.

On the other hand, consider the case of a diamond graph in Figure \ref{fig:diamond}. Again, assume for simplicity that all $X_i$ other than $\Xsrc$ are deterministic functions of their parents. Note that when we realize $f_Y$, two balls are always mapped independently unless they share the same configuration of $\pa(Y) = \{X_2, X_3\}$. We observe that $X_2$ and $X_3$ are both independent deterministic functions of $\Xsrc$. Accordingly, the probability of two particular states $x,x' \in \Xsrc$ satisfying $f_2(x)=f_2(x')$ and $f_3(x)=f_3(x')$ is equal to $\frac{1}{n^2}$. Therefore, almost all pairs of balls are mapped to $Y$ from $\Xsrc$ independently. Since everything is almost-independently mapped to $Y$, we can treat it like a bivariate problem and use techniques similar to \cref{theorem:pairwise}.

We are able to prove correctness for both of these graphs, but we do so in ways that are essentially opposite. For the line graph, we utilize strong dependence as bins with desired properties ``survive'' throughout the graph. For the diamond graph, we utilize strong independence as balls are all mapped to $Y$ essentially independently. To combine the intuitions of these two cases into a more general proof, we introduce the Random Function Graph Decomposition: %
\input{peeling}

\input{main_figures}
\begin{restatable}[Random Function Graph Decomposition]{definition}{defdecomposition}
\label{def:rfgd}
Given a DAG and a pair of nodes $(\Xsrc,Y)$, Random Function Graph Decomposition colors the nodes iteratively following any topological order of the nodes as follows:
\begin{enumerate}
    \item Color the node with a new color if $\Xsrc$ is a parent of the node or if the node has parents of different colors.
    \item Color the node with the color of its parents if all of the node's parents have the same color.
\end{enumerate}
\end{restatable}

Using the Random Function Graph Decomposition, we claim that when a node is assigned a new color as in step $1$, we utilize independence as in the diamond graph (Figure \ref{fig:diamond}), and when a node inherits its color as in step $2$ we utilize dependence as in the line graph in Figure \ref{fig:line}. We illustrate Figure \ref{fig:hall} as an example. With a careful analysis, we utilize these intuitions to prove the MEC oracle is a source-pathwise oracle with high probability.

\subsection{Peeling Algorithm for Learning Graphs}
\label{subsection:peeling}
In the previous section, we have shown how entropic causality can be used as a source-pathwise oracle. Next, we show how to learn general graphs with a source-pathwise oracle. Our algorithm will iteratively determine the graph's sources, condition on the discovered sources, determine the graph's sources after conditioning, and so on. Doing this will enable us to find a valid lexicographical ordering of the graph. Given a lexicographical ordering, we can learn the skeleton with $O(n^2)$ conditional independence tests. %

Now, we outline how we iteratively find the sources. In each stage, we consider all the remaining nodes as candidate sources. It is our goal to remove all non-sources from our set of candidates. To do this, we iterate over all pairs of candidates and do a conditional independence test \emph{conditioned on the sources that are found so far}. If the pair is conditionally independent, we do nothing. We note that this will never happen for a pair where one node is a true source and the other node is reachable from the source through a directed path: Conditioning on previously found sources cannot d-separate such paths. Otherwise, the pair is conditionally dependent. We then use the source-pathwise oracle to orient between the two nodes, and eliminate the sink node of the orientation as a candidate (i.e., if we orient $A \rightarrow B$, we eliminate $B$ as a candidate source).

\begin{figure*}[ht!]
\vskip 0.2in
\begin{center}
\subfigure[]{\label{fig:3_node_triangle_100_HES}\includegraphics[width=0.32\textwidth]{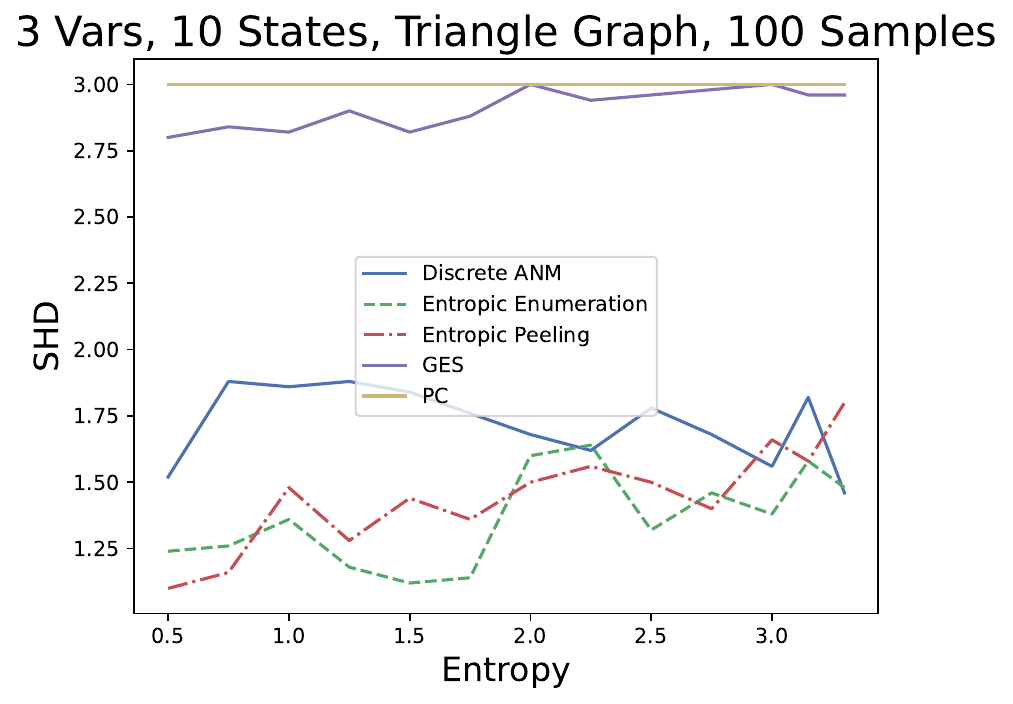}}
\subfigure[]{\label{fig:3_node_triangle_1000_HES}\includegraphics[width=0.32\textwidth]{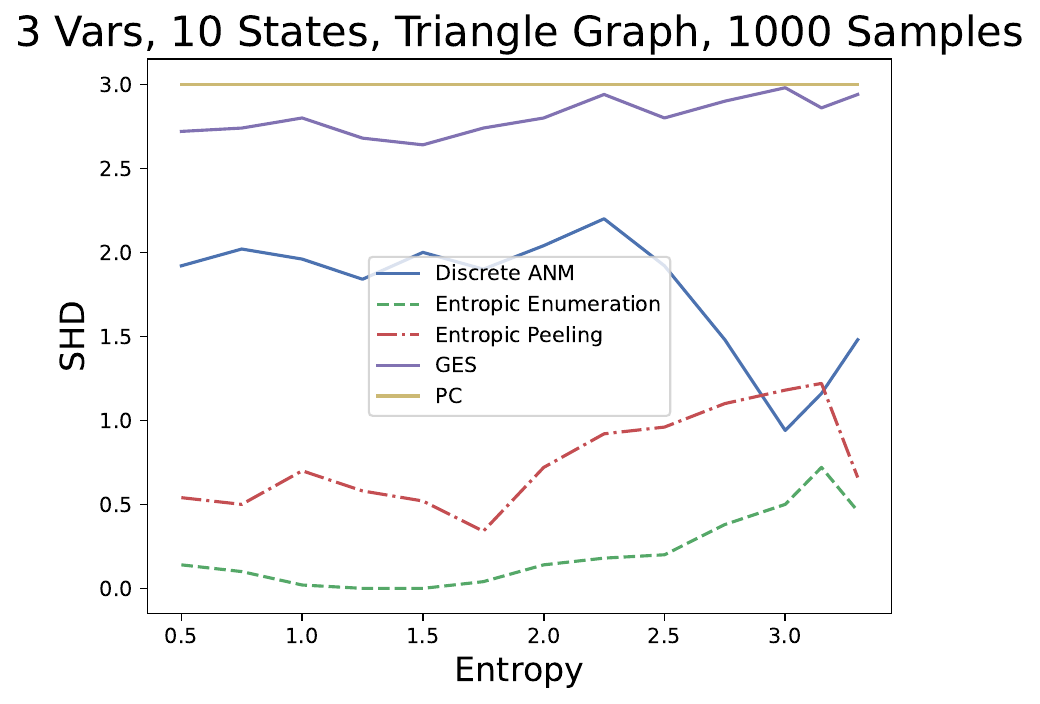}}
\subfigure[]{\label{fig:3_node_triangle_50000_HES}\includegraphics[width=0.32\textwidth]{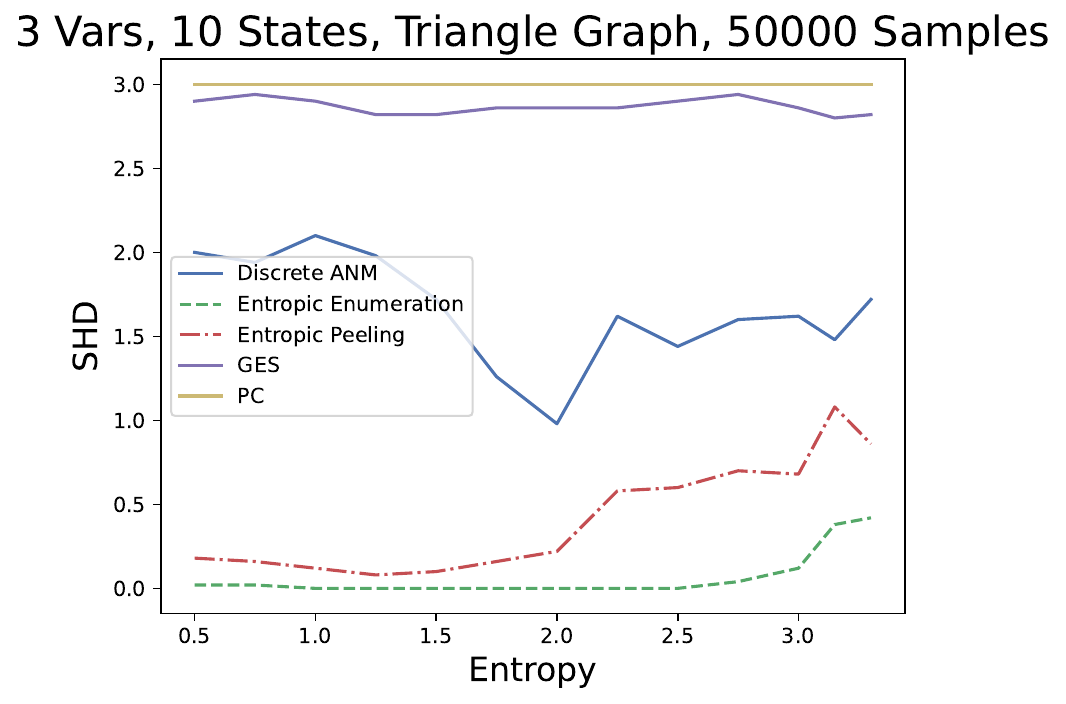}}
\caption{Performance of methods in the unconstrained setting in the triangle graph $X\rightarrow Y\rightarrow Z, X\rightarrow Z$: $50$ datasets are sampled for each configuration from the unconstrained model $X=f(\pa_X,E_X)$. The $x-$axis shows entropy of the exogenous noise. The exogenous noise of the first variable is fixed to be large ($\approx 3.3$ bits), hence it is a high entropy source (HES). Entropic methods consistently outperform the ANM algorithm in almost all regimes.} %
\label{fig:3_node_triangle_entropic_HES}
\end{center}
\vskip -0.2in
\end{figure*}

Suppose two nodes are dependent conditioned on the past sources. Then either the pair contains a source node and a descendant of the source node, or it contains two non-source nodes. In the former case when the pair contains a source node 
the source-pathwise oracle will always orient correctly and the non-source node will be eliminated. In the latter case when the pair are two non-sources we can safely eliminate either as a source candidate and accordingly oracle output is irrelevant. By the end of this elimination process, we can show that only true sources will remain as candidates in each step, which enables us to obtain a valid lexicographical ordering, and thus learn the causal graph.
We summarize this procedure as \cref{alg:general}. The following theorem shows the correctness of \cref{alg:general} given a source-pathwise oracle:
\begin{theorem}
\label{theorem:graph-alg}
\cref{alg:general} learns any faithful causal graph $D=(V,E)$ with $\mathcal{O}(\lvert V\rvert^2)$ calls to a source-pathwise oracle and $\mathcal{O}(\lvert V\rvert^2)$ conditional independence tests.
\end{theorem}

Note that \cref{theorem:graph-alg} relies on a weaker notion of faithfulness in the causal graph. In particular, it purely relies on $X_i$ and $X_j$ being dependent conditioned on some set $S$, if there is a directed path from $X_i$ to $X_j$ (or vice versa) and $S$ forms a topological prefix of the graph. Note that under \cref{assumption:graph} this faithfulness property will hold with very high probability (at least $1-O \left( 1/2^{\Omega \left( n \right)}  \right)$)\footnote{Given the $(\Omega(n),\Omega(\frac{1}{n \log(n)}))$-support for $E_i$, one can show that $X_i$ will have a support size of at least $\Omega(n)$, with probability at least $1-O(1/2^{\Omega(n)})$. The probability of any two particular states of $X_i$ producing the same conditional distribution of $X_j|X_i=x_i$ is at most $1/n$ (by considering the realization of the last probability mass from $X_i$). Accordingly, the probability of this happening for $\Omega(n)$ states of $X_i$ is bounded by $1/n^{\Omega(n)}$.}, meaning it is not positing any assumption beyond \cref{assumption:graph}.

Finally, we show that we can use entropic causality together with \cref{alg:general} for learning general causal graphs:
\begin{corollary}
\label{corollary:general}
For any SCM under \cref{assumption:graph}, using entropic causality for pairwise comparisons in \cref{alg:general} learns, with high probability, the causal graph that is implied by the SCM.
\end{corollary}

\section{Experiments}
\label{section:experiments}
We first introduce a heuristic that we call the \emph{entropic enumeration} algorithm. In this algorithm, we enumerate over all possible causal graphs consistent with the skeleton and calculate the minimum entropy needed to generate the observed distribution from the graph with independent noise at each node. The minimum entropy needed to generate the joint distribution with some graph $D$ is $\sum_{X_i} \mec(X_i | \pa_D(X_i))$ where $\pa_D(X_i)$ denotes the parents of $X_i$ in $D$. The graph requiring the least randomness is then selected.

We are not aware of any provably correct method for causal discovery between categorical variables that are non-deterministically related. For discrete variables, the only such method other than entropic causality is the discrete additive noise model~\cite{peters2011causal}. We compare entropic causality to discrete ANM for learning causal graphs, using the graph extension of ANM proposed by \citet{mooij2009regression}. To isolate the role of our algorithms in identifying causal graphs beyond the equivalence class, we support every algorithm in our comparisons with the skeleton of the true graph (obtainable from conditional independence tests given enough data). We evaluate performance via the structural Hamming distance (SHD) from the estimated graph to the true causal graph. See the Appendix and our GitHub for details: \url{https://github.com/SpencerCompton/entropic-causality-for-graphs}.

\begin{figure*}[ht!]
\vskip .2in
\begin{center}
\subfigure[]{\label{fig:alarm}\includegraphics[width=0.32\textwidth]{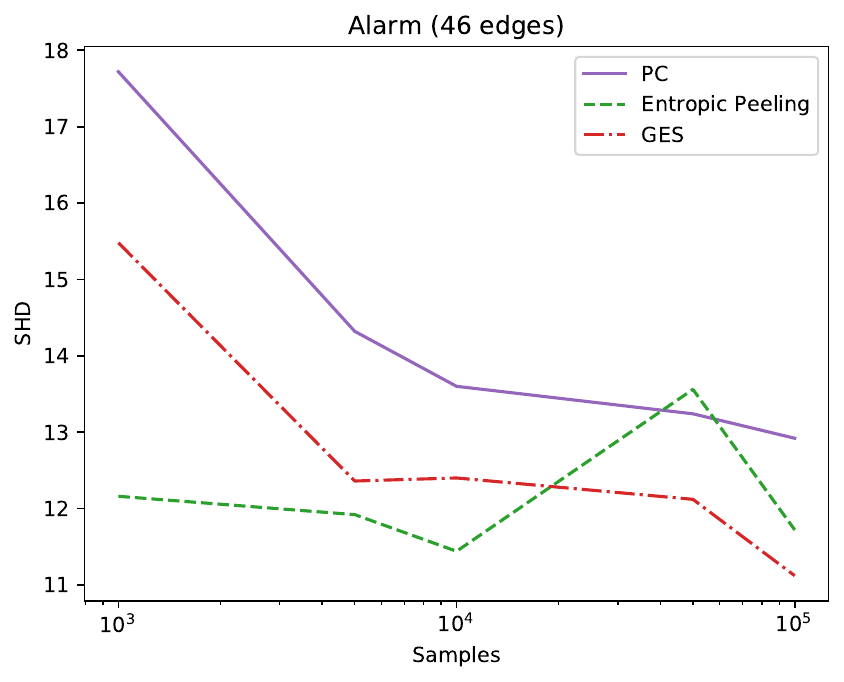}}%
\subfigure[]{\label{fig:survey}\includegraphics[width=0.32\textwidth]{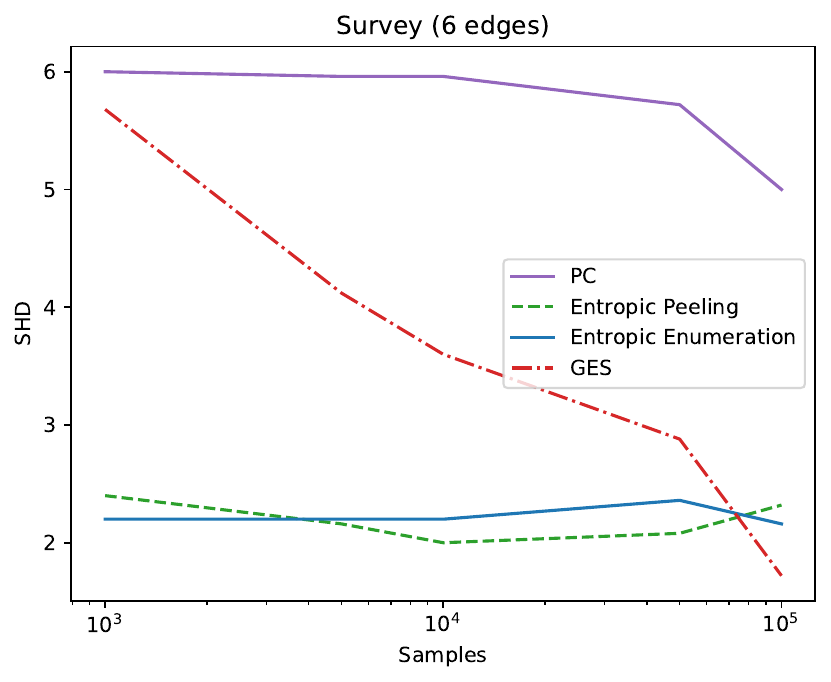}}
\subfigure[]{\label{fig:sachs}\includegraphics[width=0.32\textwidth]{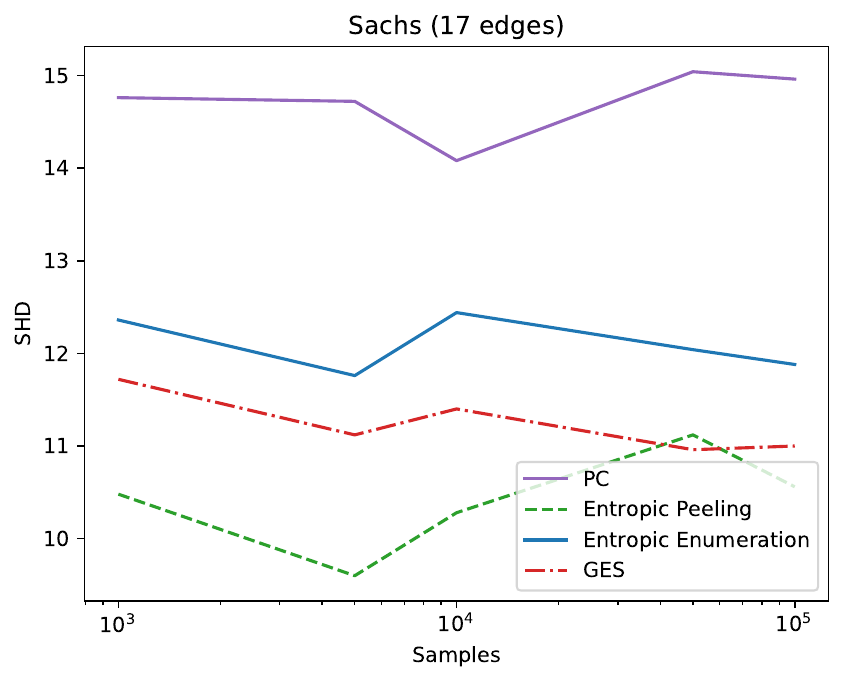}}\\
\subfigure[]{\label{fig:asia}\includegraphics[width=0.32\textwidth]{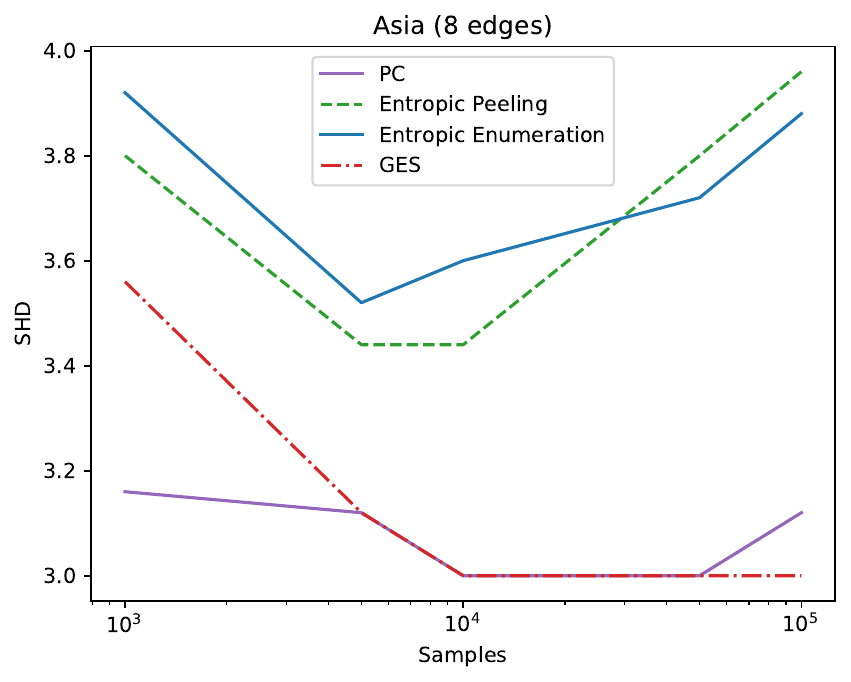}}%
\subfigure[]{\label{fig:cancer}\includegraphics[width=0.32\textwidth]{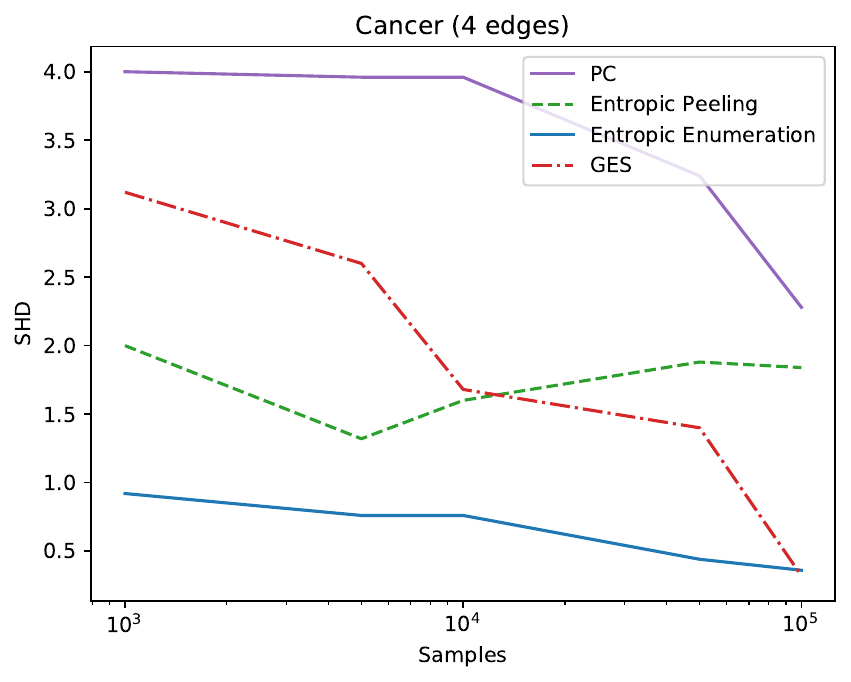}}%
\subfigure[]{\label{fig:earthquake}\includegraphics[width=0.32\textwidth]{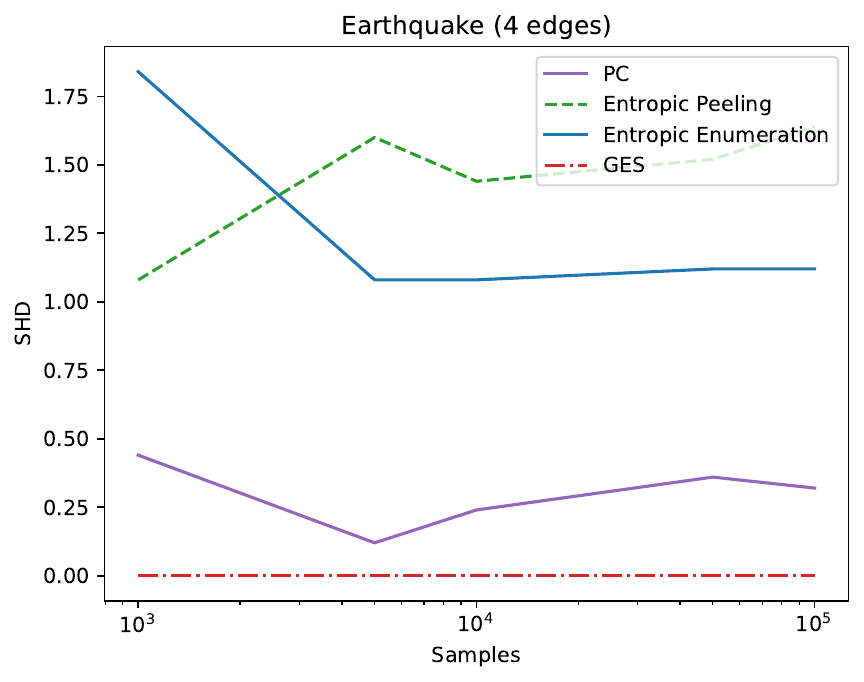}}
\subfigure[]{\label{fig:child}\includegraphics[width=0.32\textwidth]{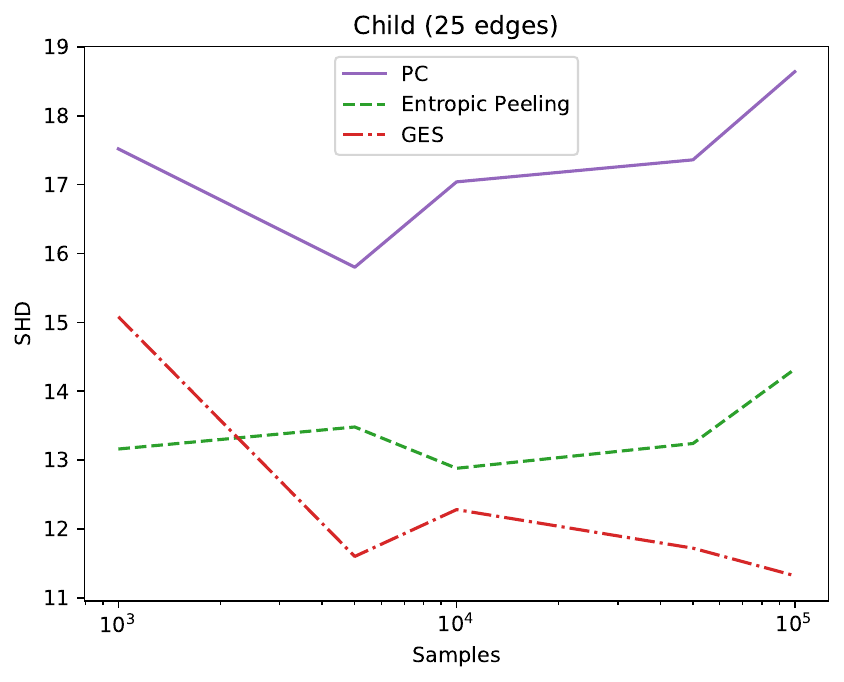}}
\subfigure[]{\label{fig:insurance}\includegraphics[width=0.32\textwidth]{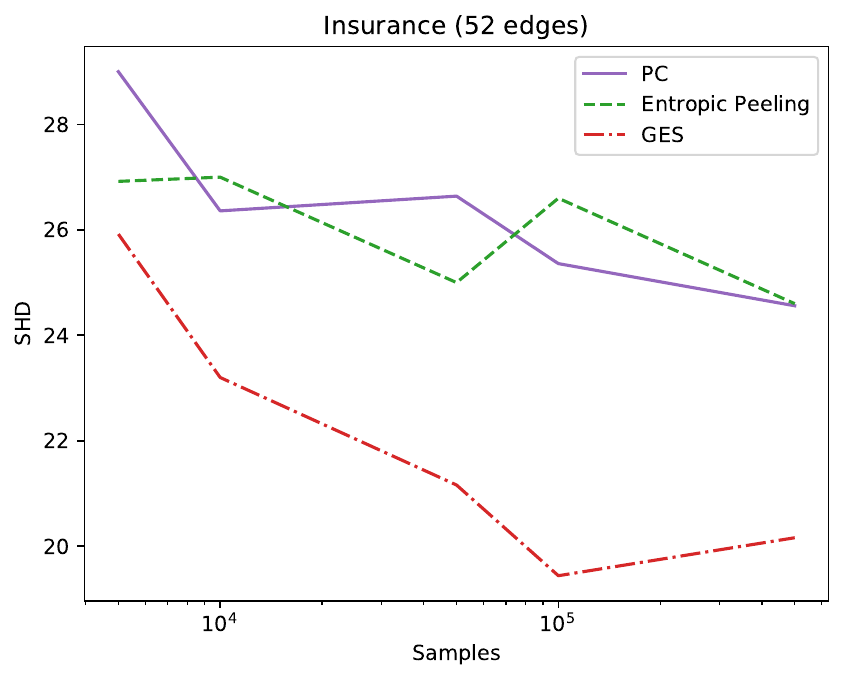}}
\caption{Performance of methods on networks from the \emph{bnlearn} repository with varying samples: $25$ datasets are sampled for each configuration from the \emph{bnlearn} network. The Alarm/Child/Insurance graphs are too large to run Entropic Enumeration.}
\label{fig:bnlearn_entropic}
\end{center}
\vskip -0.2in %
\end{figure*}

\begin{figure*}[ht!]
\begin{center}
\subfigure[]{\label{fig:3_node_line_anm100}\includegraphics[width=0.32\textwidth]{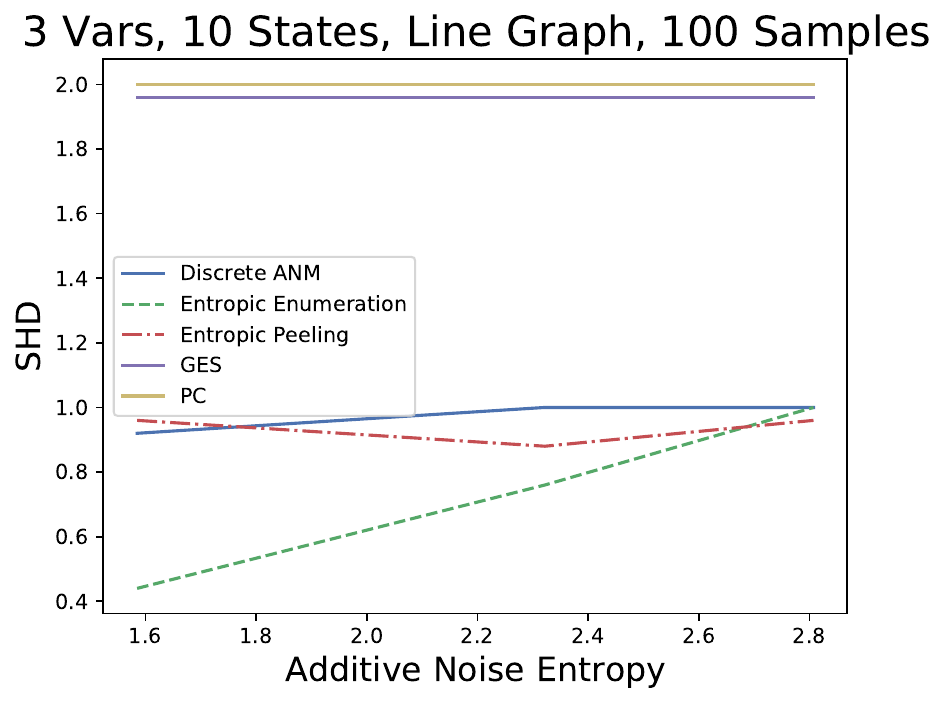}}
\subfigure[]{\label{fig:3_node_line_anm500}\includegraphics[width=0.32\textwidth]{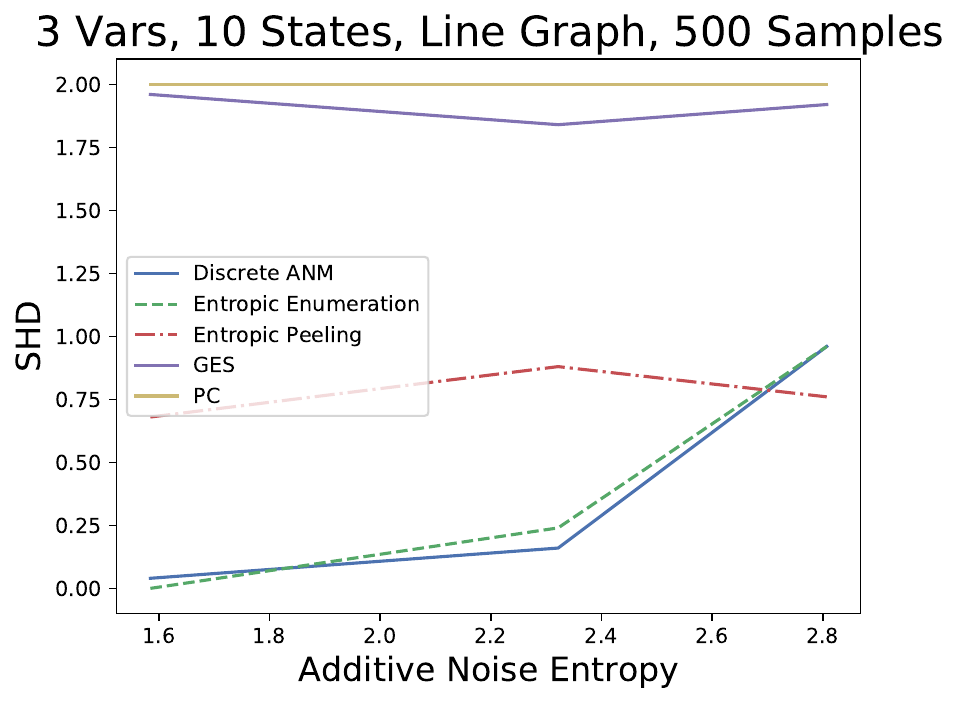}}
\subfigure[]{\label{fig:3_node_line_anm1000}\includegraphics[width=0.32\textwidth]{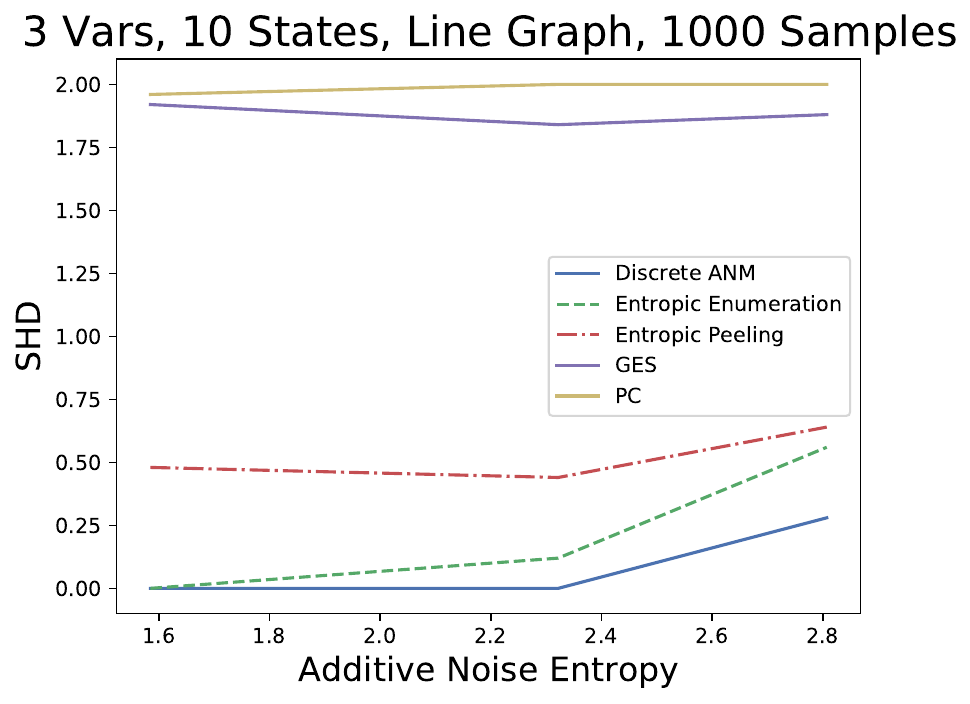}}
\caption{Performance of methods in the ANM setting on the line graph $X\rightarrow Y \rightarrow Z$: $25$ datasets are sampled for each configuration from the ANM model $X=f(\pa_X)+N$. The $x-$axis shows entropy of the additive noise.}
\label{fig:3_node_line_anm}
\end{center}
\end{figure*}
\paragraph{Performance on Synthetic Data.}
Figure \ref{fig:3_node_triangle_entropic_HES} compares the performance of entropic peeling, entropic enumeration and discrete ANM algorithms for the triangle graph, i.e., the graph with edges $X\rightarrow Y$, $Y \rightarrow Z$, and $X\rightarrow Z$. Every datapoint is obtained by averaging the SHD to the graph for $50$ instances of structural models. To ensure that the entropy of the exogenous nodes are close to the value on the x-axis, their distributions are sampled from a Dirichlet distribution with a parameter that is obtained through a binary search. %
We observe that the entropic methods consistently outperform the ANM approach. Importantly, we observe how entropic methods are able to near-perfectly learn the \emph{exact} triangle graph in almost all regimes, even though all triangle graphs are in the same Markov equivalence class and thus traditional structure learning algorithms like PC or GES cannot learn anything. With enough samples, entropic enumeration learns the graph near-perfectly until the exogenous noise nears $\log(n)$, exceeding our theoretical guarantee of $o(\log(\log(n)))$. In Figure \ref{fig:3_node_triangle_entropic_HES}, we fix the source node to have high entropy. Our motivation is that if all nodes have essentially zero randomness, then we expect the performance of any method to degrade as there is no randomness in samples to observe causality or faithfulness. In Figure \ref{fig:3_node_triangle_entropic} in Appendix, we do not fix a high-entropy source and still observe that entropic methods outperform ANM in almost all regimes. Experiments with different and larger graphs can be seen in the Appendix.

\paragraph{Performance in Discrete Additive Noise Regime.} In \cref{fig:3_node_line_anm}, we compare the performance of the entropic algorithms and discrete ANM \emph{when the true SCM is a discrete additive noise model}. Using the discrete ANM generative model, we observe that entropic enumeration out-performs the discrete ANM method with few samples and matches its performance with many samples. This demonstrates that even though entropic methods are designed for the general unconstrained SCM class, they perform similarly to ANM which was designed specifically for this setting.

\paragraph{Effect of Finite Samples.} We observe that entropic methods, particularly enumeration, work well even in regimes with low samples. Experiments focusing on the impact of finite samples can be found in the Appendix.

\paragraph{Performance on Real-World Data.} Due to the computational cost of discrete ANM, we compared entropic causality against GES and PC algorithms to evaluate how well it learned real-world causal graphs from the \emph{bnlearn} repository beyond their equivalence class. Figure \ref{fig:bnlearn_entropic} shows performance on the eight networks we evaluated. See also \cref{fig:real_data} for an examination of how the true graph's entropy compares with other orientations. In \cref{fig:real_data}, it seems interesting how for such a large graph as Alarm, the true graph is often the entropy minimizer among sampled orientations. It is also peculiar how for another large graph Insurance, the true graph is often the entropy maximizer among sampled orientations. This seems like a phenomenon that would be nice to have a better understanding about. We certainly do not claim that the assumptions of entropic causality are universally true in nature, but it does seem interesting how it seems there are real settings where the framework suggests strong signals about the true orientation.

\section{Conclusion}
In this work, we have extended the entropic causality framework to graphs. An identifiability result was proven, and two algorithms were presented and experimentally evaluated --- a theoretically-motivated sequential peeling algorithm and a heuristic entropic enumeration algorithm that performs better on small graphs. Overall, we observed strong experimental results in settings much more general than the assumptions used in our theory, indicating that a much stronger theoretical analysis might be possible. We note however that the quantity $H(E_i) = \Theta(\log(\log(n)))$ appears to be approximately a phase transition for the balls-and-bins setting, and posit that the development of novel tools may be required for such an extension of the theory.

We suggest such an advancement may involve an increased focus on a total entropy criterion (i.e., an extension of comparing $H(X)+H(E)$ to $H(Y)+H(\tilde{E})$ in the bivariate case), as in our proposed algorithm of entropic enumeration. Experiments indicate that this performs well, and one might argue that it appears to be more conceptually justified. For one, it mirrors Occam's razor in that it prefers the causal graph with minimal total randomness required. While we do not claim that this methodology will always discover the true generative model (as Occam's razor does not require the simplest explanation to \emph{always} be true), we believe these intuitions mirror nature more often than not, as confirmed by our experimental results. Moreover, such an approach appears to fare better with counter-examples for exogenous-based criterion such as the traveling ball scenario of \cite{janzing2019cause} discussed in \cite{compton2021entropic}. Showing theoretical guarantees for this approach's performance is of interest in future work, and can be framed more generally as, \emph{``Under what conditions is the true generative model the most information-theoretically efficient way to produce a distribution?''}

\flushcolsend
\clearpage

\bibliography{paper}
\bibliographystyle{icml2022}

\flushcolsend
\clearpage 
\appendix

\section{Proof of \cref{theorem:pairwise}}
\label{app:pairwise}
\subsection{Proof Outline}

Following the approach described in the proof overview in the main text (with the descriptions of helpful and hurful mass), we first introduce the \emph{surplus} of a state of $Y$ to characterize the amount of hurtful mass it receives:

Recall that $S$ is the set of ``plateau states" of $X$, i.e., those whose probabilities are close to one another. 

\begin{restatable}[Surplus]{definition}{defsurplusgeneric}
\label{definition:surplus-generic}
We define the surplus of a state $y$ of $Y$ as $z_y = \sum_{j \notin S} \max (0, P(X=j, Y=i) - \mathcal{T})$.
\end{restatable}

Intuitively, only values from states of $X$ outside the plateau states which exceed the threshold will significantly ``hurt" conditional entropy $H(X|Y=y)$. We will show there is a state $y'$ of $Y$ where $z_{y'}$ is small and $y'$ receives $\Omega(\frac{\log(n)}{\log(\log(n))})$ plateau balls. To bound $z_{y'}$, we will characterize it as the sum of contributions from three types of balls from $(X \backslash S) \times E$.\footnote{In the proofs, with a slight abuse of notation, we use $X,E$ both for the observed and exogenous variables, respectively and their supports.}

\begin{restatable}[Ball characterizations]{definition}{balltypes}
We characterize three types of balls:
\begin{enumerate}
    \item Dense balls. Consider a set $L$ of states of $X$, where a state of $X$ is in $L$ if $P(X=x) \ge \frac{1}{\log^{3}(n)}$. Dense balls are all balls of the form $(x \in L, e \in E)$. We call these dense balls, because the low-entropy of $E$ will prevent the collective mass of these balls from ``expanding'' well. 
    \item Large balls. For all balls of the form $(x \in X \backslash (S \cup L), e \in E)$ where the ball has mass $\ge \frac{\mathcal{T}}{2}$.
    \item Small balls. For all balls of the form $(x \in X \backslash (S \cup L), e \in E)$ where the ball has mass $< \frac{\mathcal{T}}{2}$.
\end{enumerate}
\label{def:balltypes}
\end{restatable}

We will show there are a non-negligible fraction of bins such that $z_y$ is small. To do so, we will bound the contribution from dense balls by showing that the small entropy of $E$ prevents ``spread'' in a sense, as there cannot be many states of $Y$ that receive much contribution towards $z_y$ from these dense balls. We will bound contribution from large balls by bounding the number of large balls, and showing that a non-negligible number of bins receive no large balls. Finally, we will bound contribution from small balls by showing how they often are mapped to states of $Y$ that have yet to receive $\frac{\mathcal{T}}{2}$ mass from the corresponding state of $X$, meaning they often don't immediately increase $z_y$. 

Finally, we will show (with high probability) the existence of a bin $y'$ with small $z_{y'}$ that will receive many plateau balls, and how this will imply $H(X | Y=y') = \Omega(\log(\log(n)))$.

\subsection{Complete Proof}
\textbf{Bounding $H(\tilde{E})$ via $H(X|Y=y)$.} Because $\tilde{E} \indep Y$, it must be true that $H(\tilde{E}) \ge \max_{y} H(X | Y=y)$. This is simple to prove by data processing inequality and is shown in Step 1 of the proof of Theorem 1 by \cite{compton2021entropic}. We aim to show there exists a $y'$ such that $H(X | Y=y') = \Omega(\log(\log(n)))$.

\textbf{Showing existence of a near-uniform plateau.} First, we aim to find a subset of the support of $X$ whose probabilities are multiplicatively close to one another. Here, we have a looser requirement for closeness than \cite{compton2021entropic}. Instead of requiring these probabilities to be within a constant factor of each other, we allow them to be up to a factor of $\log^{\close}(n)$ apart where $\close$ is a constant such that $0 < \close < 1$. While there are multiple values of $\close$ that would be suitable for our analysis, for simplicity of presentation we choose $\close = \frac{1}{4}$ throughout. This set of states of $X$ that are multiplicatively close to one another will be called the \emph{plateau} of $X$. We begin by showing how the $(\Omega(n),\Omega(\frac{1}{n \log(n)}))$-support assumption implies a plateau of states of $X$:

\begin{restatable}[Plateau existence]{lemma}{lemmaplateau}
\label{lemma:plateau}
Suppose $X$ has $(\csup n, \frac{1}{\clb n \log(n)})$-support for constants $0 < \csup \le 1$ and $\clb \ge 1$. Additionally, assume $n$ is sufficiently large such that $\frac{\log(2 \clb / \csup)}{\log(\log(n))} \le 1$. Then, there exists a subset $S \subseteq [n]$ of the support of $X$, such that the following three statements hold:
\begin{enumerate}
\item $\frac{\max_{i \in S} P(X = i)}{\min_{i \in S} P(X = i)} \le \log^\close(n)$\\ 
\item $\min_{i \in S} P(X = i) \ge \frac{1}{\clb n \log(n)}$\\
\item $|S| \ge \frac{\close \csup n}{6}$, for any $0<\close<1$.
\end{enumerate}
\end{restatable}
\begin{proof}
By definition of $(\csup n, \frac{1}{\clb n \log(n)})$ support, there are at least $\csup n$ states of $X$ with probability in range $[\frac{1}{\clb n \log(n)},1]$. Moreover, at most $\frac{\csup n}{2}$ states will have probabilities in range $[\frac{2}{\csup n}, 1]$. Otherwise, they would have total probability mass $>1$ which is impossible. Therefore, there are at least $\frac{\csup n}{2}$ states with probabilities in range $[\frac{1}{\clb n \log(n)},\frac{2}{\csup n}]$.

Now, we aim to divide the range $[\frac{1}{\clb n \log(n)},\frac{2}{\csup n}]$ into a number of contiguous segments such that all values in any segment are multiplicatively within $\log^\close(n)$ of each other. To do so, we can create segments $[\frac{1}{\clb n \log(n)} \times (\log^\close(n))^i , \frac{1}{\clb n \log(n)} \times (\log^\close(n))^{i+1}]$ from $i=0$ until the smallest $i$ that satisfies $\frac{1}{\clb n \log(n)} \times (\log^\close(n))^{i+1} \geq  \frac{2}{\csup n}$. Accordingly, we need $\lceil \frac{\log(({2}/{(\csup n)})/({1/{(\clb n \log(n))}}))}{\log(\log^\close(n))} \rceil \le 1 + \frac{1}{\close} + \frac{\log(2 \clb / \csup)}{\close \log(\log(n))} \le \frac{3}{\close}$ groups. Hence one group must have at least $\frac{\csup n /2}{3 / \close} = \frac{\close \csup n}{6}$ states of $X$ that are multiplicatively within $\log^\close(n)$ and have probability at least $\frac{1}{\clb n \log(n)}$.
\end{proof}

\textbf{Lower-bounding the most probable state of $E$.} Our proof method focuses on a balls-and-bins game where states of $X \times E$ are balls and states of $Y$ are bins. We focus first on \emph{plateau balls}, which are balls corresponding to states of $S$ (the set of plateau states of $X$) and the highest probability state of $E$. In particular, they are balls of the form $(X \in S, E=e_1)$ where $e_1$ is the most probable state of $E$. To show that these plateau balls have enough probability mass to be helpful, we first show that $P(E=e_1)$ is relatively large:

\begin{restatable}{lemma}{lemmaebound}
\label{lemma:e-bound}
If $H(E) \le \close \log(\log(n))$ then $P(E=e_1) \ge \frac{1}{\log^\close(n)}$
\end{restatable}
\begin{proof}
For any distribution with entropy $H$, its state with the highest probability has at least probability $2^{-H}$ (see Lemma 5 of \cite{compton2021entropic}). Thus if $H(E) \le \close \log(\log(n))$ then $P(E = e_1) \ge 2^{- \close \log(\log(n))} = \frac{1}{\log^\close(n)}$.
\end{proof}

\textbf{Introducing surplus.} We now begin proving how there exists a bin that receives a large amount of mass that helps the bin have large conditional entropy (such helpful mass includes the plateau balls), and not much mass that hurts the conditional entropy making it small. To formalize this hurtful mass, recall the \emph{surplus} quantity described in \cref{definition:surplus-generic}. This surplus is a way of quantifying the probability mass received by a state of $Y$ that is hurtful towards making the conditional entropy large. We define surplus, with the threshold of $\mathcal{T}$ specified as $\frac{12}{n \log(n)}$ as follows:

\begin{restatable}[Surplus, $\mathcal{T}=\frac{12}{n \log(n)}$]{definition}{defsurplus}
We define the surplus of a state $i$ of $Y$ as $z_i = \sum_{j \notin S} \max (0, P(X=j, Y=i) - \frac{12}{n \log(n)})$, where $S$ is the set of plateau states of $X$. 
\end{restatable}

\textbf{Characterizing balls-and-bins.} Now we will show that there are a non-negligible number of states of $Y$ where the surplus is small. Recall from our proof outline that we view the process of realizing the random function $f$ as a balls-and-bins game. In particular, each element of $X \times E$ (a ball) is i.i.d. uniformly randomly assigned to a state of $Y$ (a bin). Only balls of the form $(x \in X \backslash S, e \in E)$ affect a bin's surplus. To bound surplus for bins, we characterize it as the sum of contributions from three types of balls from $(X \backslash S) \times E$, and restate this characterization from the proof outline:

\balltypes*

\textbf{Bounding the harmful effects of dense balls.} Recall that $\mathcal{T}=\frac{12}{n \log(n)}$. Now, we show how to bound the contribution of dense balls towards surplus. By our assumptions, $Y=f(X,E)$, and $H(E)$ is small, meaning there is not much randomness in our function. We defined $L$ as states of $X$ with probability at least $\frac{1}{\log^{3}(n)}$, so $|L| \le \log^{3}(n)$. We would like to show that there are not too many bins where the dense balls contribute a significant amount to surplus. If $H(E)=0$, this would be easy to show as then there would only be $|L| \le \log^{3}(n)$ dense balls and thus they could only affect the surplus of $\log^{3}(n)$ bins. However, we aim to show this claim in the more general setting where $H(E)=o(\log(\log(n)))$. To accomplish this, we follow the same intuition to show that the limited entropy of $E$ prevents this small number of states of $X$ from greatly ``spreading'' to significantly affect a large number of states of $Y$. In particular, we show:

\begin{restatable}[Limited expansion]{lemma}{lemmanoexpansion}
\label{lemma:no-expansion}
Suppose $Y$ can be written as a function $f(X,E)$ and $X \indep E$. Consider any subset $R$ of the support of $X$. For any subset $T$ of the support of $Y$ that satisfies $\forall t\in T: P(X \in R, Y=t) > \delta$, the cardinality of $T$ is upper bounded as
$|T| \le \frac{H(E) + \log(|R|) + 2}{\delta \log(\frac{1}{\delta})}$.
\end{restatable}
\begin{proof}
Consider a variable $X'$, whose distribution is obtained from the distribution of $X$ by keeping only the states in $R$, and then normalized. More formally, for any $i \in R$, $P(X' = i) = \frac{P(X = i)}{P(X \in R)}$, and for any $i \notin R$, $P(X' = i) = 0$. 

Recall $Y=f(X,E)$. Using the same $f,E$, we define $Y' = f(X',E)$. Note that $P(X \in R, Y=i) \le P(Y' = i)$. If $P(X \in R, Y=i) \ge \delta$, then it must be true that $P(Y'=i) \ge \delta$. Moreover, this implies that if there exists such a subset $T$ then $H(Y') \ge |T| \delta \log(\frac{1}{\delta}) - 2$ (note the negative two is from the fact that modifying a distribution by adding non-negative numbers to probabilities can decrease entropy by at most $2$). Moreover, by data-processing inequality note that $H(Y') \le H(X')+H(E | X' ) \le H(X')+H(E) \le \log(|R|) + H(E)$, where previous inequality is due to the fact that conditioning reduces entropy. This implies the desired inequality for the cardinality of set $T$.
\end{proof}

To more directly use this for our goal, we present:

\begin{restatable}{corollary}{corollarydenseball}
\label{corollary:denseball}

There exist no subset $|T|=n/4$ such that $\forall t \in T: P(X \in L, Y=t) \ge \frac{1}{n \log(\log(n)) \log^{2\close}(n)}$
\end{restatable}
\begin{proof}
Note that $|L| \le \log^3(n)$. By \cref{lemma:no-expansion}, any such $T$ must satisfy:
\begin{align}
    & |T| \\
    & \le \frac{H(E)+\log(|L|)+2}{1/(n \cdot \log(\log(n)) \cdot \log^{2 \close}(n)) \cdot \log(n)} \\
    & \le \frac{5 \log^2(\log(n)) \cdot n \cdot \log^{2 \close}(n)}{\log(n)} \\
    & \le \frac{5 \log^2(\log(n)) \cdot n \cdot \log^{1/2}(n)}{\log(n)} \label{step:close-bound}\\
    & \le \frac{n}{4} \label{step:final-nlarge}
\end{align}
We obtain Step \ref{step:close-bound} by previously setting $\close = \frac{1}{4}$. We obtain Step \ref{step:final-nlarge} when $n$ is sufficiently large such that $\frac{5 \log^2(\log(n))}{\log^{1/2}(n)} \le \frac{1}{4}$. It can be shown that $n\geq 5$ is sufficient.
\end{proof}

As a result, dense balls cannot significantly affect the surplus of many bins. 

\textbf{Bounding the harmful effects of large balls.} We now show how large balls cannot significantly affect the surplus of too many bins, by showing there is a non-negligible number of bins that receive no large balls. 

\begin{restatable}[Avoided big]{lemma}{lemmaavoidbig}
\label{lemma:empty-bins} Given a balls-and-bins game with $c \cdot n \ln(n)$ balls mapped uniformly randomly to $n$ bins, at least $\frac{n^{1-c}}{2}$ bins will receive no balls with high probability if $c$ is a constant such that $0 < c \le \frac{1}{3}$.
\end{restatable}
\begin{proof}
This follows directly from \cite{wajc2017negative}. By \cite{wajc2017negative}, with high probability the number of empty bins will be $n^{1-c} \pm O(\sqrt{n \log(n)})$. For sufficiently large $n$, $O(\sqrt{n \log(n)}) \le \frac{n^{2/3}}{2} \le \frac{n^{1-c}}{2}$ and thus the number of empty bins is at least $\frac{n^{1-c}}{2}$ with high probability.
\end{proof}

Note how this relates to the coupon collector's problem, where it is well-known that $\Theta(n \log(n))$ trials are necessary and sufficient to receive at least one copy of all coupons with high probability. This is analogous to the number of balls needed such that every bin has at least one ball. The result of \cref{lemma:empty-bins} is intuitive from the coupon collector's problem, because the number of trials needed concentrates very well. Meaning, with a constant-factor less number of trials than the expectation required, there are many coupons that have not yet been collected with high probability.

\begin{corollary}
\label{corollary:bigballs}
As there are at most $\frac{1}{\mathcal{T}/2} \le \frac{n \log(n)}{6} \le \frac{1}{4} \cdot n \ln(n)$ large balls, with high probability there are at least $\frac{n^{3/4}}{2}$ bins that receive no large balls.
\end{corollary}

\textbf{Bounding the harmful effects of small balls.} For the small balls, we will also show that they cannot contribute too much surplus to too many states of $Y$. We will notably use that all small balls correspond to a state of $X$ where $P(X=x) \le \frac{1}{\log^{3}(n)}$. We will utilize this to show that most small balls are assigned to a state of $Y$ that has not yet received $>\frac{\mathcal{T}}{2}$ mass from its corresponding state of $X$, and accordingly would not increase the surplus. To accomplish this, we define a surplus quantity that only takes into account small balls:

\begin{definition}[Small ball surplus]
We define the small ball surplus of a state $y$ of $Y$ as
\begin{align*}
    &z_y^{\textrm{small}} = \sum_{x \notin (S \cup L)} \\
    & \max \left( \left(\hspace{-0.1in}\sum\limits_{\substack{\hspace{-0.05in}e  :\\\hspace{0.05in} P(X=x,E=e)\\\hspace{0.02in}<\frac{\mathcal{T}}{2}}}\hspace{-0.2in}P(X=x, E=e, Y=y) \right) - \mathcal{T} \right)
\end{align*}
\end{definition}

With this notion of surplus constrained to small balls, we show the following:

\begin{restatable}[Small ball limited surplus]{lemma}{lemmasmallball}
\label{lemma:smallball}
With high probability, there are at most $\frac{n}{4}$ values of $i$, i.e., number of bins, where $z_{i}^{\textrm{small}} \ge \frac{1}{n \log(\log(n)) \log^{2 \close}(n)}$.
\end{restatable}
\begin{proof}
We will consider all small balls in an arbitrary order. Let $x(t)$ be the corresponding state of $X$ for the $t$-th small ball, $e(t)$ the corresponding state of $E$, and $w_{\textrm{ball}}(t)$ be the ball's probability mass (i.e., $P(X=x(t),E=e(t))$). Recall that for all small balls it must hold that $x(t) \notin L$ and thus $P(X=x(t)) < \frac{1}{\log^{3}(n)}$. We define the total small ball surplus as $Z^{\textrm{small}}=\sum_{y \in Y} z_y^{\textrm{small}}$. Now, we will consider all small balls in an arbitrary order and realize their corresponding entry of $f$ to map them to a state of $Y$. Initially, we have not realized the entry of $f$ for any balls and thus all $z^{\textrm{small}}_y=0$ and $Z^{\textrm{small}}=0$. As we map small balls to states of $Y$, we define $\Delta(t)$ as the increase of $Z^{\textrm{small}}$ after mapping the $t$-th ball to a state of $Y$. By definition, $\sum_t \Delta(t)$ is equal to $Z^{\textrm{small}}$ after all values of $f$ have been completely realized. 

Our primary intuition is that we will show for many small balls it holds that $\Delta(t)=0$. As a result, we expect $Z^{\textrm{small}}$ to not be very large. 

As a result, we expect $Z^{\textrm{small}}$ to not be very large. Let $y(t)$ be equal to $f(x(t),e(t))$, the state of $Y$ that the $t$-th ball is mapped to. 
As $f$ is realized for each configuration, let $w_{Y}^t(y,x)$ denote the total mass of balls assigned to state $y$ of $Y$ so far from state $x$ of $X$, i.e., $w_{Y}^t(y',x') \coloneqq \sum\limits_{x', e: f(x',e)=y'} w_{\textrm{ball}}(t')$.

We upper-bound the expectation of $\Delta(t)$:

\begin{claim}
\label{claim:gamma-expectation}
Regardless of the realizations of all $\Delta(t')$ for $t'<t$, it holds that $\Delta(t)$ is a random variable with values in range $[0,w_{\textrm{ball}}(t)]$ and $E[\Delta(t)] \le \frac{w_{\textrm{ball}}(t)}{\log^{2}(n)}$.
\end{claim}

\begin{proof}
The only conditions under which $\Delta(t)$ takes a positive value (which is upper-bounded by $w_{\textrm{ball}}(t)$), is when $w_Y^t(y(t),x(t)) > \frac{\mathcal{T}}{2}$ before the $t$-th ball is realized. Recall that $P(x(t)) \le \frac{1}{\log^{3}(n)}$. Accordingly, the number of states $y'$ of $Y$ where $w_Y^t(y',x(t)) > \frac{\mathcal{T}}{2}$ is upper-bounded by $\frac{P(X=x(t))}{\mathcal{T}/2} \le \frac{1/\log^{3}(n)}{6/(n \log(n))} = \frac{n \log(n)}{6 \log^{3}(n)} \le \frac{n}{\log^{2}(n)}$. This is due to the fact that balls partition the total mass of $P(X=x(t))$ since we have $P(X=x(t))=\sum_eP(X=x(t),E=e)$. 
This implies that the probability that the $t$-th ball will be mapped to a state $y'$ of $Y$ such that $w_Y^t(y',x(t))$  already exceeds the threshold of $\mathcal{T}/2$ (in other words where we might have $\Delta(t)>0$) is upper-bounded by $\frac{n / \log^{2}(n)}{n} = \frac{1}{\log^{2}(n)}$ due to the fact that the function $f$ is realized independently and uniformly randomly for each pair of $(x,e)$, i.e., for every distinct ball. Accordingly, $E[\Delta(t)] \le \frac{w(t)}{\log^{2}(n)}$.
\end{proof}

This enables us to upper-bound the sum of $\Delta(t)$:

\begin{claim}
$\sum_t \Delta(t) \le \frac{1}{4 \log(n)}$ with high probability.
\end{claim}
\begin{proof}

We will transform $\Delta(t)$ into a martingale. In particular, we define $\Delta'(t) = \Delta'(t-1) + \Delta(t) - E[\Delta(t) | \Delta(1), \dots, \Delta(t-1)]$. We define $\Delta'(0)=0$, and note that $\Delta'(c)$ is a martingale. By Azuma's inequality, we show $|\sum_t \Delta'(t)| \le \frac{1}{8 \log(n)}$ with high probability:
\begin{align*}
    & P[|\Delta(t)| > \varepsilon] < 2 e^{- \frac{\varepsilon^2}{2 \sum c_i^2}} \\
    & \le 2 e^{- \frac{\left( \frac{1}{8 \log(n)} \right)^2}{2 (\max_i c_i) \cdot \sum c_i}} \\
    & \le 2 e^{- \frac{\left( \frac{1}{8 \log(n)} \right)^2}{2 \times \mathcal{T}/2 \cdot 1}} \\
    & = 2 e^{\frac{- n \log(n)}{12 \times 8 \times |V| \times \log(n)}} \\
\end{align*}
Accordingly, by definition of $\Delta'(t)$ this implies $|(\sum_t \Delta(t)) - \sum_c E[\Delta(t) | \Delta(1) , \dots, \Delta(c-1)]| \le \frac{1}{8 \log(n)}$. By \cref{claim:delta-expectation-claim} we know all $E[\Delta(t) | \Delta(1) ,\dots, \Delta(c-1)] \le \frac{w_{\textrm{config}(c)}}{\log^2(n)}$ and accordingly, $\sum_t E[\Delta(t) | \Delta(1), \dots, \Delta(c-1)] \le \frac{1}{\log^2(n)}$. Together, these imply $\sum_t \Delta(t) \le \frac{1}{8 \log(n)} + \frac{1}{\log^2(n)}$ with high probability, and for sufficiently large $n$ it holds that $\frac{1}{\log^2(n)} \le \frac{1}{8 \log(n)}$. Thus, our high-probability on $|\Delta'(t)|$ implies that $\sum_t \Delta(t) \le \frac{1}{4 \log(n)}$ with high probability.
\end{proof}

Finally, we conclude that our upper-bound on $\sum_t \Delta(t)$ implies an upper-bound on the number of states of $Y$ with non-negligible small ball support:

\begin{claim}
If $\sum_t \Delta(t) \le \frac{1}{4 \log(n)}$, then there are at most $\frac{n}{4}$ bins where $z_i^{\textrm{small}} \ge \frac{1}{n \log(\log(n)) \log^{2 \close}(n)}$.
\end{claim}
\begin{proof}
$Z^{\textrm{small}} = \sum_t \Delta(t) \le \frac{1}{4 \log(n)}$. Given this upper-bound for total small ball surplus, we can immediately upper-bound the number of states of $Y$ with small ball surplus greater than $\frac{1}{n \log(\log(n)) \log^{2 \close}(n)}$ by the quantity $\frac{1 / (4 \log(n))}{1 / (n \cdot \log(\log(n)) \cdot \log^{2\close}(n))} \le \frac{n \cdot \log(\log(n)) \cdot \log^{1/2}(n)}{4 \log(n)} \le \frac{n}{4}$. We obtain this by using $\close = \frac{1}{4}$ and for sufficiently large $n$ such that $\log(\log(n)) \le \log^{1/2}(n)$.

\end{proof}
\end{proof}

\textbf{Combining the three ball types: many bins with small surplus.} Now, we combine all these intuitions to show there are many bins that have a small amount of surplus. We have shown that, with high probability, the are at most $n/4$ bins with non-negligible mass from dense balls by \cref{corollary:denseball}, and at most $n/4$ bins with non-negligible mass from small balls \cref{lemma:smallball}. Combining these sets, there are at most $n/2$ bins with non-negligible mass from dense balls or small balls. By \cref{corollary:bigballs}, with high probability at least $\frac{n^{3/4}}{2}$ bins will receive no large balls. Our goal is to show the intersection of the sets is large, so there are many bins that have small surplus.

\begin{restatable}{lemma}{lemmasetintersection}
\label{lemma:set-intersection}
 Let there be two sets $A,B \subseteq [n]$, where $|A| \ge \frac{n}{2}$ and $A$ and $B$ are both independently uniformly random subsets of size $|A|$ and $|B|$, respectively. It holds that $P(|A \cap B| \ge \frac{|B|}{4}) \ge 1 - 2e^{\frac{-|B|}{8}}$.
\end{restatable}

\begin{proof}
To accomplish this, we will heavily utilize properties of negative association (NA). Lemma 8 of \cite{wajc2017negative} shows that permutation distributions are NA. Lemma 9 of \cite{wajc2017negative} shows closure properties of NA random variables. In particular, they show that concordant monotone functions defined on disjoint subsets of a set of NA variables are also NA. Accordingly, consider concordant monotone functions where each bin $i$ has a random variable $\mathcal{A}_i$ that takes value $1$ if it is the first $|A|$ values of a permutation distribution and value $0$ otherwise. These random variables are thus NA. Suppose we first realize the set $B$, independently of the realization of $A$. Then, a bin $y \in B$ would be in $A \cap B$ if $A_{y} = 1$. It is clear this formulation of the random process has a bijective mapping with the true random process, so $P(|A \cap B| \ge \frac{|B|}{4}) = P(\sum_{y \in B} \mathcal{A}_{y} \ge \frac{|B|}{4})$. By Theorem 5 of \cite{wajc2017negative}, we can use Hoeffding's upper tail bound to show $P(\sum_{y \in B} \mathcal{A}_{y} < \frac{|B|}{4}) \le P(|\sum_{y \in B} \mathcal{A}_{y} - E[\sum_{y \in B} \mathcal{A}_{y}]| > \frac{|B|}{4}) \le 2 e^{\frac{-|B|}{8}}$.
\end{proof}

\begin{corollary}
With high probability, there are at least $\frac{n^{3/4}}{8}$ bins with surplus $z_{y} \le \frac{2}{n \log(\log(n)) \log^{2 \close}(n)}$.
\end{corollary}
\begin{proof}
We have defined three types of balls, and have proven results that show how there are many bins with negligible bad contribution for each type of ball. Now, we combine these with \cref{lemma:set-intersection} to show there are many bins where there is not much bad contribution in total. By \cref{corollary:denseball} there are at most $n/4$ bins with more than $\frac{1}{n \log(\log(n)) \log^{2 \close (n)}(n)}$ mass from dense balls. By \cref{lemma:smallball}, there are at most $n/4$ bins with small ball surplus more than $\frac{1}{n \log(\log(n)) \log^{2 \close (n)}(n)}$. Let $A$ be the set of bins with at most $\frac{1}{n \log(\log(n)) \log^{2 \close (n)}(n)}$ mass from dense balls and at most $\frac{1}{n \log(\log(n)) \log^{2 \close (n)}(n)}$ small ball surplus. By combining \cref{corollary:denseball} and \cref{lemma:smallball} we know $|A| \ge \frac{n}{2}$ with high probability. Let $B$ be the set of bins that receive no big balls. By \cref{corollary:bigballs}, it holds that $|B| \ge \frac{n^{3/4}}{2}$ with high probability. By \cref{lemma:set-intersection}, it holds that $|A \cap B| \ge \frac{n^{3/4}}{8}$ with failure probability at most $2e^{\frac{-2n^{3/4}}{16}}$. Moreover, all such bins will have total surplus at most $\frac{2}{n \log(\log(n)) \log^{2 \close}(n)}$, because they receive no large balls and total surplus is then upper-bounded by the sum of small ball surplus and total mass from dense balls.
\end{proof}

\textbf{Existence of a small surplus bin with many plateau balls.} Recall plateau balls, which are balls of $X \times E$ that take the form $(x \in S, E = e_1)$, where $e_1$ is the most probable state of $E$. We show that at least one of the bins with small surplus will receive many plateau balls with high probability:

\begin{lemma}
There exists a bin with surplus at most $\frac{2}{n \log(\log(n)) \log^{2 \close}(n)}$ and at least $\frac{\log(n)}{2 \log(\log(n))}$ plateau balls.
\end{lemma}
\begin{proof}
Note that total surplus is independent of how plateau balls are mapped. Accordingly, we have determined a set of $\frac{n^{3/4}}{8}$ bins with small enough surplus. We aim to show that one of these bins receives a large number of plateau balls with high probability. We will rely on negative association (NA) in the balls-and-bins process to prove our result. 

\begin{restatable}{claim}{nathreshold}
\label{claim:na-threshold}
Indicator variables for if a bin receives some threshold of balls in a i.i.d. uniformly random balls-and-bins game are NA.
\end{restatable}
\begin{proof}
This follows immediately by using results of \cite{wajc2017negative}. By Theorem 10 of \cite{wajc2017negative}, the random variables of the number of balls assigned to each bin are NA. By Lemma 9 of \cite{wajc2017negative}, concordant monotone functions define on disjoint subsets of a set of NA random variables are NA. Accordingly, if we have an indicator variable for whether a bin receives at least some number of balls, these indicator variables are NA.
\end{proof}

Now, we lower-bound the expectation of these indicator variables:

\begin{restatable}{claim}{claimindividualbin}
\label{claim:individual-bin}
Suppose $cn$ balls ($c \le 1$) are thrown i.i.d. uniformly randomly into $n$ bins. The probability that a particular bin receives at least $k=\frac{d \log(n)}{\log(\log(n))}$ balls is at least $\frac{1}{e n^d}$ given that $\frac{d}{c} \le \log(\log(n))$.
\end{restatable}
\begin{proof}
We use the method outlined by \cite{Gupta2011balls}. We lower-bound the probability of a bin receiving at least $k$ balls as follows:
\begin{align}
     \binom{cn}{k} \cdot (\frac{1}{n})^k \cdot &(1 - \frac{1}{n})^{cn-k} \nonumber
     \ge (\frac{cn}{k})^k \cdot \frac{1}{n^k} \cdot \frac{1}{e} \nonumber\\
    & \ge \frac{1}{e} \cdot (\frac{c}{k})^k \nonumber\\
    & = \frac{1}{e} \cdot (\frac{c \log(\log(n))}{d \log(n)})^{\log_{\log(n)}(n^d)} \nonumber\\
    & \ge \frac{1}{e} \cdot (\frac{1}{\log(n)})^{\log_{\log(n)}(n^d)} \label{step:assume-log} \\\nonumber
    & = \frac{1}{en^d}
\end{align}

We obtain Step \ref{step:assume-log} by using $\frac{d}{c} \le \log(\log(n))$.
\end{proof}

By \cref{lemma:plateau} there are at least $\frac{\close \csup}{6} \cdot n=\frac{\csup}{24}\cdot n$ plateau balls. Now, consider NA indicator variables $\mathcal{B}_i$ for whether or not a particular bin receives at least $\frac{\log(n)}{2 \log(\log(n))}$ plateau balls. By \cref{claim:na-threshold}, these indicator variables are NA. By \cref{claim:individual-bin}, it holds that $E[\mathcal{B}_i] \ge \frac{1}{en^{0.5}}$ for sufficiently large $n$ where $\frac{1/2}{\csup / 24} = \frac{12}{\csup} \le \log(\log(n))$. Finally, we can upper-bound the probability that $\mathcal{B}_i=0$ for all bins with small enough surplus, of which there are at least $\frac{n^{3/4}}{8}$. Using marginal probability bounds for NA variables shown in Corollary 3 of \cite{wajc2017negative}, all such $\mathcal{B}_i=0$ with probability at most $(\frac{1}{en^{0.5}})^{\frac{n^{3/4}}{8}}$.
\end{proof}

\textbf{Proving large conditional entropy.} Finally, we show how the existence of a bin with small surplus and many plateau balls implies that the bin has large conditional entropy:

\begin{restatable}[High-entropy conditional]{lemma}{lemmahighentropy}
\label{lemma:high-conditional}
Given a bin $y'$ that has $z_{y'} \le \frac{2}{n \cdot \log(\log(n)) \cdot \log^{2 \close}(n)}$, and receives $\frac{\log(n)}{2 \log(\log(n))}$ plateau balls, then $H(X | Y=y') = \Omega(\log(\log(n)))$.
\end{restatable}
\begin{proof}
To show $H(X | Y=y')$ is large, we first define the vector $v$ such that $v(x) = P(X=x,Y=y')$. Similarly, we define $\overline{v}(x) = \frac{v}{P(Y=y')}$, meaning $\overline{v}(x) = P(X=x | Y=y')$ and $|\overline{v}|_1 = 1$. Our underlying goal is to show $H(\overline{v})$ is large. To accomplish this, we will split the probability mass of $v$ into three different vectors $\vinit, \vplat, \vsur$ such that $v = \vinit + \vplat + \vsur$. The entries of $\vplat$ will correspond to mass from plateau states of $X$, $\vinit$ will correspond to the first $\mathcal{T}$ mass from non-plateau states of $X$, and $\vsur$ will correspond to mass that contributes to the surplus $z_{y'}$. We more formally define the three vectors as follows:

\begin{itemize}
    \item $\vplat$. The vector of probability mass from plateau states of $X$. $\vplat(x)$ is $0$ if $x \notin S$ and $\vplat(x) = P(X=x,Y=y')$ if $x \in S$.
    \item $\vinit$. For non-plateau states of $X$, their first $\mathcal{T}$ probability mass belongs to $\vinit$. $\vinit(x)= \min (P(X=x,Y=y'),\mathcal{T})$ if $x \notin S$ and $\vinit(x)=0$ otherwise.
    \item $\vsur$. For non-plateau states of $X$, their probability mass beyond the first $\mathcal{T}$ mass belongs to $\vsur$. This corresponds to the surplus quantity. $\vsur(x)= \max (0,P(X=x,Y=y')-\mathcal{T})$ if $x \notin S$ and $\vsur(x)=0$ otherwise. By this definition, $z_{y'} = |\vsur|_1$.
\end{itemize}

To show $H(X|Y=y')=H(\overline{v})$ is large, we divide our approach into two steps:

\begin{enumerate}
    \item Show there is substantial helpful mass: $|\vinit + \vplat|_1 = \Omega \left( \frac{1}{n \cdot \log(\log(n)) \cdot \log^{2 \close}(n)} \right)$
    \item Show the distribution of helpful mass has high entropy: $H \left( \frac{\vinit + \vplat}{|\vinit + \vplat|_1} \right) = \Omega (\log(\log(n)))$. 
    \item Show that, even after adding the hurtful mass, the conditional entropy is large: $H(X | Y=y') = H(\overline{v}) \ge H \left( \frac{\vinit + \vplat}{|\vinit + \vplat|_1} \right) - O(1) = \Omega (\log(\log(n)))$
\end{enumerate}

In the first step, we are showing that the distribution when focusing on just the helpful mass of $\vinit, \vplat$ has high a substantial amount of probability mass. In the second step, we prove how this distribution of helpful mass has high entropy. In the third step, we show that the hurtful mass of $\vsur$ does not decrease entropy more than a constant.

First, we show that there is a substantial amount of helpful mass:
\begin{claim}
\label{claim:helpful-mass}
$|\vinit + \vplat|_1 = \frac{1}{2 \clb n \cdot \log(\log(n)) \cdot \log^{2 \close}(n)} $
\end{claim}
\begin{proof}
Recall that the bin $y'$ received $\frac{\log(n)}{2\log(\log(n))}$ plateau balls. As defined in \cref{lemma:plateau}, the set $S$ of plateau states is defined such that $\frac{\max_{x \in S} P(X=x)}{\min_{x \in S} P(X=x)} \le \log^{\close}(n)$ and $\min_{x \in S}P(X=x) \ge \frac{1}{\clb n \log(n)}$. Also recall that by \cref{lemma:e-bound} the most probably state of $E$ has large probability. In particular, $P(E=e_1) \ge \frac{1}{\log^{\close}(n)}$. Let the subset $S' \subseteq S$ be the subset of plateau states of $X$ such that their plateau ball is mapped to $y'$. In particular, for every $x \in S'$ it holds that $f(x,e_1)=y'$. Accordingly, $P(X=x,Y=y') \ge P(X=x)\cdot P(E=e_1)$ for $x \in S'$. Thus, the total weight from plateau states of $X$ is at least $|S'| \cdot \min_{x \in S'} P(X=x) \cdot P(E=e_1) \ge |S'| \cdot \frac{\max_{x \in S'} P(X=x)}{\log^{\close}(n)} \cdot P(E=e_1) \ge \frac{1}{2\clb n \log(\log(n)) \log^{2 \close}(n)}$.
\end{proof}

Next, we show the distribution of helpful mass has high entropy:
\begin{claim}
$H \left( \frac{\vinit + \vplat}{|\vinit + \vplat|_1} \right) \ge \frac{\log(\log(n))}{4}$ 
\end{claim}
\begin{proof}
Let us define $\overvhelp = \frac{\vinit + \vplat}{|\vinit + \vplat|_1}$ to be the vector of helpful mass, and we will show $H(\overvhelp)$ is large by upper-bounding $\max_x \overvhelp(x)$. 

For non-plateau states of $X$, it follows from \cref{claim:helpful-mass} that $\max_{x \notin S} \overvhelp(x) \le \frac{\mathcal{T}}{|\vinit + \vplat|_1} \le \frac{\mathcal{T}}{\frac{1}{2 \clb n \cdot \log(\log(n)) \cdot \log^{2 \close}(n)}} = \frac{24 \clb \log(\log(n)) \cdot \log^{2 \close}(n)}{\log(n)}$. 

For plateau states of $X$, in \cref{claim:helpful-mass} we also developed the lower-bound of $|\vinit + \vplat|_1 \ge |S'| \cdot \frac{\max_{x \in S'} P(X=x)}{\log^{\close}(n)} \cdot P(E=e_1) \ge \frac{\log(n) \cdot \max_{x \in S'} P(X=x)}{2\log^{2 \close}(n) \log(\log(n))}$. Accordingly, we can upper-bound $\max_{x \in S'} \overvhelp(x) \le \frac{\max_{x \in S'} P(X=x)}{|\vinit + \vplat|_1} \le \frac{2 \log(\log(n)) \log^{2 \close}(n)}{\log(n)}$.

Accordingly, we can lower-bound the entropy of $H(\overvhelp) = \sum_{x} \overvhelp(x) \cdot \log(\frac{1}{\overvhelp(x)}) \ge \sum_{x} \overvhelp(x) \cdot \log(\frac{1}{\max_{x'}\overvhelp(x')}) = \log(\frac{1}{\max_{x'}\overvhelp(x')}) \ge \log(\frac{24 \clb \log(n)}{\log^{2 \close}(n) \log(\log(n))}) = (1 - 2 \close) \log(\log(n)) - \log(\log(\log(n))) - \log(24 \clb) = \frac{\log(\log(n))}{2} - \log(\log(\log(n))) - \log(24 \clb) \ge \frac{\log(\log(n))}{4}$ for sufficiently large $n$ where $\frac{\log(\log(n))}{2} \ge \log(\log(\log(n))) + \log(24 \clb)$.
\end{proof}

Finally, we show the hurtful mass does not decrease entropy much, and thus our conditional distribution has high entropy:
\begin{claim}
\label{claim:helpful-entropy}
$H(X | Y=y') = H(\overline{v}) \ge \Omega(1) \cdot H \left( \frac{\vinit + \vplat}{|\overvinit + \overvplat|_1} \right) - O(1) = \Omega (\log(\log(n)))$
\end{claim}
\begin{proof}
We lower-bound $H(\overline{v})$ with the main intuitions that $H\left(\frac{\vinit + \vplat}{|\vinit + \vplat|_1}\right) = \Omega(\log(\log(n)))$ and $\frac{|\vinit + \vplat|_1}{|\vinit + \vplat + \vsur|_1} = \Omega(1)$. We more precisely obtain this lower-bound for $H(\overline{v})$ as follows:
\begin{align}
    & H(\overline{v}) = H\left(\frac{\vinit + \vplat + \vsur}{|\vinit + \vplat + \vsur|_1}\right) \nonumber\\ 
    & = \sum_x \frac{\vinit(x) + \vplat(x) + \vsur(x)}{|\vinit + \vplat + \vsur|_1} \times \nonumber\\
    & \qquad\log\frac{|\vinit + \vplat + \vsur|_1}{\vinit(x) + \vplat(x) + \vsur(x)} \nonumber \\ 
    & \ge \sum_x \frac{\vinit(x) + \vplat(x)}{|\vinit + \vplat + \vsur|_1} \times  \nonumber\\
    & \qquad\log\frac{|\vinit + \vplat + \vsur|_1}{\vinit(x) + \vplat(x)} - 2\label{step:decrease} \\ 
    & \ge \sum_x \frac{\vinit(x) + \vplat(x)}{|\vinit + \vplat + \vsur|_1} \times \nonumber\\
    & \qquad\log\frac{|\vinit + \vplat|_1}{\vinit(x) + \vplat(x)} \nonumber -2 \\ 
    & = \frac{|\vinit + \vplat|_1}{|\vinit + \vplat + \vsur|_1} H\left(\frac{\vinit + \vplat}{|\vinit + \vplat|_1}\right) \nonumber\\
    &-2\nonumber\\
    & = \frac{|\vinit + \vplat|_1}{|\vinit + \vplat|_1 + z_{y'}} H\left(\frac{\vinit + \vplat}{|\vinit + \vplat|_1}\right) -2\nonumber\\
    & \ge  \frac{1}{1+2\clb} \cdot H\left(\frac{\vinit + \vplat}{|\vinit + \vplat|_1}\right) -2  \label{step:use-mass} \\
    & = \Omega(\log(\log(n))) \label{step:use-entropy}
\end{align}

To obtain Step \ref{step:decrease}, we note that all summands are manipulated from the form $\sum_x p_x \log(\frac{1}{p_x})$ to $\sum_x p'_x \log(\frac{1}{p'_x})$ where $p'_x \le p_x$ for all $x$. As the derivative of $p \log(\frac{1}{p})$ is non-negative for $0 \le p \le \frac{1}{e}$, the value of at most two summands can decrease, and they can each decrease by at most one. To obtain Step \ref{step:use-mass}, we use \cref{claim:helpful-mass}. To obtain Step \ref{step:use-entropy}, we use \cref{claim:helpful-entropy}.
\end{proof}

Thus, we have shown $H(X|Y=y') = \Omega(\log(\log(n)))$.
\end{proof}
\begin{corollary}
Under our assumptions, $H(X|Y=y')=\Omega(\log(\log(n)))$ and thus $H(\tilde{E}) = \Omega(\log(\log(n)))$.
\end{corollary}

\section{Proof of \cref{theorem:dag-oracle}}
\label{app:dag-oracle}
\subsection{Proof Outline}

\label{subsection:graph-overview}
For much of this proof, we follow intuitions and use terminology from the proof of \cref{theorem:pairwise}. Consider a pair of variables $X$ and $Y$ such that $X$ is a source and there is a path from $X$ to $Y$. We aim to show that $\mec (Y | X) < \mec (X|Y)$. It is simple for us to show that $\mec (Y|X) = o(\log(\log(n)))$. To show $\mec (X|Y) = \Omega(\log(\log(n)))$, we will use an approach similar to \cref{theorem:pairwise} in that we will show existence of a state $y'$ of $Y$ such that $H(X|Y=y')$ is large. For showing there is a large $H(X|Y=y')$, we will show that there is a $y'$ where its surplus is small and it receives many plateau balls. While we can factor $Y$ as a function of $X$ and small-entropy $E$ (i.e., $Y=f(X,E)$), this is \emph{not} a uniformly random function so we cannot simply apply the result of \cref{theorem:pairwise}. In fact, a key difficulty is that this graph setting with more than two variables results in correlations between mappings. For example, in a graph such as the line graph (Figure \ref{fig:line}) with each node being a uniformly random deterministic function of its parents, one can show that conditioning on $f_{Y}(f_{X_2}(X = x))=y$ almost doubles the probability that $f_{Y}(f_{X_2}(X = x'))=y$. Our new proof method must be able to withstand the dependencies that are introduced by this setting.

To provide some intuition, we give a very high-level overview for how to show existence of a large $H(\Xsrc | Y=y')$ for two particular graphs, and we then expand to generalize these intuitions. 

First, we consider the line graph. For simplicity, suppose that all nodes are deterministic functions of their parents (i.e., all $H(E_i)=0$). Using the method from \cref{theorem:pairwise}, we can see that there exists a large $H(\Xsrc | X_2 = x_2')$. This is because we can show there is a bin of $X_2$ that has small surplus and receives $\Omega(\frac{\log(n)}{\log(\log(n))})$ balls. However, this analysis is actually loose in a sense. For a $c$ where $0<c<1$, we can actually show there are $n^c$ such bins that have small surplus and receive $\Omega(\frac{\log(n)}{\log(\log(n))})$ plateau balls. Now, when we look at how $X_2$ is mapped to $Y$, each of the bins of $X_2$ will ``stick together.'' More formally, each bin of $X_2$ will have all of its mass mapped together to a uniformly random state of $Y$. This is because it is a deterministic function, but our proof will utilize a similar idea for when the function is not deterministic but the entropy is still small. It is then our hope that a good fraction of the bins with our desired properties (small surplus and many plateau balls) at $X_2$, will be mapped to a state of $Y$ that does not have much surplus. In this sense, we have ``heavy bins'' and a non-negligible proportion of them are ``surviving'' from one node to the next because they aren't mapped to a bin with too much surplus. Through careful analysis, we are able to show that at least one such bin survives to the node of $Y$, and thus $H(X|Y=y')$ is large. This proof method would hold if we extend this line graph to any constant length.

Second, we consider the diamond graph (Figure \ref{fig:diamond}). Again, we assume all functions are deterministic for simplicity. Recall that for the line graph, our proof method was to show that there were many heavy bins at $X_2$, and then some heavy bins kept ``sticking together'' and ``surviving'' until we reached $Y$. This was because if two states of $X$ were mapped to the same state of $X_2$, then they would ``stick together'' and would always be mapped to the same state for later nodes (e.g. if $f_{X_2}(x)=f_{X_2}(x')$ then $f_{X_3}(f_{X_2}(x)) = f_{X_3}(f_{X_2}(x'))$). However, this is very far from what is happening in diamond graph. In diamond graph, observe that $Y=f_Y(X_2,X_3)$. By definition of our graph, $X_2$ and $X_3$ are independent deterministic functions of $X$. Two states $x$ and $x'$ of $X$ will be mapped to $Y$ independently unless both $f_{X_2}(x)=f_{X_2}(x')$ and $f_{X_3}(x)=f_{X_3}(x')$. As these are independent, the probability of this happening is $\frac{1}{n^2}$. Thus, the expected number of pairs that are not mapped to $Y$ independently of each other is $\binom{n}{2} \times \frac{1}{n^2} < \frac{1}{2}$. Accordingly, essentially all states of $X$ will be mapped to a state of $Y$ i.i.d. uniformly randomly. This enables us to more directly use the result and techniques of \cref{theorem:pairwise} and treat $X$ and $Y$ as a bivariate problem. 

While we are able to show how both of these graphs will result in  a large $H(\Xsrc | Y=y')$, we do so very differently. For the line graph we show that there are bins with the properties we desire (small surplus and many plateau balls), that they will ``stick together'' as we move down through the graph, and at least one will ``survive'' to $Y$ and thus $H(\Xsrc | Y=y')$. For the diamond graph we show that when we get to $Y$, almost everything will be mapped independently randomly again, and that we can more directly use our bivariate techniques. There is a strong sense in which these two proof methods are opposites of each other (utilizing probability mass staying together throughout the graph as opposed to being independent at the end), yet we would like one unified approach for handling general graphs. To accomplish this, we introduce the \emph{Random Function Graph Decomposition} to combine intuitions of these two settings into a characterization for all graphs. 

\begin{restatable}[Random Function Graph Decomposition]{definition}{defdecomposition}
For the Random Function Graph Decomposition we specify a source $X$ and a node $Y$ such that there is a path from $X$ to $Y$. We ignore all nodes not along a path from $X$ to $Y$. We define the remaining nodes as the set $\vdecomp$. Then, we consider the nodes of $\vdecomp$ an arbitrary valid topological ordering and color each node as follows:

\begin{itemize}
    \item If $X$ is a parent of the node, or if the node has multiple parents and they are not all the same color, we \underline{create} a new color for this node. 
    \item Otherwise, all of the node's parent(s) have the same color, and this node will \underline{inherit} said color.
\end{itemize}
\end{restatable}

At a high-level, when a new color is created for a node, then everything is being mapped to the node almost-independently (similar to the intuition of the diamond graph). When a node inherits its color, there is a sense in which things ``stick together'' (similar to the intuition of the line graph). Let $\colorroot(Y)$ be the earliest node in any topological ordering that has the same color as $Y$ in the Random Function Graph Decomposition (it can be shown that $\colorroot(Y)$ is unique). We aim to use the Random Function Graph Decomposition to show that everything will be mapped to $\colorroot(Y)$ mostly independently. This will result in there being some bins with our desired properties (small surplus, many plateau balls) at $\colorroot(Y)$. Then, we will show that at least one of these bins survives throughout all bins with the same color from $\colorroot(Y)$ to $Y$, implying existence of a large $H(\Xsrc | Y=y')$. 

In particular, to show that balls are mapped to $\colorroot(Y)$ mostly independently, we introduce the notion of \emph{related} mass. More concretely, we define $\relate_1(x)$ mass as the amount of mass of balls that are ever mapped to the same state as $x$ among any variable. We define $\relate_2(x)$ mass as the amount of mass of balls that are mapped to the same state as $x$ for variables of at least two distinct colors in the Random Function Graph Decomposition. Inductively, we will show there are $\Omega(n)$ plateau balls such that $\relate_1(x) = O(\frac{1}{n})$ and $\relate_2(x) = O(\frac{1}{n^2})$. Moreover, we show that the quantity $\relate_2(x)$ upper-bounds mass that can have some dependence with $x$ in how it is mapped to $\colorroot(Y)$. With this upper-bound on dependence, we are able to use techniques of \cref{theorem:pairwise} to show there are many bins of $\colorroot(Y)$ with many plateau balls and not much surplus. Finally, we show that, within the color of $\colorroot(Y)$ and $Y$, at least one of these bins ``survives'' to $Y$ and accordingly $\mec(\Xsrc | Y)$ is large.
\subsection{Complete Proof}
We must show that for a source $\Xsrc$ and a node $Y$ such that there is a path from $\Xsrc$ to $Y$, $\mec(\Xsrc | Y) > \mec(Y | \Xsrc)$. 

\textbf{Upper bounding $\mec(Y | \Xsrc)$.} It is simple to show that $\mec(Y | \Xsrc)$ is small:

\begin{claim}
$\mec(Y | \Xsrc) \le o(|V| \log(\log(n)))$
\end{claim}
\begin{proof}
$Y$ can be written as a function of $\Xsrc$ and the set of all $E_i$ excluding $E_{\Xsrc}$. As $\Xsrc$ is independent of these $E_i$, and their total entropy is $\sum_i H(E_i) = o(\lvert V\rvert \log(\log(n)))$, the claim holds since $\lvert V\rvert=\mathcal{O}(1)$.
\end{proof}

\textbf{Bounding $\mec(\Xsrc|Y)$ via $H(\Xsrc|Y=y)$.} Our method for lower-bounding $\mec(\Xsrc | Y)$ is substantially more involved. As in \cref{theorem:pairwise}, we will lower-bound it by $\mec(\Xsrc|Y) \ge \max_y H(\Xsrc|Y=y)$ (see \cref{theorem:pairwise} for proof). Our proof aims to show there is a conditional entropy such that $\max_y H(\Xsrc|Y=y) = \Omega(\log(\log(n)))$.

\textbf{Showing existence of a near-uniform plateau.} A key step in our approach, as in the proof of \cref{theorem:pairwise}, is that we will find a subset of the support of $X$ whose probabilities are multiplicative close to one another. In particular, we will find a subset of $\Xsrc$ where their probabilities are within a factor of $\log^{\close}(n)$ of each other, where $0 < \close < 1$. For our analysis, we require a value of $\close$ that is $\Omega(1)$ yet below some threshold. While there are multiple values of $\close$ that satisfy this condition, we will use $\close =1/4$. This set of states of $\Xsrc$ that are multiplicatively close to one another will be called the \emph{plateau} of $\Xsrc$. We use \cref{lemma:plateau} proven in \cref{theorem:pairwise} to show how the $(\Omega(n),\Omega(\frac{1}{n \log(n)}))$-support assumptions implies a plateau of states of $X$:

\lemmaplateau*

\textbf{Characterization as a balls-and-bins game.} Our proof method of \cref{theorem:pairwise} characterizes a balls-and-bins game where states of $X \times E$ are balls and states of $Y$ are bins. As we realized an entry $f(x,e)$ as a uniformly random state of $Y$, we characterized this as a ball (a state of $X \times E$) being assigned to a uniformly random bin (a state of $Y$). In the graph setting of this theorem, such a characterization is more complicated. Any node $X_i$ is a uniformly random function of $\pa(X_i)$ and $E_i$. We define $E^*$ to be the Cartesian product of all $E_i$ other than $E_X$. Using this, we characterize balls as being states of $X \times E^*$. Note how any random variable in our SCM is a deterministic function of $X \times E^*$. In particular, it is the composition of (potentially many) $f_i$ terms. For simplicity of notation, we let $f^*_{T}(x \times e^*)$ denote the value of a set of variables $T$ for a particular state of $x \times e^*$. In the characterization of our balls-and-bins game, all balls with the same configuration of $\pa(X_i)$ and $E_i$ are mapped uniformly randomly together to a state of $X_i$. In other words, configurations are realized i.i.d. uniformly randomly. Using our notation, this means two balls $(x_a,e_a^*)$ and $(x_b,e_b^*)$ are mapped independently to variable $X_i$ if any only if $f^*_{\pa(X_i) \cup E_i}(x_a,e_a^*) \ne f^*_{\pa(X_i) \cup E_i}(x_b,e_b^*)$. 

\textbf{Lower-bounding the most probable state of $E_i$ and $E^*$.} We focus first on \emph{plateau balls}, which are balls corresponding to states of $S$ (the set of plateau states of $X$) and the highest probability state of $E^*$. In particular, they are balls of the form $(X \in S, E^*=e^*_1)$ where $e^*_1$ is the most probable state of $E^*$. To show that these plateau balls have enough probability mass to be helpful, we first use \cref{lemma:e-bound} proven in \cref{theorem:pairwise} that implies all $\max\limits_e H(E_i = e) \ge \frac{1}{\log^{\close}(n)}$: 

\lemmaebound*

This implies a lower-bound on the probability $P(E^*=e^*_1)$:

\begin{lemma}
\label{lemma:e-star-bound}
If all $\max\limits_e P(E_i = e) \ge \frac{1}{\log^{\close}(n)}$, then $P(E^* = e^*_1) \ge \frac{1}{\log^{\close |V|}(n)}$.
\end{lemma}
\begin{proof}
As $E^*$ is the Cartesian product of $|V|-1$ variables $E_i$, it holds that $\max P(E^* = e^*) \ge (\min_i \max_e P(E_i=e))^{|V|-1} \ge (\frac{1}{\log^{\close}(n)})^{|V|-1} \ge \frac{1}{\log^{\close |V|}(n)}$.
\end{proof}

\textbf{Introducing surplus.} In \cref{theorem:pairwise}, we prove how there exists a bin that receives a large amount of mass that helps the bin have large conditional entropy (such helpful mass includes the plateau balls), and not much mass that hurts the conditional entropy making it small. To formalize this hurtful mass, we introduced the \emph{surplus} quantity described in \cref{definition:surplus-generic}. This surplus is a way of quantifying the probability mass received by a state of $Y$ that is hurtful towards making the conditional entropy large. The proof of \cref{theorem:pairwise} achieves a lower-bound for $\max_y H(X|Y=y)$ by proving existence of a state $y'$ of $Y$ where $y'$ receives many plateau balls and the surplus is small. Likewise, we will also prove existence of such a state of $Y$ with many plateau balls and small surplus, in the graph setting. We formalize the notion of surplus as follows:

\begin{restatable}[Surplus, $\mathcal{T}=\frac{120}{n \log(n)}$]{definition}{defsurplusgraph}
We define the surplus of a state $i$ of $Y$ as $z_i = \sum_{j \notin S} \max (0, P(X=j, Y=i) - \frac{120}{n \log(n)})$.
\end{restatable}

\textbf{Introducing the Random Function Graph Decomposition.} In \cref{subsection:graph-overview}, we introduced intuitions from considering the diamond graph in Figure \ref{fig:diamond} and the line graph in Figure \ref{fig:line}. In the proof outline for a diamond graph, we utilized the intuition that almost all balls were independently assigned to $Y$. This enables us to use techniques from \cref{theorem:pairwise}, as almost all balls were independently assigned to a uniformly random state of $Y$, closely mirroring the setting of \cref{theorem:pairwise}. In the proof outline for line graph, we used techniques of \cref{theorem:pairwise} to show that there would be many bins that received many plateau balls and small surplus. Then, we showed that at least one of these bins would mostly ``survive'' and remain in-tact to $Y$. While our intuitions for both of these graphs enabled us to show existence of a large $H(\Xsrc | Y=y)$, but they did so with near-opposite methods. Our intuition for the diamond graph exploits independence (everything is assigned almost independently to $Y$), while our intuition for the line graph exploits dependence (some bins with our desired properties ``survive'' from the second node onwards). We introduce the \emph{Random Function Graph Decomposition} to combine intuitions of these two graphs into a characterization for all graphs:

\defdecomposition*

At a high-level, when a new color is created for a node, then we will see that plateau balls are being mapped to the node almost-independently (similar to the intuition of the diamond graph). When a node inherits its color, there is a sense in which things ``stick together'' (similar to the intuition of the line graph). For some node $X_i \in \vdecomp$, we define $\colored(X_i)$ to be the node's color in the Random Function Graph Decomposition. Under a fixed topological ordering, let $\colorroot(Y)$ be the earliest node that has the same color as $Y$ in the Random Function Graph Decomposition (it can be shown that $\colorroot(Y)$ is unique). We aim to use the Random Function Graph Decomposition to show that everything will be mapped to $\colorroot(Y)$ mostly independently. This will result in there being some bins with our desired properties (small surplus, many plateau balls) at $\colorroot(Y)$. Then, we will show that at least one of these bins survives throughout all bins with the same color from $\colorroot(Y)$ to $Y$, implying existence of a large $H(\Xsrc | Y=y')$. 

\textbf{Introducing related mass.} To show how plateau balls are mapped to $\colorroot(Y)$ mostly independently, we introduce the concept of \emph{related} mass. Related mass introduces a measure of how much mass has come into contact with a particular plateau ball:

\begin{restatable}[Related mass]{definition}{defrelated}
\begin{itemize}
    We define related mass of two types as follows.
    
    \item For a plateau state $x$ of $X$, we define $\relate_1(x)$ mass as the amount of mass of balls from non-plateau states of $X$ that are ever mapped to the same state as the plateau ball of $x$ among any variable in the Random Function Graph Decomposition. In other words, $x',e^*$ contributes to $\relate_1(x)$ if it satisfies the following for some $X_i$: $x'$ together with some realization $e^*$ contributes to the same bin of $X_i$ that $x$ is mapped to together with $e_1^*$. More formally, we define $\mathcal{B}_{1}(x)$ as the set of balls whose mass counts towards $\relate_1(x)$, where  $\mathcal{B}_1(x)= \{x' \in X \backslash S, e^* \in E^* | \exists X_i \in \vdecomp \textrm{s.t.} f^*_{X_i}(x',e^*) = f^*_{X_i}(x,e^*_1) \}$. Accordingly, $\relate_1(x) = \sum_{x',e^* \in \mathcal{B}_1(x)} P(X=x') \cdot P(E^* = e^*)$.
    \item For a plateau state $x$ of $X$, we define $\relate_2(x)$ mass as the amount of mass of balls from non-plateau states of $X$ that are ever mapped to the same state as the plateau ball of $x$ among variables of at least two distinct colors in the Random Function Graph Decomposition. In other words, $x',e^*$ contributes to $\relate_2(x)$ if it satisfies the following for some $X_i,X_j$ with distinct colors: $x'$ together with some realization $e^*$ contributes to the same bin of $X_i$ that $x$ is mapped to together with $e_1^*$; same holds for $X_j$. More formally, we define $\mathcal{B}_{2}(x)$ as the set of balls whose mass counts towards $\relate_2(x)$, where $\mathcal{B}_2(x)= \{x' \in X \backslash S, e^* \in E^* | \exists X_i, X_j \in \vdecomp \textrm{s.t.} f^*_{X_i}(x',e^*) = f^*_{X_i}(x,e^*_1), f^*_{X_j}(x',e^*) = f^*_{X_j}(x,e^*_1), \colored(X_i) \ne \colored(X_j) \}$. Accordingly, $\relate_2(x) = \sum_{x',e^* \in \mathcal{B}_2(x)} P(X=x') \cdot P(E^* = e^*)$.
\end{itemize}
\end{restatable}

Now, we will consider an arbitrary topological ordering of $\vdecomp$. In this ordering, we define $\ord(X_i)$ for $X_i \in \vdecomp$ as the index of $X_i$ in the topological ordering. We introduce a modification of $\relate_1(x)$ where $\relate_1^{\ord(X_i)}(x)$ only considers nodes of $\vdecomp$ that are strictly earlier in the topological ordering than $X_i$. We define $\relate_2^{\ord(X_i)}(x)$ analogously. It is our goal to show that there are many plateau states $x \in S$ such that $\relate_2^{\ord(\colorroot(Y))}(x)$ is small. This will enable us to show how there are many plateau balls that are mapped to $\colorroot(Y)$ independently of almost all other mass.

\textbf{Upper-bounding related mass.} To show independence in how some plateau balls are mapped to $\colorroot(Y)$, we bound $\relate_2^{\ord(\colorroot(Y))}(x)$ for some plateau states $x \in S$.

To show this, we will process nodes in the topological ordering. After processing the first $i$ nodes, we will argue that there is a large set $S_{\textrm{indep}}^i$ with upper-bounds on all $\relate_1^i(x)$ and $\relate_2^i(x)$.

\begin{lemma}
\label{lemma:inductive-related}
With high probability, after processing the first $i$ nodes in the topological ordering, there exists a set $S_{\textrm{indep}}^i$ such that $|S_{\textrm{indep}}^i| = \frac{|S|}{6^i}$, all $x \in S_{\textrm{indep}}^i$ satisfy $\relate_1^i(x) \le \frac{6 i}{n}$ and $\relate_2^i(x) \le \frac{18 \times i \times (i-1)}{n^2}$, and all $x,x' \in S_{\textrm{indep}}^i$ satisfy $f^*_{X_j}(x,e^*_1) \ne f^*_{X_j}(x',e^*_1)$ for all $1 \le j \le i$.
\end{lemma}
\begin{proof}
We begin with the following claims.
\begin{claim}
\cref{lemma:inductive-related} holds for $i=1$.
\end{claim}
\begin{proof}
To find a subset of $S$ to be $S_{\textrm{indep}}^i$, we will choose any arbitrary subset of size $\frac{S}{6}$. By definition of $\relate_1$ and $\relate_2$, all plateau balls are different states of $\Xsrc$ so $\relate_1^1(x)=\relate_2^1(x)=0$ for every $x \in S$, and $f^*_{\Xsrc}(x,e^*_1) \ne f^*_{\Xsrc}(x',e^*_1)$ for all $x,x' \in S$.
\end{proof}

\begin{claim}\label{claim:11}
\cref{lemma:inductive-related} holds for $i$ if it holds for all $j<i$. 
\end{claim}
\begin{proof}
First, we realize $f_{X_i}$ for all cells other than those corresponding to configurations of $\pa(X_i) \cup E_{X_i}$ that contain an element of $S_{\textrm{indep}}^{i-1}$. Now, we consider the process of realizing the entries of $f_{X_i}$ corresponding to elements of $S_{\textrm{indep}}^{i-1}$ in an arbitrary order. We define a random variable for every element of $S_{\textrm{indep}}^{i-1}$. For the $j$-th element, we define $\mathcal{S}_j$ as follows:
\begin{itemize}
    \item If the element is mapped to a bin that another element of $S_{\textrm{indep}}^{i-1}$ has been mapped to, then $\mathcal{S}_j = -1$.
    \item Otherwise, if the element $x \in S_{\textrm{indep}}^{i-1}$ is mapped to a bin that contains total mass at least $\frac{6}{n}$, or total mass from $\mathcal{B}_1^{i-1}(x)$ of at least $\frac{6 \relate_1^{i-1}(x)}{n}$, then $\mathcal{S}_j = 0$.
    \item Else, then $\mathcal{S}_j = 1$. 
\end{itemize}

The intuition behind $\mathcal{S}_j$ is that we will count an element of $x \in S_{\textrm{indep}}^{i}$ as being eligible for $S_{\textrm{indep}}^{i}$ if it lands in a bin with no other value of $S_{\textrm{indep}}^{i-1}$, and if it lands in a bin that will not increase $\relate_1^i(x)$ or $\relate_2^i(x)$ by too much. 

\begin{claim}
\label{claim:set-candidates}
Consider the set comprised of each element $x \in S_{\textrm{indep}}^{i-1}$ that satisfies the following. Suppose $x$ is assigned to a bin such that before $x$ is mapped to the bin, the bin has total mass at most $\frac{6}{n}$ and total mass intersecting from $\mathcal{B}^{i-1}(x)$ of at most $\frac{6 \relate_1^{i-1}(x)}{n}$. Moreover, suppose $x$ is the only element of $S_{\textrm{indep}}^{i-1}$ that is ever assigned to this bin. Then, this set of all such $x$ would meet the desired properties required of $ S_{\textrm{indep}}^{i}$.
\end{claim}
\begin{proof}
The increase of the quantity $\relate_1^i(x)$ is bounded by the amount of other mass in the bin that $x$ is assigned to. Accordingly, $\relate_1^i(x) \le \relate_1^{i-1}(x) + \frac{6}{n} \le \frac{6 \times (i-1)}{n} + \frac{6}{n} = \frac{6 i}{n}$. The increase of the quantity $\relate_2^i(x)$ is bounded by the amount of mass from $\mathcal{B}_1^{i-1}(x)$ in the bin $x$ is assigned to. Accordingly, $\relate_2^i(x) \le \relate_2^{i-1}(x) + \frac{6\relate_1^{i-1}(x)}{n} \le \frac{18 \times (i-1) \times (i-2)}{n^2} + \frac{36 (i-1)}{n^2} = \frac{18 \times i \times (i-1)}{n^2}$
\end{proof}

Moreover, we claim that $\sum \mathcal{S}_i$ serves as a lower bound for the set of elements eligible for $S_{\textrm{indep}}^i$ referenced in \cref{claim:set-candidates}.
\begin{claim}
The number of elements of $S_{\textrm{indep}}^{i-1}$ that are eligible for $S_{\textrm{indep}}^{i}$ by satisfying \cref{claim:set-candidates} is at least $\sum_j \mathcal{S}_j$.
\end{claim}
\begin{proof}
For each bin, consider the sum of $\mathcal{S}_j$ for variables corresponding to elements of $S_{\textrm{indep}}^{i-1}$ that were assigned to the bin (if any). If the sum is nonpositive, then we trivially claim the set of elements meeting the criteria in this bin is at least the sum, as there will be at least $0$ such elements. Otherwise, the sum must be $1$, This implies there is exactly one element of $S_{\textrm{indep}}^{i-1}$ assigned to the bin, and that it met the criteria when it was assigned, because its corresponding $\mathcal{S}_j=1$. Moreover, as no other elements could have been assigned to the bin later, it still meets the criteria. Combining both cases, we see that the sum of $\mathcal{S}_j$ for each bin is a lower-bound for the number of elements satisfying the criteria in said bin, and thus globally the sum of all $\mathcal{S}_j$ is a lower-bound for how many elements meet the criteria in total. 
\end{proof}

We aim to now use the sum of $\mathcal{S}_j$ as a lower-bound for the size of the set of elements meeting the criteria. To do so, we will first lower-bound $E[\mathcal{S}_j]$. 

\begin{claim}
Regardless of the realization of any previous randomness, $E[\mathcal{S}_j] \ge \frac{1}{3}$.
\end{claim}
\begin{proof}
$\mathcal{S}_j$ is equal to $-1$ only if it is assigned to a bin with another element of $S_{\textrm{indep}}^{i-1}$. The number of such bins is upper-bounded by $|S_{\textrm{indep}}^{i-1}| \le |S_{\textrm{indep}}^{1}| \le \frac{n}{6}$. Otherwise, $\mathcal{S}_j$ is equal to $0$ only if the bin had mass at least $\frac{n}{6}$ or it has mass from the corresponding $\mathcal{B}_1^{i-1}(x)$ of at least $\frac{6\relate_1^{i-1}(x)}{n}$. There can only be at most $\frac{n}{6}$ bins satisfying the former, and at most $\frac{n}{6}$ bins satisfying the latter. Accordingly, there are at least $n - 3 \times \frac{n}{6} = \frac{n}{2}$ where if the corresponding element is assigned to it, then $\mathcal{S}_j=1$. Hence $E[\mathcal{S}_j] \ge \frac{1}{2} - \frac{1}{6} = \frac{1}{3}$.
\end{proof}

As we need a set $S_{\textrm{indep}}^{i}$ with cardinality $|S_{\textrm{indep}} ^{i}| = \frac{|S_{\textrm{indep}}^{i-1}|}{6}$, we show the following:

\begin{claim}
\label{claim:s-sum}
$\sum_j \mathcal{S}_j \ge \frac{|S_{\textrm{indep}}^{i-1}|}{6}$ with high probability.
\end{claim}
\begin{proof}
We will modify the variables to make a martingale and then utilize Azuma's inequality. We define $\mathcal{S}'_j = \mathcal{S}'_{j-1} + \mathcal{S}_j - E[(\mathcal{S}_j | \mathcal{S}_1 ,\dots, \mathcal{S}_{j-1})]$. Accordingly, the sequence of $\mathcal{S}'$ is a martingale of length $|S_{\textrm{indep}}^{i-1}|$ where $|\mathcal{S}'_{j-1} - \mathcal{S}'_{j}| \le 1$. Thus, we can use Azuma's inequality to show $P(|\mathcal{S}'_{|S_{\textrm{indep}}^{i-1}|} - \mathcal{S}'_1| \ge \frac{|S_{\textrm{indep}}^{i-1}|}{6}) \le 2e^{\frac{-|S_{\textrm{indep}}^{i-1}|}{72}} = 2e^{\frac{-|S|}{72 6^{i-1}}} = 2e^{-\Omega(n)}$. By definition, $\sum_j \mathcal{S}_j = \sum_j \mathcal{S}'_j + \sum_j E[S_j]$. By our result with Azuma's inequality, we then claim that with high probability it holds that $\sum_j \mathcal{S}_j \ge -\frac{|S_{\textrm{indep}}^{i-1}|}{6} + \frac{|S_{\textrm{indep}}^{i-1}|}{3} = \frac{|S_{\textrm{indep}}^{i-1}|}{6}$. 
\end{proof}

Combining \cref{claim:set-candidates} and \cref{claim:s-sum}, we have now shown that there exists a valid set $S_{\textrm{indep}}^{i}$ of size $|S_{\textrm{indep}}^{i}| = \frac{|S_{\textrm{indep}}^{i-1}|}{6}$, completing the proof of Claim \ref{claim:11}.
\end{proof}
By induction Lemma \ref{lemma:inductive-related} holds.

\end{proof}

\begin{corollary}
\label{corollary:s-indep}
There exists a subset of plateau states $S_{\textrm{indep}} \subseteq S$ such that $|S_{\textrm{indep}}| \ge \frac{|S|}{6^{|V|}} = \Omega(n)$ and every $x \in S_{\textrm{indep}}$ satisfies $\relate_2^{\ord(\colorroot(Y))}(x) \le \frac{18 \times |V| \times (|V|-1)}{n^2}$ . Moreover, for all pairs $x,x' \in S_{\textrm{indep}}$ it holds that they never share a state, meaning $f^*_{X_i}(x) \ne f^*_{X_i}(x')$ for all $X_i \in \vdecomp$.
\end{corollary}
\begin{proof}
One such set is simply $|S_{\textrm{indep}}^{|V|}|$ as shown in \cref{lemma:inductive-related}.
\end{proof}

\textbf{Characterizing balls.} Recall that each variable assigns balls with the same configuration of its parents and exogenous variable together. We aim to show a similar result at $\colorroot(Y)$. To do so, we will characterize the balls within configurations into types:

\balltypes*

We use $\mathcal{T}=\frac{120 |V|}{n \log (n)}$.

Now, we will show that for every variable $X_i$ there are many bins without too much surplus, such that the plateau configurations have many bins that they may be assigned to that will help us obtain a bin with small surplus and many plateau configurations.

\begin{restatable}[Configuration and ball characterizations]{definition}{configtypes}
We characterize three types of configurations/balls:
\begin{enumerate}
    \item Large \emph{configurations}. For all configurations of the form $(\pa(X_i) \cup E_i)$ where the configuration has balls of total mass $\ge \frac{\mathcal{T}}{2}$.
    \item Dense \emph{ball}. Consider a set $L$ of states of $\Xsrc$, where a state of $\Xsrc$ is in $L$ if $P(\Xsrc=x) \ge \frac{1}{\log^{3}(n)}$. Dense balls are all balls of the form $(x \in L, e \in E^*)$. We call these dense balls, because the low-entropy of $E$ will prevent the collective mass of these balls from being distributed well throughout.
    \item Small \emph{ball}. For all balls of the form $(x \in \Xsrc \backslash (S \cup L), e \in E^*)$ where the ball has mass $< \frac{\mathcal{T}}{2}$.
\end{enumerate}
\label{def:configtypes}
\end{restatable}

\textbf{Bounding dense ball surplus.}
Recall the following used in \cref{theorem:pairwise} to bound contributions from dense balls:

\lemmanoexpansion*

We use the following corollary:

\corollarydenseball*

These imply the following for our graph setting. While this may seems strictly weaker than \cref{corollary:denseball}, we will utilize that the event of a bin having too much mass from dense balls is now independent from how large configurations are mapped.

\begin{corollary}
\label{corollary:dense-indep}
Let $\mathcal{C}_{\textrm{large}}$ denote the set of large configurations of $\pa(X_i) \cup E_i$ as defined in \cref{def:configtypes}. Let $C$ be a random variable denoting the configuration of the corresponding ball of $\pa(X_i) \cup E^*$. We claim that dense balls in configurations other than $\mathcal{C}_{\textrm{large}}$ are not distributed well throughout $X_i$. In particular, there exists no subset $|T|=n/4$ such that $\forall t \in T: P(X \in L, Y=t, C \notin \mathcal{C}_{\textrm{large}}) \ge \frac{1}{n \log(\log(n)) \log^{2\close}(n)}$.
\end{corollary}
\begin{proof}
Note that the results of \cref{lemma:no-expansion} and \cref{corollary:denseball} still hold in this setting as any $X_i \in \vdecomp$ can be written as a function of $\Xsrc$ and $\cup_{j} E_j$. This corollary trivially follows from \cref{corollary:denseball}, as it is strictly weaker in that we add a restriction that $C \notin \mathcal{C}_{\textrm{large}}$. Any set $T$ that contradicts \cref{corollary:dense-indep} would immediately contradict \cref{corollary:denseball}.
\end{proof}

\textbf{Bounding large configuration surplus.} To bound contribution to surplus by large configurations, we bound the number of bins that receive any mass from large configurations. Recall \cref{lemma:empty-bins} from the proof of \cref{theorem:pairwise}:

\lemmaavoidbig*

Accordingly, we can use the following:

\begin{corollary}
\label{corollary:bigballs}
As there are at most $\frac{1}{\mathcal{T}/2} \le \frac{n \log(n)}{60 |V|} \le \frac{1}{40 |V|} \cdot n \ln(n)$ large balls, with high probability there are at least $\frac{n^{1-\frac{1}{40 |V|}}}{2}$ bins that receive no large balls.
\end{corollary}

\textbf{Bounding small ball surplus.} Here we will bound the surplus from small balls. Note that, while this proof is not short, it is using the same ideas as the corresponding section in the proof of \cref{theorem:pairwise}. However, there are some subtle differences that necessitate a separate proof for the graph setting. We use identical text from the proof of \cref{theorem:pairwise} when applicable.

For the small balls, we will also show that they cannot contribute too much surplus to too many states of any $X_i$. We will notably use that all small balls correspond to a state of $\Xsrc$ where $P(\Xsrc=x) \le \frac{1}{\log^{3}(n)}$. We will utilize this to show that most small balls are assigned to a state of $X_i$ that has not yet received $>\frac{\mathcal{T}}{2}$ mass from its corresponding state of $\Xsrc$, and accordingly would not increase the surplus. To accomplish this, we define a surplus quantity that only takes into account small balls:

\begin{definition}[Small ball surplus]
We define the small ball surplus of a state $y$ of $Y$ as
\begin{align*}
    &z_{j}^{\textrm{small}} = \sum_{x_{\textrm{src}} \notin (S \cup L)} \max \Bigg( 0,- \mathcal{T}+\\
    &\left. \sum\limits_{\substack{\hspace{-0.05in}e^*  :\\\hspace{0.05in} (X=x_{\textrm{src}},E^*=e.C \notin \mathcal{C}_{\textrm{large}})\\}}P(\Xsrc=x, E^*=e, X_i=j)   \right).
\end{align*}
\end{definition}

With this notion of surplus constrained to small balls, we show the following:

\begin{restatable}[Small ball limited surplus]{lemma}{lemmasmallball}
\label{lemma:smallball}
With high probability, there are at most $\frac{n}{4}$ values of $i$, i.e., number of bins, where $z_{i}^{\textrm{small}} \ge \frac{1}{n \log(\log(n)) \log^{2 \close}(n)}$.
\end{restatable}
\begin{proof}
We will consider configurations $C \notin \mathcal{C}_{\textrm{large}}$ in an arbitrary order, and within each configuration consider balls in an arbitrary order. Let $x_i(c)$ be the corresponding state of $X_i$ for the $c$-th configuration, let $x_{{\textrm{src}}_c}(t)$ be the corresponding state of $\Xsrc$ for the $t$-th ball in the $c$-th configuration. $e_{{\textrm{src}}_c}(t)$ be the corresponding state of $E^*$ for the $t$-th ball in the $c$-th configuration, and $w_{c,\textrm{ball}}(t)$ be the $t$-th ball's probability mass in the $c$-th configuration (i.e., $P(\Xsrc=x_{{\textrm{src}}_c}(t),E=e_{{\textrm{src}}_c}(t))$). Moroever, we define $w_{\textrm{config}}(c)$ as the weight of all such balls with configuration $c$. Recall that for all small balls it must hold that $x_{{\textrm{src}}_c}(t) \notin L$ and thus $P(\Xsrc=x_{{\textrm{src}}_c}(t)) < \frac{1}{\log^{3}(n)}$. We define the total small ball surplus as $Z^{\textrm{small}}=\sum_{j \in X_i} z_j^{\textrm{small}}$. Now, we will consider all non-large configurations in an arbitrary order and realize their corresponding entry of $f$ to map them to a state of $X_i$. Initially, we have not realized the entry of $f$ for any balls and thus all $z^{\textrm{small}}_{j}=0$ and $Z^{\textrm{small}}=0$. As we map configurations to states of $X_i$, we define $\Delta(c)$ as the increase of $Z^{\textrm{small}}$ after mapping the $c$-th configuration to a state of $X_i$. By definition, $\sum_c \Delta(c)$ is equal to $Z^{\textrm{small}}$ after all values of $f$ have been completely realized. 

Our primary intuition is that we will show for many small balls it holds that they have zero contribution towards their configuration's quantity $\Delta(c)$. 
As $f$ is realized for each configuration, let $w_{X_i}(x_i,x_{\textrm{src}})$ denote the total mass of balls assigned to state $x_i$ of $X_i$ so far from state $x_{\textrm{src}}$ of $\Xsrc$, i.e., $w_{X_i}(x_i',x_{\textrm{src}}') \coloneqq \sum\limits_{c' < c, t: x_i(c)=x_i'} w_{c,\textrm{ball}}(t')$. Note that this quantity is shared among all configurations.

\begin{claim}
\label{claim:delta-expectation-claim}
Regardless of the realizations of all $\Delta(c')$ for $c'<c$, it holds that $\Delta(c)$ is a random variable with values in range $[0,w_{\textrm{config}}(c)]$ and $E[\Delta(c)] \le \frac{w_{\textrm{config}}(c)}{\log^{2}(n)}$.
\end{claim}

\begin{proof}
Let us define $\Delta_t(c)$ as the contribution of the $t$-th ball to $\Delta(c)$. By definition, $\sum_t \Delta_t(c) = \Delta(c)$. We aim to show $E[\Delta_t(c)] \le \frac{w_{c,\textrm{ball}}(t)}{\log^{2}(n)}$. This would immediately imply the desired bound on $E[\Delta(t)]$ by linearity of expectation.

The only conditions under which $\Delta_t(c)$ takes a non-negative value (which is upper-bounded by $w_{c,\textrm{ball}}(t)$), is when $w_{X_i}(x_i(c),x_{\textrm{src}_c}(t)) > \frac{\mathcal{T}}{2}$ before the entry of $f$ for the $c$-th configuration is realized (other. Recall that $P(x_{\textrm{src}_c}(t)) \le \frac{1}{\log^{3}(n)}$. Accordingly, the number of states $x_i'$ of $X_i$ where $w_{X_i}^c(x_i',x_{\textrm{src}_c}(t)) > \frac{\mathcal{T}}{2}$ is upper-bounded by $\frac{P(\Xsrc = x_{\textrm{src}}))}{\mathcal{T}/2} \le \frac{1/\log^{3}(n)}{60 |V|/(n \log(n))} = \frac{n \log(n)}{60 |V| \log^{3}(n)} \le \frac{n}{\log^{2}(n)}$. This is due to the fact that balls partition the total mass of $P(\Xsrc=x_{\textrm{src}}(t))$ since we have $P(\Xsrc=x_{\textrm{src}}(t))=\sum_eP(X=x_{\textrm{src}}(t),E=e)$. 
This implies that the probability that the $t$-th ball of configuration $c$ will be mapped to a state $x_{i}'$ of $X_i$ such that $w_{X_i}^c(x_i',x_{\textrm{src}_c}(t))$  already exceeds the threshold of $\mathcal{T}/2$ (in other words where we will have $\Delta_t(c)>0$) is upper-bounded by $\frac{n / \log^{2}(n)}{n} = \frac{1}{\log^{2}(n)}$ due to the fact that the function $f$ is realized uniformly randomly. Accordingly, $E[\Delta_t(c)] \le \frac{w_{c,\textrm{ball}}(t)}{\log^{2}(n)}$ and thus $E[\Delta(c)] \le \frac{w_{\textrm{config}}(c)}{\log^{2}(n)}$.
\end{proof}

This enables us to upper-bound the sum of $\Delta(t)$:

\begin{claim}
$\sum_c \Delta(c) \le \frac{1}{4 \log(n)}$ with high probability.
\end{claim}
\begin{proof}

We will transform $\Delta(c)$ into a martingale. In particular, we define $\Delta'(c) = \Delta'(c-1) + \Delta(c) - E[\Delta(c) | \Delta(1), \dots, \Delta(c-1)]$. We define $\Delta'(0)=0$, and note that $\Delta'(c)$ is a martingale. By Azuma's inequality, we show $|\sum_c \Delta'(c)| \le \frac{1}{8 \log(n)}$ with high probability:
\begin{align*}
    & P[|\Delta(c)| > \varepsilon] < 2 e^{- \frac{\varepsilon^2}{2 \sum c_i^2}} \\
    & \le 2 e^{- \frac{\left( \frac{1}{8 \log(n)} \right)^2}{2 (\max_i c_i) \cdot \sum c_i}} \\
    & \le 2 e^{- \frac{\left( \frac{1}{8 \log(n)} \right)^2}{2 \times \mathcal{T}/2 \cdot 1}} \\
    & = 2 e^{\frac{- n \log(n)}{120 \times 8 \times |V| \times \log(n)}} \\
\end{align*}
Accordingly, by definition of $\Delta'(c)$ this implies $|(\sum_c \Delta(c)) - \sum_c E[\Delta(c) | \Delta(1) , \dots, \Delta(c-1)]| \le \frac{1}{8 \log(n)}$. By \cref{claim:delta-expectation-claim} we know all $E[\Delta(c) | \Delta(1) ,\dots, \Delta(c-1)] \le \frac{w_{\textrm{config}(c)}}{\log^2(n)}$ and accordingly, $\sum_c E[\Delta(c) | \Delta(1), \dots, \Delta(c-1)] \le \frac{1}{\log^2(n)}$. Together, these imply $\sum_c \Delta(c) \le \frac{1}{8 \log(n)} + \frac{1}{\log^2(n)}$ with high probability, and for sufficiently large $n$ it holds that $\frac{1}{\log^2(n)} \le \frac{1}{8 \log(n)}$. Thus, our high-probability on $|\Delta'(c)|$ implies that $\sum_t \Delta(t) \le \frac{1}{4 \log(n)}$ with high probability.
\end{proof}

Finally, we conclude that our upper-bound on $\sum_t \Delta(t)$ implies an upper-bound on the number of states of $Y$ with non-negligible small ball support:

\begin{claim}
If $\sum_t \Delta(t) \le \frac{1}{4 \log(n)}$, then there are at most $\frac{n}{4}$ bins where $z_i^{\textrm{small}} \ge \frac{1}{n \log(\log(n)) \log^{2 \close}(n)}$.
\end{claim}
\begin{proof}
By definition, $Z^{\textrm{small}} = \sum_t \Delta(t) \le \frac{1}{4 \log(n)}$. Given this upper-bound for total small ball surplus, we can immediately upper-bound the number of states of $X_i$ with small ball surplus greater than $\frac{1}{n \log(\log(n)) \log^{2 \close}(n)}$ by the quantity $\frac{1 / (4 \log(n))}{1 / (n \cdot \log(\log(n)) \cdot \log^{2\close}(n))} \le \frac{n \cdot \log(\log(n)) \cdot \log^{1/2}(n)}{4 \log(n)} \le \frac{n}{4}$. We obtain this by using $\close = \frac{1}{4}$ and for sufficiently large $n$ such that $\log(\log(n)) \le \log^{1/2}(n)$. %

\end{proof}
This concludes the proof of the lemma.
\end{proof}

\textbf{Concluding many bins with small surplus.} Now, we combine all these intuitions to show there are many bins that have a small amount of surplus. We have shown that, with high probability, the are at most $n/4$ bins with non-negligible mass from dense balls by \cref{corollary:denseball}, and at most $n/4$ bins with non-negligible surplus from small balls from non-large configurations \cref{lemma:smallball}. Combining these sets, there are at most $n/2$ bins with non-negligible mass from dense balls or surplus from small balls. By \cref{corollary:bigballs}, with high probability at least $\frac{n^{1-\frac{1}{40 |V|}}}{2}$ bins will receive no large configurations mapped to it. Our goal is to show the intersection of the sets is large, so there are many bins that have small surplus.  We use \cref{lemma:set-intersection} proven in \cref{theorem:pairwise}:

\lemmasetintersection*
\begin{corollary}
\label{corollary:num-bins}
With high probability, there are at least $\frac{n^{1-\frac{1}{40 |V|}}}{8}$ bins with surplus $z_{y} \le \frac{2}{n \log(\log(n)) \log^{2 \close}(n)}$.
\end{corollary}
\begin{proof}
We have defined three types of balls, and have proven results that show how there are many bins with negligible bad contribution for each type of ball. Now, we combine these with \cref{lemma:set-intersection} to show there are many bins where there is not much bad contribution in total. By \cref{corollary:denseball} there are at most $n/4$ bins with more than $\frac{1}{n \log(\log(n)) \log^{2 \close (n)}(n)}$ mass from dense balls. By \cref{lemma:smallball}, there are at most $n/4$ bins with small ball surplus more than $\frac{1}{n \log(\log(n)) \log^{2 \close (n)}(n)}$. Let $A$ be the set of bins with at most $\frac{1}{n \log(\log(n)) \log^{2 \close (n)}(n)}$ mass from dense balls and at most $\frac{1}{n \log(\log(n)) \log^{2 \close (n)}(n)}$ small ball surplus. By combining \cref{corollary:denseball} and \cref{lemma:smallball} we know $|A| \ge \frac{n}{2}$ with high probability. Let $B$ be the set of bins that receive no large configurations. By \cref{corollary:bigballs}, it holds that $|B| \ge \frac{n^{1-\frac{1}{40 |V|}}}{2}$ with high probability. $A$ and $B$ are independent, as $A$ is undetermined by the mapping of large configurations. By \cref{lemma:set-intersection}, it holds that $|A \cap B| \ge \frac{n^{1-\frac{1}{40 |V|}}}{8}$ with failure probability at most $2e^{\frac{-n^{1-\frac{1}{40 |V|}}}{16}}$. Moreover, all such bins will have total surplus at most $\frac{2}{n \log(\log(n)) \log^{2 \close}(n)}$, because they receive no large configurations and total surplus is then upper-bounded by the sum of small ball surplus and total mass from dense balls.
\end{proof}

\textbf{Existence of many desirable bins at $\colorroot(Y)$}.

We have shown that at each node $X_i$ there are many bins without much surplus. If we restrict this calculation of surplus to not include mass from plateau configurations in $S_{\textrm{indep}}$ for $\colorroot(Y)$, then the set of bins that do not have much surplus is independent of the assignment of such configurations with plateau balls not having much related mass. Consider the set $S^{\ord(\colorroot(Y))}_{\textrm{indep}}$. By \cref{corollary:s-indep}, we know $|S^{\ord(\colorroot(Y))}_{\textrm{indep}}| \ge \frac{|S|}{6^{|V|}}$, no two corresponding plateau balls ever share a state before $\colorroot(Y)$, and all $\relate_2^{\ord(\colorroot(Y))}(x) \le \frac{18 \times |V| \times (|V|-1)}{n^2}$. Recall that $\colorroot(Y)$ by its definition must create a new color, and thus either have $\Xsrc$ as a parent, or have at least two distinct colors in its parent set. Given these properties, we know that each plateau ball in this set has at most $\relate_2^{\ord(\colorroot(Y))}(x) \le \frac{18 \times |V| \times (|V|-1)}{n^2}$ mass in its configuration for $\colorroot(Y)$, and all elements in the set will be in different configurations. We aim to show that there are many bins with small surplus that receive many of these configurations corresponding to the plateau balls in the set. To do so, we use the following results shown in \cref{theorem:pairwise}. First, we use negative association:

\nathreshold*

Second, we lower bound the probability of a bin researching a certain threshold:

\claimindividualbin*

\begin{corollary}
\label{corollary:num-configs}
With high probability, there are $\frac{n^{1 - \frac{1}{20 |V|}}}{16 e}$ bins with $z_{j} \le \frac{2}{n \log(\log(n)) \log^{2 \close}(n)}$ and at least $\frac{\log(n)}{40 |V| \log(\log(n))}$ configurations mapped from states of $S_{\textrm{indep}}^{\ord(\colorroot(Y))}$.
\end{corollary}
\begin{proof}
Consider the indicator variable $\mathcal{B}_i$ if a bin met the threshold of plateau configurations. We show that with high probability $\sum_i \mathcal{B}_i$ is large enough, considering just the bins with small surplus. By \cref{corollary:num-bins} we know there are at least $\frac{n^{1 - \frac{1}{40 |V|}}}{8}$ such bins with high probability. Now let us focus on just the configurations corresponding to $S_{\textrm{indep}}^{\ord(\colorroot(Y))}$ that we know has cardinality $\Omega(n)$. By \cref{claim:individual-bin}, the probability of a bin receiving at least $\frac{\log(n)}{40 |V| \log(\log(n))}$ such configurations is at least $\frac{1}{e n ^{\frac{1}{40 |V|}}}$. Accordingly, it holds that $\sum_i E[\mathcal{B}_i] \ge \frac{n^{1 - \frac{1}{20|V|}}}{8e}$. 

By Hoeffding's inequality, we show $|\sum_i \mathcal{B_i} - \sum_i E[\mathcal{B_i}]| \le \frac{n^{1 - \frac{1}{20 |V|}}}{16e}$ with high probability:
\begin{align*}
    & P[|S_n - E_n| > t] < 2 e^{- \frac{2 t^2}{\sum c_i^2}} \\
    & \le 2 e^{- \frac{2 n^{2 - \frac{1}{10 |V|}}}{16^2 e^2 \frac{n^{1 - \frac{1}{40 |V|}}}{8}}} \\
    & \le 2 e^{- \frac{n^{1 - \frac{3}{40 |V|}}}{16 e^2}}.
\end{align*}

Therefore, $\sum_i \mathcal{B}_i \ge \frac{n^{1 - \frac{1}{20|V|}}}{16e}$ with high probability.
\end{proof}

\textbf{Survival of desirable bins to $Y$.} Now we aim to show that of the bins that received many plateau configurations and had small surplus, that enough will ``survive'' and keep these properties as we process nodes within the same color as $Y$, and that eventually at least one such bin will survive to $Y$ with high probability. 

\begin{lemma}
\label{lemma:survival}
After processing $i$ nodes of the same color as $Y$, with high probability there are at least $\frac{n^{1 - \frac{i}{20|V|}}}{16e}$ sets of plateau balls, such that each set has cardinality at least $\frac{\log(n)}{100 \log(\log(n))}$, have been assigned together to a bin with surplus $z_{j} \le \frac{2}{n \log(\log(n)) \log^{2 \close}(n)}$, and no two sets were ever mapped to the same state within this color.
\end{lemma}
\begin{proof}
Trivially, this holds for $i=1$ from \cref{corollary:num-configs}.

First, we note that all sets of plateau balls will again be mapped together. This is because they are in the same configuration, as for any node $X_i$, as it inherits its color, it must be true that they all have the same values for $\pa(X_i)$ and because they are plateau balls they must have $E^* = e^*_1$. 

Now, we will make a random variable $\mathcal{S}_i$ for whether the $i$-th configuration survived together. Roughly, we desire $\mathcal{S}_i$ to be $1$ if it is assigned to a bin with small surplus with none of the other bins that has survived to this stage, we desire $\mathcal{S}_i$ to be $0$ if it is assigned to a bin with non-small surplus, and $\mathcal{S}_i$ to be $-1$ if it lands in a small surplus bin with another bin that had survived (the intuition is that said bin would likely have a positive $\mathcal{S}_j$ and now we must cancel them out). Now, we slightly modify $\mathcal{S}_i$ so all $\mathcal{S}_i$ are independent. By \cref{corollary:num-bins} we know there will be at least $\frac{n^{1 - \frac{1}{40 |V|}}}{8}$ small surplus bins with high probability. Let us create a subset of bad bins $B_{\textrm{bad}}$ for which $\mathcal{S}_i$ will take value $-1$. Before realizing the assignment for the $i$-th configuration, add all small surplus bins that have already received a bin that survived to this round. Arbitrarily fill the remainder of $B_{\textrm{bad}}$ so that $|B_{\textrm{bad}}| =  \frac{n^{1 - \frac{i}{20|V|}}}{16e}$ . Let us define $|B_{\textrm{good}}|=\frac{n^{1 - \frac{1}{40 |V|}}}{8} - \frac{n^{1 - \frac{i}{20|V|}}}{16e}$. So, if the configuration is assigned to $B_{\textrm{bad}}$ then $\mathcal{S}_i = -1$, if assigned to $B_{\textrm{good}}$ then $\mathcal{S}_i = 1$, and otherwise $\mathcal{S}_i = 0$. By Hoeffding's inequality, it holds that $\sum_i \mathcal{S_i} \ge \frac{n^{1 - \frac{i+1}{20|V|}}}{16e}$ with high probability
and thus the lemma holds.
\end{proof}

\textbf{Concluding large conditional entropy from desirable bin.} By \cref{lemma:survival}, it is clear that with high probability there is at least $\frac{n^{19/20}}{16e} >\sqrt{n}$ bin $y'$ of $Y$ satisfying the desired properties. Now, we seek to prove that this implies $H(\Xsrc | Y=y')$. Consider a looser definition of surplus:

\begin{restatable}[Relaxed Surplus, $\mathcal{T}^{\textrm{relax}}=\frac{120 |V|}{n \log(n)}$]{definition}{defsurplusgraph}
We define the surplus of a state $i$ of $Y$ as $z_i^{\textrm{relax}} = \sum_{j \notin S} \max (0, P(X=j, Y=i) - \frac{120 |V|}{n \log(n)})$.
\end{restatable}

\begin{claim}
There exists a bin with at least $\frac{\log(n)}{100 \log(\log(n))}$ plateau balls and relaxed surplus at most $\frac{2 |V|}{n \log(\log(n)) \log^{2 \close}(n)}$
\end{claim}
\begin{proof}
There are two contributors towards relaxed surplus. First, when configurations were assigned to $\colorroot(Y)$, each plateau ball brought $\relate_2(x)$ mass with it that could contribute to the surplus. By \cref{lemma:inductive-related}, each of the plateau balls we considered satisfied $\relate_2(x) \le \frac{18 |V| (|V|-1)}{n^2}$. Accordingly, there is at most $n \times \frac{18 |V| (|V|-1)}{n^2} \le \frac{18 |V|^2}{n}$ such mass in total. Among the at least $\sqrt{n}$ bins that survived to $Y$, let us choose the one with the least initial mass from $\relate_2$. Accordingly, it must have at most $\frac{18 |V|^2}{n^{1.5}}$ such mass.

Now, consider the at most $|V-1|$ times that mass may have been acquired by landing in a bin with at most $\frac{2}{n \log(\log(n)) \log^{2 \close}(n)}$ surplus. Combining all these masses and calculating the worst-case relaxed mass results in an upper-bound of $\frac{2}{n \log(\log(n)) \log^{2 \close}(n)} \times (|V|-1) + \frac{18 |V|^2}{n^{1.5}} \le \frac{2 |V|}{n \log(\log(n)) \log^{2 \close}(n)}$. This is because the definition of relaxed surplus gives enough threshold to fit within it all the mass that was within the regular surplus threshold for each of the groups we are aggregating. 
\end{proof}

Now, we show that this implies $H(X|Y=y')$ is large, with almost exactly the same proof as \cref{lemma:high-conditional}:

\begin{restatable}[High-entropy conditional]{lemma}{lemmahighentropygraph}
\label{lemma:high-conditional-graph}
Given a bin $y'$ that has $z^{\textrm{relax}}_{y'} \le \frac{2 |V|}{n \cdot \log(\log(n)) \cdot \log^{2 \close}(n)}$, and receives $\frac{\log(n)}{100 \log(\log(n))}$ plateau balls, then $H(\Xsrc | Y=y') = \Omega(\log(\log(n)))$.
\end{restatable}
\begin{proof}
To show $H(\Xsrc | Y=y')$ is large, we first define the vector $v$ such that $v(x) = P(\Xsrc=x,Y=y')$. Similarly, we define $\overline{v}(x) = \frac{v}{P(Y=y')}$, meaning $\overline{v}(x) = P(\Xsrc=x | Y=y')$ and $|\overline{v}|_1 = 1$. Our underlying goal is to show $H(\overline{v})$ is large. To accomplish this, we will split the probability mass of $v$ into three different vectors $\vinit, \vplat, \vsur$ such that $v = \vinit + \vplat + \vsur$. The entries of $\vplat$ will correspond to mass from plateau states of $X$, $\vinit$ will correspond to the first $\mathcal{T}^{\textrm{relaxed}}$ mass from non-plateau states of $\Xsrc$, and $\vsur$ will correspond to mass that contributes to the surplus $z_{y'}$. We more formally define the three vectors as follows:

\begin{itemize}
    \item $\vplat$. The vector of probability mass from plateau states of $\Xsrc$. $\vplat(x)$ is $0$ if $x \notin S$ and $\vplat(x) = P(\Xsrc=x,Y=y')$ if $x \in S$.
    \item $\vinit$. For non-plateau states of $\Xsrc$, their first $\mathcal{T}$ probability mass belongs to $\vinit$. $\vinit(x)= \min (P(X=x,Y=y'),\mathcal{T}^{\textrm{relaxed}})$ if $x \notin S$ and $\vinit(x)=0$ otherwise.
    \item $\vsur$. For non-plateau states of $\Xsrc$, their probability mass beyond the first $\mathcal{T}^{\textrm{relaxed}}$ mass belongs to $\vsur$. This corresponds to the surplus quantity. $\vsur(x)= \max (0,P(\Xsrc=x,Y=y')-\mathcal{T}^{\textrm{relaxed}})$ if $x \notin S$ and $\vsur(x)=0$ otherwise. By this definition, $z_{y'} = |\vsur|_1$.
\end{itemize}

To show $H(\Xsrc|Y=y')=H(\overline{v})$ is large, we divide our approach into two steps:

\begin{enumerate}
    \item Show there is substantial helpful mass: $|\vinit + \vplat|_1 = \Omega \left( \frac{1}{n \cdot \log(\log(n)) \cdot \log^{2 \close}(n)} \right)$
    \item Show the distribution of helpful mass has high entropy: $H \left( \frac{\vinit + \vplat}{|\vinit + \vplat|_1} \right) = \Omega (\log(\log(n)))$ 
    \item Show that, even after adding the hurtful mass, the conditional entropy is large: $H(\Xsrc | Y=y') = H(\overline{v}) \ge H \left( \frac{\vinit + \vplat}{|\vinit + \vplat|_1} \right) - O(1) = \Omega (\log(\log(n)))$
\end{enumerate}

In the first step, we are showing that the distribution when focusing on just the helpful mass of $\vinit, \vplat$ has high a substantial amount of probability mass. In the second step, we prove how this distribution of helpful mass has high entropy. In the third step, we show that the hurtful mass of $\vsur$ does not decrease entropy more than a constant.

First, we show that there is a substantial amount of helpful mass:
\begin{claim}
\label{claim:helpful-mass}
$|\vinit + \vplat|_1 = \frac{1}{100 \clb n \cdot \log(\log(n)) \cdot \log^{2 \close}(n)} $
\end{claim}
\begin{proof}
Recall that the bin $y'$ received $\frac{\log(n)}{100\log(\log(n))}$ plateau balls. As defined in \cref{lemma:plateau}, the set $S$ of plateau states is defined such that $\frac{\max_{x \in S} P(X=x)}{\min_{x \in S} P(X=x)} \le \log^{\close}(n)$ and $\min_{x \in S}P(X=x) \ge \frac{1}{\clb n \log(n)}$. Also recall that by \cref{lemma:e-bound} the most probably state of $E$ has large probability. In particular, $P(E=e_1) \ge \frac{1}{\log^{\close}(n)}$. Let the subset $S' \subseteq S$ be the subset of plateau states of $X$ such that their plateau ball is mapped to $y'$. In particular, for every $x \in S'$ it holds that $f(x,e_1)=y'$. Accordingly, $P(\Xsrc=x,Y=y') \ge P(\Xsrc=x)\cdot P(E=e_1)$ for $x \in S'$. Thus, the total weight from plateau states of $\Xsrc$ is at least $|S'| \cdot \min_{x \in S'} P(\Xsrc=x) \cdot P(E=e_1) \ge |S'| \cdot \frac{\max_{x \in S'} P(\Xsrc=x)}{\log^{\close}(n)} \cdot P(E=e_1) \ge \frac{1}{100\clb n \log(\log(n)) \log^{2 \close}(n)}$.
\end{proof}

Next, we show the distribution of helpful mass has high entropy:
\begin{claim}
$H \left( \frac{\vinit + \vplat}{|\vinit + \vplat|_1} \right) \ge \frac{\log(\log(n))}{4}$ 
\end{claim}
\begin{proof}
Let us define $\overvhelp = \frac{\vinit + \vplat}{|\vinit + \vplat|_1}$ to be the vector of helpful mass, and we will show $H(\overvhelp)$ is large by upper-bounding $\max_x \overvhelp(x)$. 

For non-plateau states of $\Xsrc$, it follows from \cref{claim:helpful-mass} that $\max_{x \notin S} \overvhelp(x) \le \frac{\mathcal{T}^{\textrm{relaxed}}}{|\vinit + \vplat|_1} \le \frac{\mathcal{T}^{\textrm{relaxed}}}{\frac{1}{100 \clb n \cdot \log(\log(n)) \cdot \log^{2 \close}(n)}} = \frac{100 |V| \clb \log(\log(n)) \cdot \log^{2 \close}(n)}{\log(n)}$. 
For plateau states of $X$, in \cref{claim:helpful-mass} we also developed the lower-bound of $|\vinit + \vplat|_1 \ge |S'| \cdot \frac{\max_{x \in S'} P(X=x)}{\log^{\close}(n)} \cdot P(E=e_1) \ge \frac{\log(n) \cdot \max_{x \in S'} P(X=x)}{2\log^{2 \close}(n) \log(\log(n))}$. Accordingly, we can upper-bound $\max_{x \in S'} \overvhelp(x) \le \frac{\max_{x \in S'} P(X=x)}{|\vinit + \vplat|_1} \le \frac{100 \log(\log(n)) \log^{2 \close}(n)}{\log(n)}$.

Accordingly, we can lower-bound the entropy of $H(\overvhelp) = \sum_{x} \overvhelp(x) \cdot \log(\frac{1}{\overvhelp(x)}) \ge \sum_{x} \overvhelp(x) \cdot \log(\frac{1}{\max_{x'}\overvhelp(x')}) = \log(\frac{1}{\max_{x'}\overvhelp(x')}) \ge \log(\frac{1200 \clb \log(n)}{\log^{2 \close}(n) \log(\log(n))}) = (1 - 2 \close) \log(\log(n)) - \log(\log(\log(n))) - \log(1200 \clb) = \frac{\log(\log(n))}{2} - \log(\log(\log(n))) - \log(1200 \clb) \ge \frac{\log(\log(n))}{4}$ for sufficiently large $n$ where $\frac{\log(\log(n))}{2} \ge \log(\log(\log(n))) + \log(1200 \clb)$.
\end{proof}

Finally, we show the hurtful mass does not decrease entropy much, and thus our conditional distribution has high entropy:
\begin{claim}
\label{claim:helpful-entropy}
$H(X | Y=y') = H(\overline{v}) \ge \Omega(1) \cdot H \left( \frac{\vinit + \vplat}{|\overvinit + \overvplat|_1} \right) - O(1) = \Omega (\log(\log(n)))$
\end{claim}
\begin{proof}
We lower-bound $H(\overline{v})$ with the main intuitions that $H\left(\frac{\vinit + \vplat}{|\vinit + \vplat|_1}\right) = \Omega(\log(\log(n)))$ and $\frac{|\vinit + \vplat|_1}{|\vinit + \vplat + \vsur|_1} = \Omega(1)$. We more precisely obtain this lower-bound for $H(\overline{v})$ as follows:
\begin{align}
    & H(\overline{v}) = H\left(\frac{\vinit + \vplat + \vsur}{|\vinit + \vplat + \vsur|_1}\right) \nonumber\\ 
    & = \sum_x \frac{\vinit(x) + \vplat(x) + \vsur(x)}{|\vinit + \vplat + \vsur|_1} \times \nonumber\\
    & \qquad\log\frac{|\vinit + \vplat + \vsur|_1}{\vinit(x) + \vplat(x) + \vsur(x)} \nonumber \\ 
    & \ge \sum_x \frac{\vinit(x) + \vplat(x)}{|\vinit + \vplat + \vsur|_1} \times  \nonumber\\
    & \qquad\log\frac{|\vinit + \vplat + \vsur|_1}{\vinit(x) + \vplat(x)} - 2\label{step:decrease} \\ 
    & \ge \sum_x \frac{\vinit(x) + \vplat(x)}{|\vinit + \vplat + \vsur|_1} \times \nonumber\\
    & \qquad\log\frac{|\vinit + \vplat|_1}{\vinit(x) + \vplat(x)} \nonumber -2 \\ 
    & = \frac{|\vinit + \vplat|_1}{|\vinit + \vplat + \vsur|_1} H\left(\frac{\vinit + \vplat}{|\vinit + \vplat|_1}\right) \nonumber \\
    & -2\nonumber\\
    & = \frac{|\vinit + \vplat|_1}{|\vinit + \vplat|_1 + z_{y'}} H\left(\frac{\vinit + \vplat}{|\vinit + \vplat|_1}\right) -2\nonumber\\
    & \ge  \frac{1}{1+50\clb |V|} \cdot H\left(\frac{\vinit + \vplat}{|\vinit + \vplat|_1}\right) -2  \label{step:use-mass} \\
    & = \Omega(\log(\log(n))) \label{step:use-entropy}
\end{align}

To obtain Step \ref{step:decrease}, we note that all summands are manipulated from the form $\sum_x p_x \log(\frac{1}{p_x})$ to $\sum_x p'_x \log(\frac{1}{p'_x})$ where $p'_x \le p_x$ for all $x$. As the derivative of $p \log(\frac{1}{p})$ is non-negative for $0 \le p \le \frac{1}{e}$, the value of at most two summands can decrease, and they can each decrease by at most one. To obtain Step \ref{step:use-mass}, we use \cref{claim:helpful-mass}. To obtain Step \ref{step:use-entropy}, we use \cref{claim:helpful-entropy}.
\end{proof}

Thus, we have shown $H(X|Y=y') = \Omega(\log(\log(n)))$.
\end{proof}
\begin{corollary}
Under our assumptions, $H(\Xsrc|Y=y')=\Omega(\log(\log(n)))$ and thus $H(\tilde{E}) = \Omega(\log(\log(n)))$.
\end{corollary}

\section{Counterexample for General Identifiability with Unconfounded-Pairwise Oracles}
\label{section:non-identifiable}
We formalize the oracle first discussed in \cref{section:learning-graphs}:
\begin{definition}[Unconfounded-pairwise oracle]
An unconfounded-pairwise oracle is an oracle that returns the correct orientation of an edge if the edge exists in the true graph and if there is no confounding for the edge.
\end{definition}

Consider a causal graph $G_1$ with four nodes and the edge set $\{X_1 \rightarrow X_2, X_1 \rightarrow X_3, X_2 \rightarrow X_4, X_2 \rightarrow X_3, X_3 \rightarrow X_4\}$. Likewise, consider $G_2$ with four nodes and the edge set $\{X_2 \rightarrow X_1, X_3 \rightarrow X_1, X_4 \rightarrow X_2, X_2 \rightarrow X_3, X_4 \rightarrow X_3\}$. Note that these graphs are in the same Markov equivalence class. Finally, consider a causal graph $G_3$ with four nodes and the edge set $\{X_1 \rightarrow X_2, X_1 \rightarrow X_3, X_4 \rightarrow X_2, X_2 \rightarrow X_3, X_4 \rightarrow X_3\}$. If we orient all edges in the skeleton for $G_1$ and $G_2$ without conditioning using an unconfounded-pairwise oracle, $G_3$ is a consistent output of the oracle for both $G_1$ and $G_2$. As a result, the peeling approach would see the same set of edge orientations for $G_1$ and $G_2$ and not be able to identify the true source.

\section{Proof of \cref{theorem:graph-alg}}
We need to show that in each iteration of the \textbf{while} loop in \emph{line 4}, the algorithm correctly identifies all non-source nodes and only the true non-sources. In the first iteration of the loop, there is no node to condition on and therefore tests are independence (not conditional independence) tests. Consider a non-source node $N$ in the initial graph. Then there must be a directed path from some source node $\Xsrc$ to $N$. Due to the faithfulness assumption, two nodes with a directed path cannot be unconditionally independent. Then either $\oracle(\Xsrc,N | \emptyset)$ will return $\Xsrc\rightarrow N$ in \emph{line 12} or $\oracle(N,\Xsrc|\emptyset)$ will return $\Xsrc\rightarrow N$. In either case, $N$ is added to the non-source list. Now suppose $S$ is a source node. It is never identified as a non-source since it is either conditionally independent with the node it is compared with, found in \emph{line 9} , or it is conditionally dependent with a non-source (for which it has a path to) and the oracle orients correctly in \emph{line 11} or \emph{line 13}. Since all non-sources are correctly identified and only the true non-sources are identified as non-sources, all sources are correctly identified as well in \emph{line 15}. Since sources are incomparable in the partial order, the order in which they are added to the topological order does not impact the validity of topological order in \emph{line 18}. 

Suppose the \textbf{while} loop identified all sources correctly for all iterations $j$, $\forall j<i$. Let $\mathcal{S}_i$ be the sources in $\mathcal{R}$, i.e., the sources that are to be discovered in iteration $i$. Let $Pa_{\mathcal{S}_i}$ be the set of parents of $\mathcal{S}_i$ in the initial graph. Then $Pa_{\mathcal{S}_i}\subseteq \mathcal{C}$, i.e., the set of conditioned nodes include all the parents of the current source nodes. Therefore, $\mathcal{C}$ blocks all backdoor paths from $\mathcal{S}_i$, effectively disconnecting the previous found sources from the graph: This is because $G\backslash \mathcal{C}$ is a valid Bayesian network for the conditional distribution $p(.|\mathcal{C}=c)$. Therefore, in $G\backslash \mathcal{C}$, $\mathcal{S}_i$ are source nodes and the source pathwise oracle correctly identifies all non-source nodes, similar to the base case. 

This implies after the \textbf{while} loop terminates, we have a valid topological order for the nodes in the graph. Finally, the for loop in \emph{line 19} converts the obtained total order into a partial order since either it removes an edge, or adds an edge that is consistent with the topological order. Therefore we only need to show that non-edges are correct. Suppose $X_i, X_j$ are non-adjacent in the true graph. Then conditioned on all ancestors of $X_i, X_j$ they are independent due to d-separation in the graph. Therefore all non-edges are identified at some iteration of the \textbf{for loop}. Furthermore, under the faithfulness assumption, no edge can be mistakenly identified as a non-edge:
No two adjacent nodes at any stage of the algorithm are independent conditioned on the ancestors of the cause variable
.

\input{appendix_experiments}
\end{document}

%% file: tikz_graph.tex
\begin{figure*}[ht!]
	\centering
			\subfigure[Line Graph]{	\begin{tikzpicture}[
				observable/.style={circle, draw=black!60, fill=gray!5, very thick, minimum size=8mm},
				]
				\node[observable,fill=yellow] (X) {$X_{\text{src}}$};
				\node[observable,fill=skyblue] (Z) [below = 0.4cm  of X] {$X_2$};
				\node[observable,fill=skyblue] (T) [right= 0.4cm of X] {$X_3$};
				\node[observable,fill=skyblue] (Y)[below = 0.4cm of T]{$Y$};

				\draw[->,line width=0.5mm] (X.south) -- (Z.north);
				\draw[->,line width=0.5mm] (Z.north east) -- (T.south west);
				\draw[->,line width=0.5mm] (T.south) -- (Y.north);
			\end{tikzpicture}
			\label{fig:line}
			}
			\hspace{0.6in}
			    \subfigure[Diamond Graph]{	\begin{tikzpicture}[
				observable/.style={circle, draw=black!60, fill=gray!5, very thick, minimum size=8mm},
				]
				\node[observable,fill=yellow] (X) {$X_{\text{src}}$};
				\node[observable,fill=skyblue] (Z) [below right= 0.17cm and 0.45cm of X] {$X_3$};
				\node[observable,fill=vermillion] (T) [below left= 0.17cm and 0.45cm of X] {$X_2$};
				\node[observable,fill=bluishgreen] (Y)[below = .7cm of X]{$Y_{\hspace{0.02in} }$};

				\draw[->,line width=0.5mm] (X.south west) -- (T.north east);
				\draw[->,line width=0.5mm] (X.south east) -- (Z.north west);
				\draw[->,line width=0.5mm] (T.south east) -- (Y.north west);
				\draw[->,line width=0.5mm] (Z.south west) -- (Y.north east);
			\end{tikzpicture}
			\label{fig:diamond}
			}\hspace{0.6in}
			\subfigure[Hall Graph]{	\begin{tikzpicture}[
				observable/.style={circle, draw=black!60, fill=gray!5, very thick, minimum size=8mm},
				]
				\node[observable,fill=yellow] (X) {$X_{\text{src}}$};
				\node[observable,fill=skyblue] (X2) [right = 0.4cm  of X] {$X_2$};
				\node[observable,fill=vermillion] (X3) [below = 0.4cm of X2] {$X_3$};
				\node[observable,fill=bluishgreen] (X4)[right = 0.4cm of X2]{$X_4$};
				 \node[observable,fill=skyblue] (X5)[below = 0.4cm of X4]{$X_5$};
                \node[observable,fill=orange] (Y)[right = 0.4cm of X5]{$Y$};

				\draw[->,line width=0.5mm] (X.east) -- (X2.west);
				\draw[->,line width=0.5mm] (X.south east) -- (X3.north west);
				\draw[->,line width=0.5mm] (X2.east) -- (X4.west);
				\draw[->,line width=0.5mm] (X2.south east) -- (X5.north west);
				\draw[->,line width=0.5mm] (X3.north east) -- (X4.south west);
				\draw[->,line width=0.5mm] (X4.south east) -- (Y.north west);
				\draw[->,line width=0.5mm] (X5.east) -- (Y.west);				\end{tikzpicture}
			\label{fig:hall}
			}
			\caption{Graphs colored according to the Random Function Graph Decomposition (Definition \ref{def:rfgd}) used in the proof of Theorem \ref{theorem:dag-oracle}.}
\vspace{-.1in}
\end{figure*}

%% file: peeling.tex
\begin{algorithm}[h!]
\small
    \begin{algorithmic}[1]
        
        \STATE $\mathcal{R} \gets \{1,\dots, |V| \}$ \COMMENT{set of remaining nodes}
        \STATE {$\mathcal{I} \gets \emptyset$} \COMMENT{set of pairs found to be conditionally independent}
        \STATE {$\mathcal{T} \gets [$ $] $} \COMMENT{list of nodes in topological order}
        \WHILE{$|\mathcal{R}|>0$}
            \STATE {$\mathcal{N} \gets \emptyset$} \COMMENT{set of nodes discovered as non-sources}
            \STATE{$\mathcal{C} \gets \{1,\dots, |V| \} \backslash \mathcal{R}$} \COMMENT{condition on previous sources}
            \FORALL {$(X_i, X_j) \in \{\mathcal{R} \times \mathcal{R} \}$}
                \IF {$X_i \notin \mathcal{N}$ \AND $X_j \notin \mathcal{N}$ \AND $(X_i, X_j) \notin \mathcal{I}$}
                    \IF{$\CI (X_i, X_j | \mathcal{C})$}
                        \STATE {$\mathcal{I} \gets \mathcal{I} \cup (X_i, X_j)$}
                    \ELSIF{$\oracle(X_i,X_j | \mathcal{C}) \textrm{ orients } X_i \rightarrow X_j$}
                        \STATE {$\mathcal{N} \gets \mathcal{N} \cup \{X_j\}$}
                        \COMMENT{$X_j$ is not a source}
                    \ELSE
                        \STATE {$\mathcal{N} \gets \mathcal{N} \cup \{X_i\}$}
                        \COMMENT{$X_i$ is not a source}
                    \ENDIF
                \ENDIF
            \ENDFOR
            \STATE{$\mathcal{S} \gets \mathcal{R} \backslash \mathcal{N}$} \COMMENT{the remaining nodes that are a source}
            \STATE{$\mathcal{R} \gets \mathcal{R} \backslash \mathcal{S}$} \COMMENT{remove sources from remaining nodes}
            \FORALL{$X_i \in \mathcal{S}$} 
                \STATE {append $X_i$ to $\mathcal{T}$}
            \ENDFOR
        \ENDWHILE
        \COMMENT{Now, $\mathcal{T}$ is a valid topological ordering}
        \FORALL{$(i,j) \in \{1, \dots, |V|\}^2$ where $i < j$}
            \IF {$\CI(\mathcal{T}(i), \mathcal{T}(j) | \{\mathcal{T}(1), \dots, \mathcal{T}(j-1) \} \backslash \mathcal{T}(i))$}
                \STATE {no edge between $\mathcal{T}(i)$ and $\mathcal{T}(j)$}
            \ELSE
                \STATE{orient $\mathcal{T}(i) \rightarrow \mathcal{T}(j)$}
            \ENDIF
        \ENDFOR
    \end{algorithmic}
  \caption{Learning general graphs with oracle
    \label{alg:general}}
\end{algorithm}

%% file: appendix_experiments.tex
\section{Additional Experiments}
\label{app:experiments}
\subsection{Further Synthetic Experiments and Comparisons with ANM}
Figure \ref{fig:3_node_triangle_anm} compares the performance under the additive noise model assumption, i.e., the data is sampled from $X=f(\pa_X)+N_X$ for all variables. The noise term is chosen as a uniform, zero-mean cyclic noise in the supports $\{-1,0,1\}, \{-2,-1,0,1,2\}, \{-3,-2,-1,0,1,2,3\}$ which corresponds to different entropies $H(N_X)$. This entropy is shown on the $x-$axis in Figure \ref{fig:3_node_triangle_anm}. The corresponding plots where the $x-$axis represents the number of samples are given in the Appendix. While this is the generative model that discrete ANM is designed for, its performance is either approximately matched or exceeded by entropic enumeration.

For further experiments with different number of nodes, please see Figures \ref{fig:3_node_triangle_anm}, \ref{fig:3_node_triangle_anm_samples}, 
\ref{fig:5_node_anm}.

\begin{figure*}[h!]
\begin{center}
\subfigure[a]{\label{fig:3_node_triangle_anm_100}\includegraphics[width=0.3\textwidth]{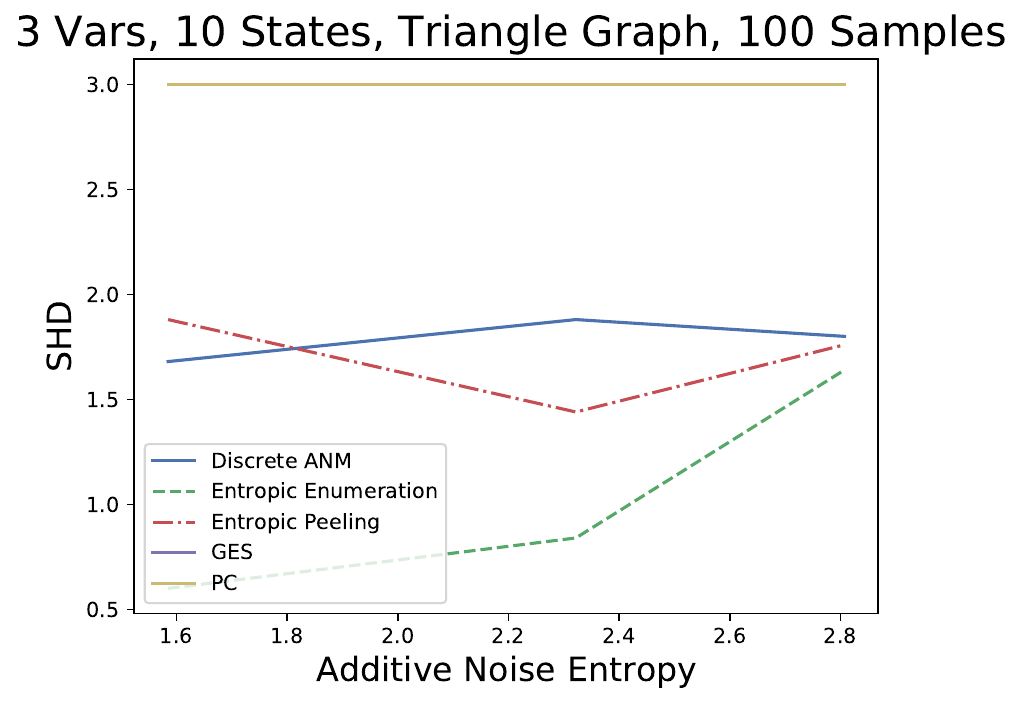}}
\subfigure[b]{\label{fig:3_node_triangle_anm_500}\includegraphics[width=0.3\textwidth]{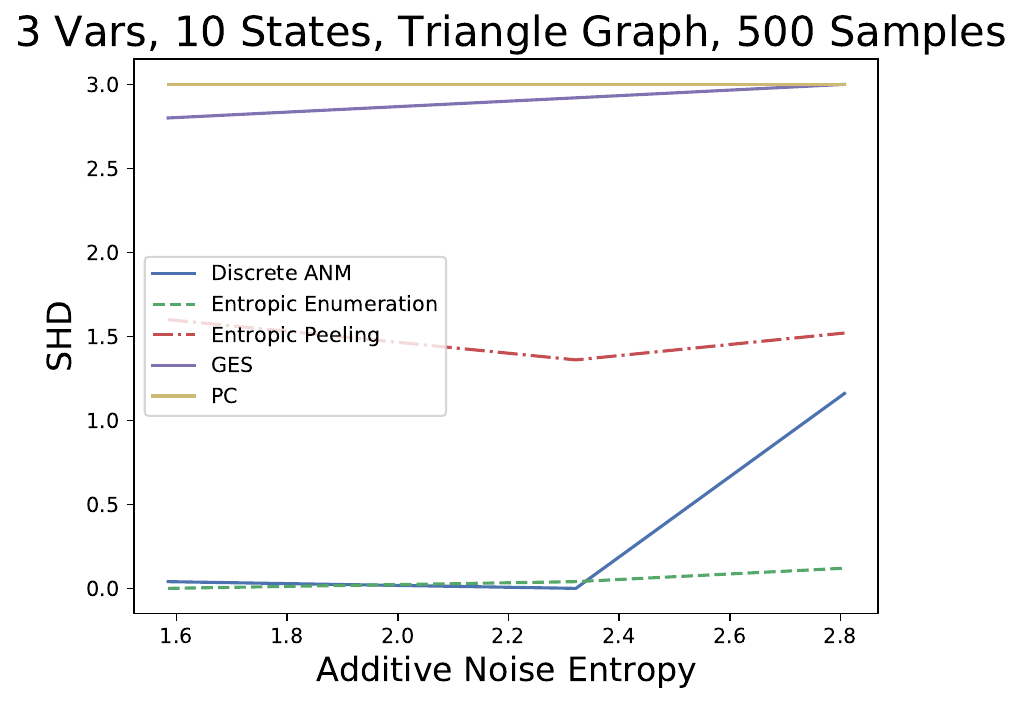}}
\subfigure[c]{\label{fig:3_node_triangle_anm_1000}\includegraphics[width=0.3\textwidth]{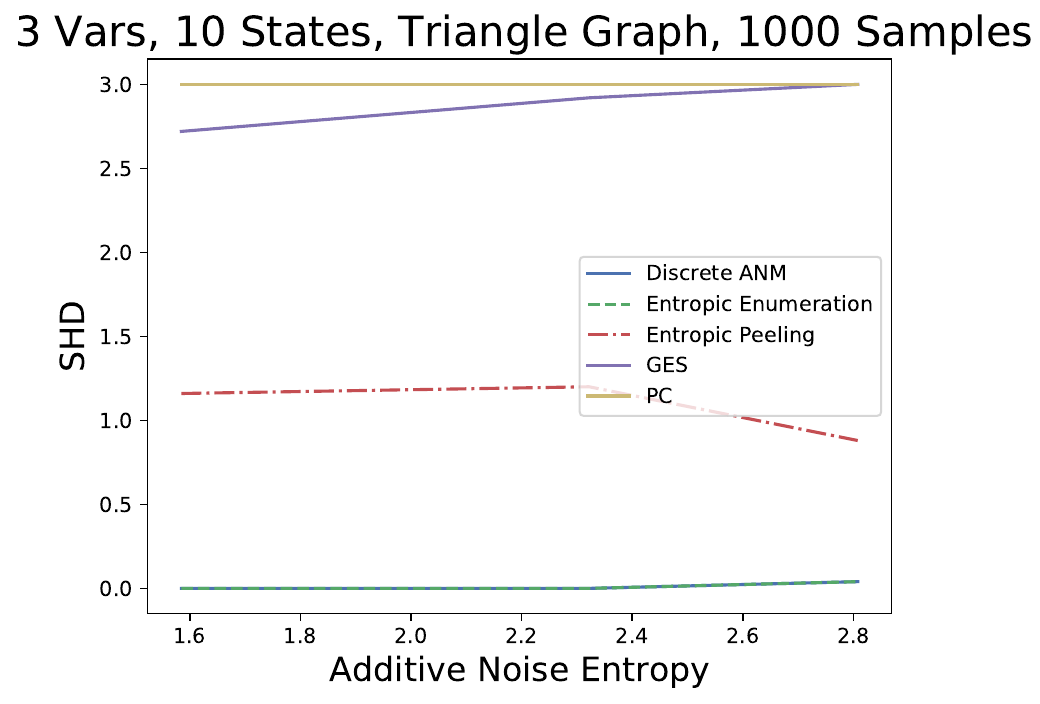}}
\caption{Performance of methods in the ANM setting in the triangle graph $X\rightarrow Y\rightarrow Z, X\rightarrow Z$: $25$ datasets are sampled for each configuration from the ANM model $X=f(\pa_X)+N$. The $x-$axis shows entropy of the additive noise. 
}
\label{fig:3_node_triangle_anm}
\end{center}
\vskip -0.2in
\end{figure*}

\begin{figure*}[ht!]
\vskip 0.2in
\begin{center}
\subfigure[]{\label{fig:3_node_triangle_anm_100}\includegraphics[width=0.3\textwidth]{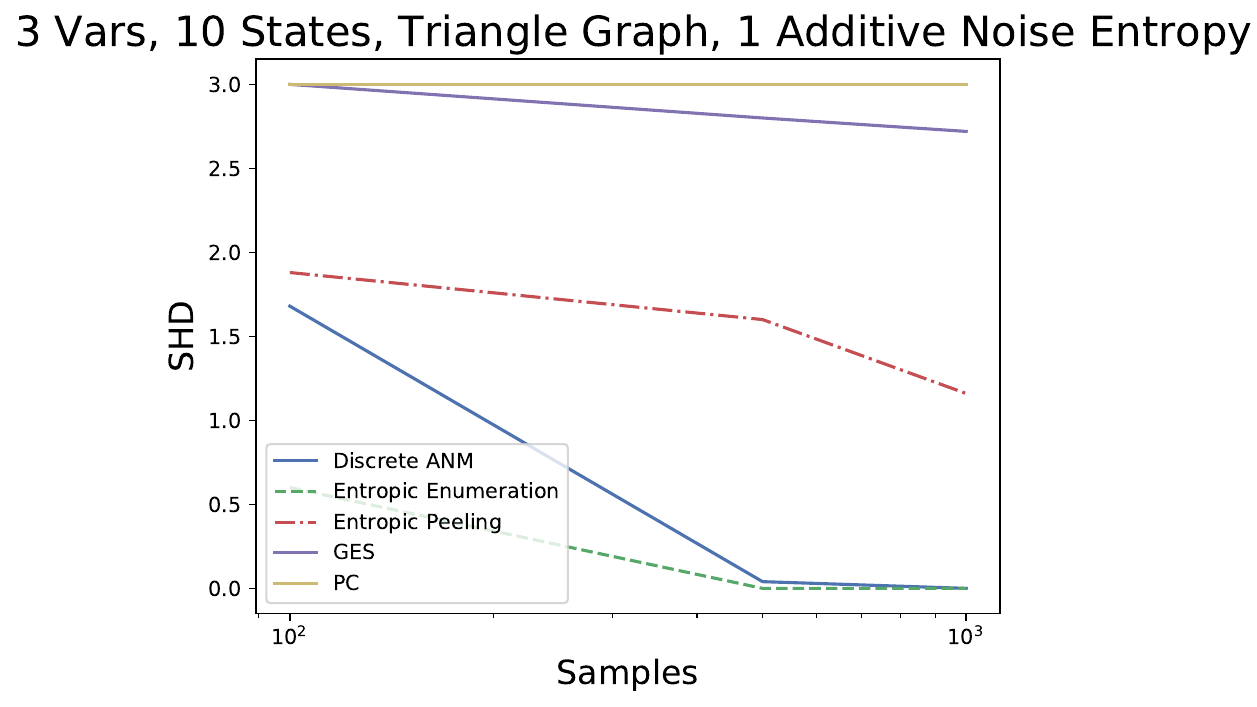}}
\subfigure[]{\label{fig:3_node_triangle_anm_500}\includegraphics[width=0.3\textwidth]{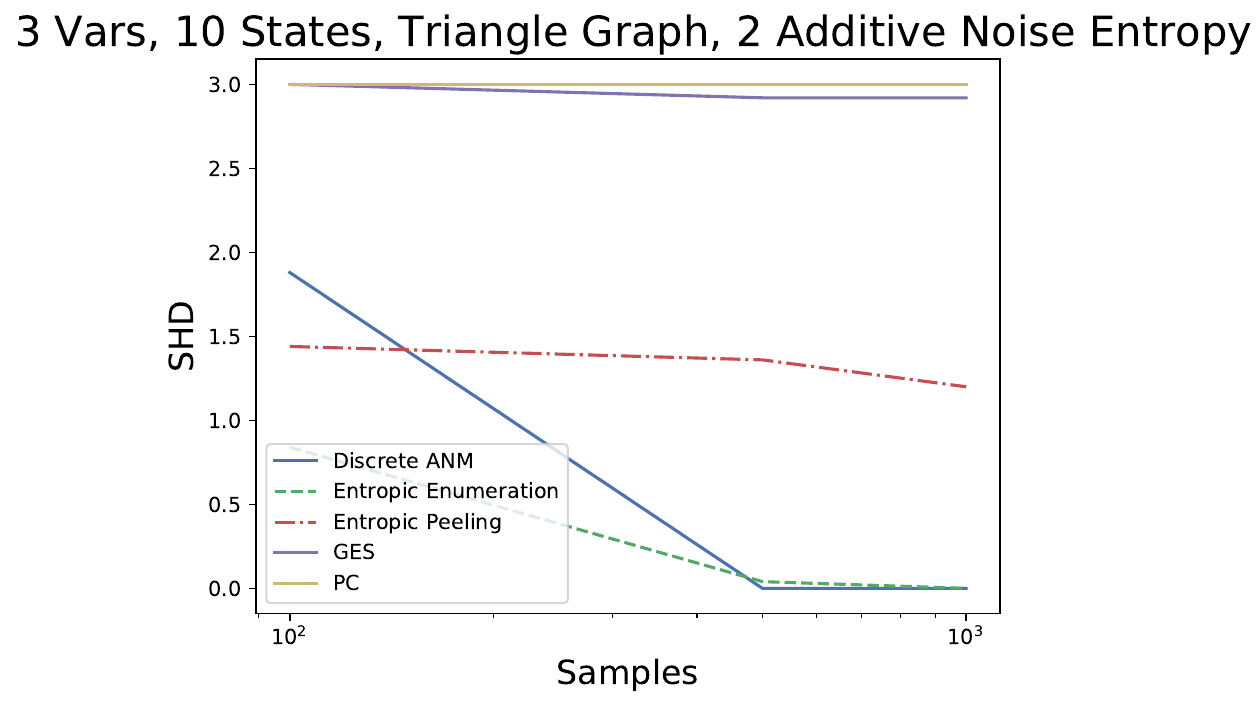}}
\subfigure[]{\label{fig:3_node_triangle_anm_1000}\includegraphics[width=0.3\textwidth]{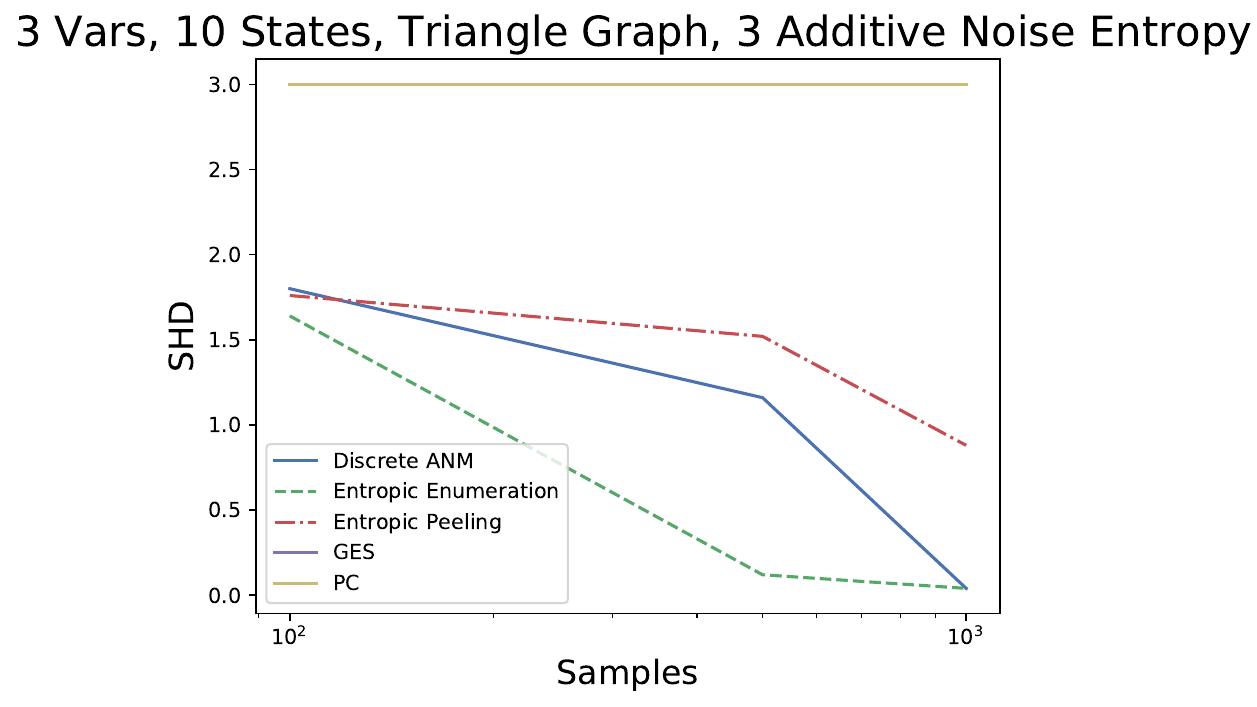}}
\caption{Performance of methods in the ANM setting in the triangle graph $X\rightarrow Y\rightarrow Z, X\rightarrow Z$: $25$ datasets are sampled for each configuration from the ANM model $X=f(\pa_X)+N$. The $x-$axis shows the number of samples in each dataset. Entropic enumeration outperforms ANM algorithm in the low-noise low-sample regime.}
\label{fig:3_node_triangle_anm_samples}
\end{center}
\vskip -0.2in
\end{figure*}

\begin{figure*}[ht!]
\vskip 0.2in
\begin{center}
\subfigure[]{\label{fig:3_node_triangle_anm_100}\includegraphics[width=0.3\textwidth]{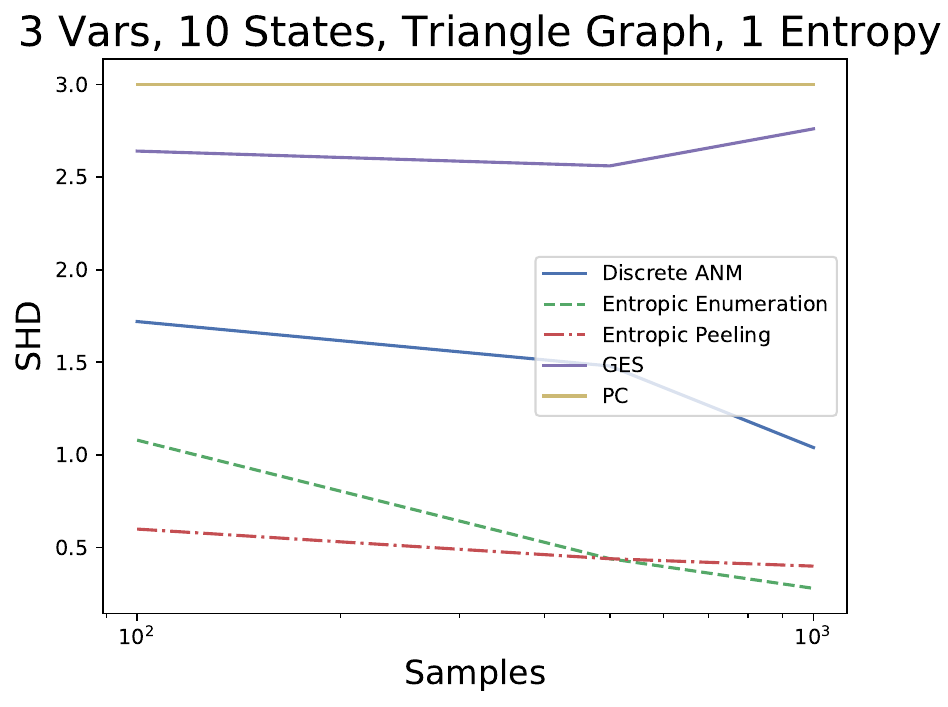}}
\subfigure[]{\label{fig:3_node_triangle_anm_500}\includegraphics[width=0.3\textwidth]{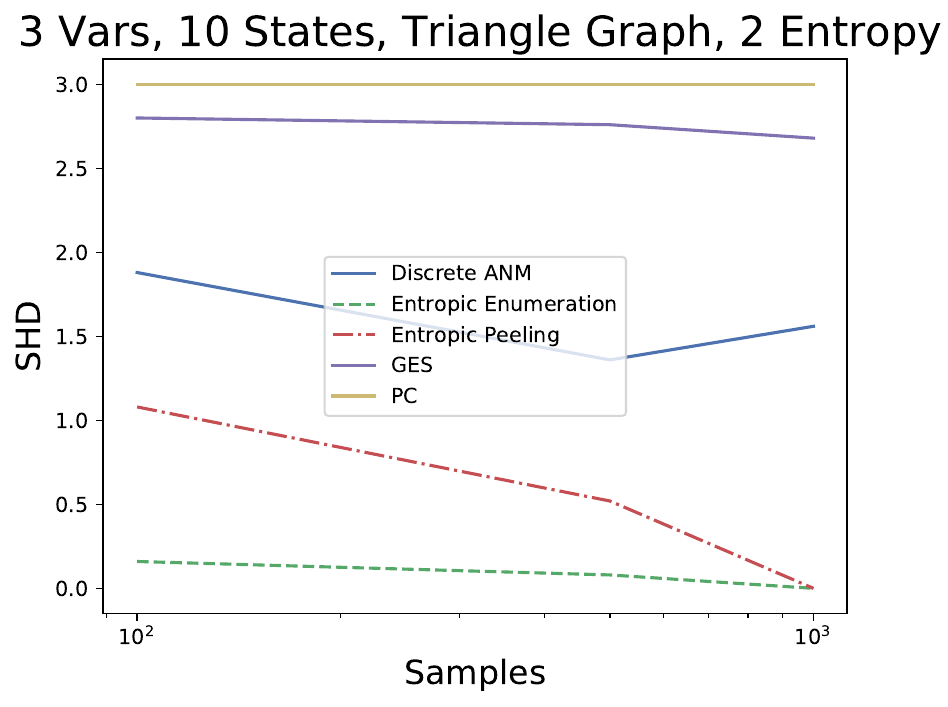}}
\subfigure[]{\label{fig:3_node_triangle_anm_1000}\includegraphics[width=0.3\textwidth]{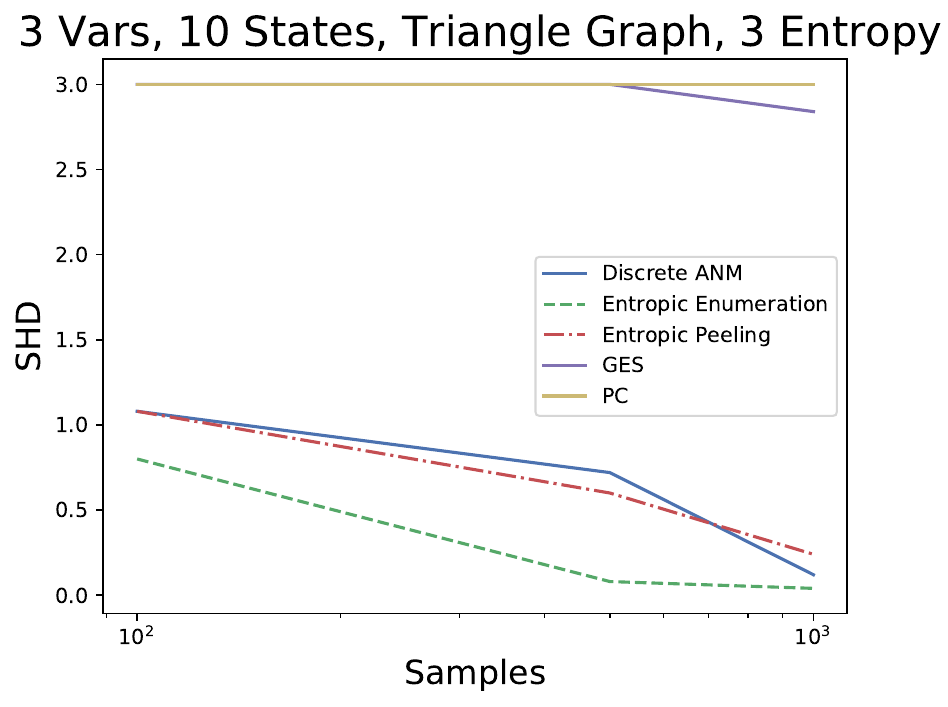}}
\caption{Performance of methods in the unconstrained setting in the triangle graph $X\rightarrow Y\rightarrow Z, X\rightarrow Z$: $25$ datasets are sampled for each configuration from the unconstrained model $X=f(\pa_X,E_X)$. The $x-$axis shows the number of samples in each dataset. Entropic enumeration and peeling algorithms consistently outperform the ANM algorithm in almost all regimes.}
\label{fig:3_node_triangle_entropic_samples}
\end{center}
\vskip -0.2in
\end{figure*}

\begin{figure*}[ht!]
\vskip 0.2in
\begin{center}
\subfigure[]{\label{fig:3_node_line_anm100}\includegraphics[width=0.3\textwidth]{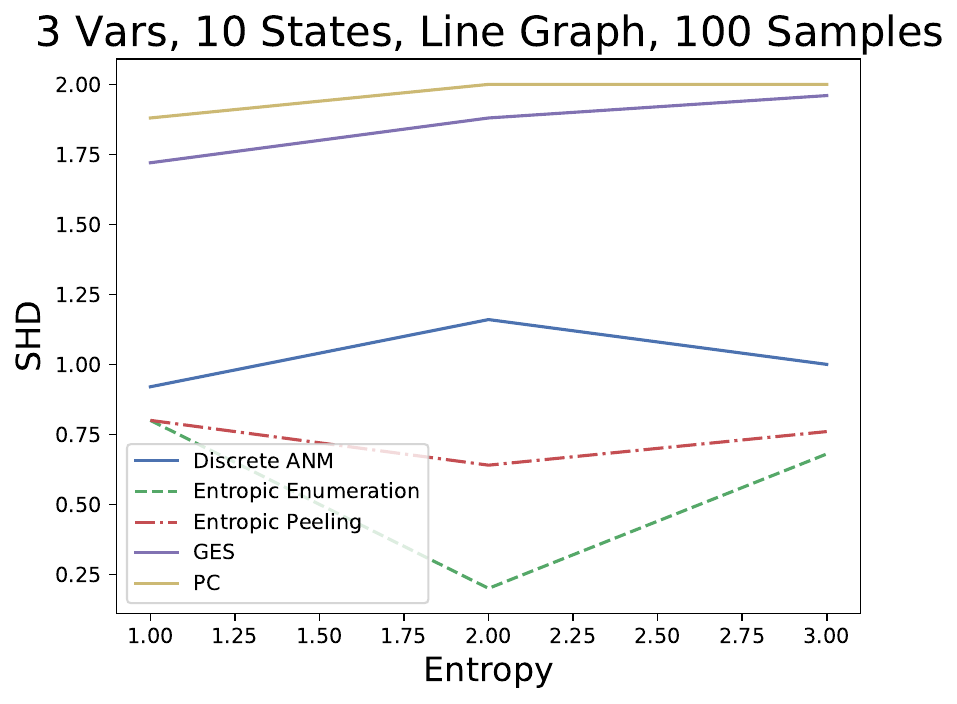}}
\subfigure[]{\label{fig:3_node_line_anm500}\includegraphics[width=0.3\textwidth]{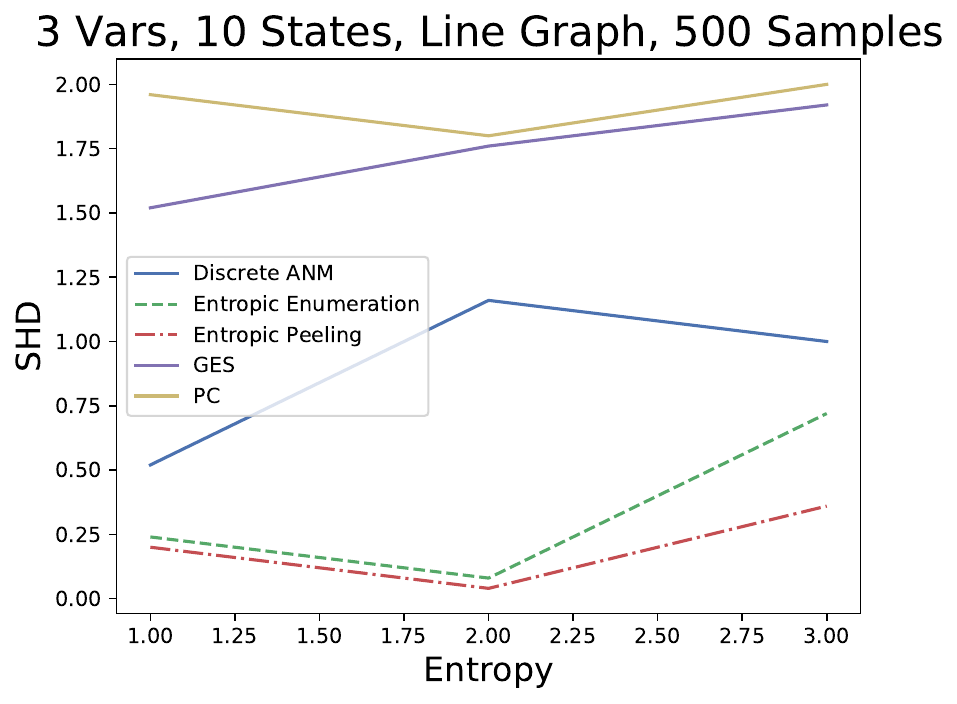}}
\subfigure[]{\label{fig:3_node_line_anm1000}\includegraphics[width=0.3\textwidth]{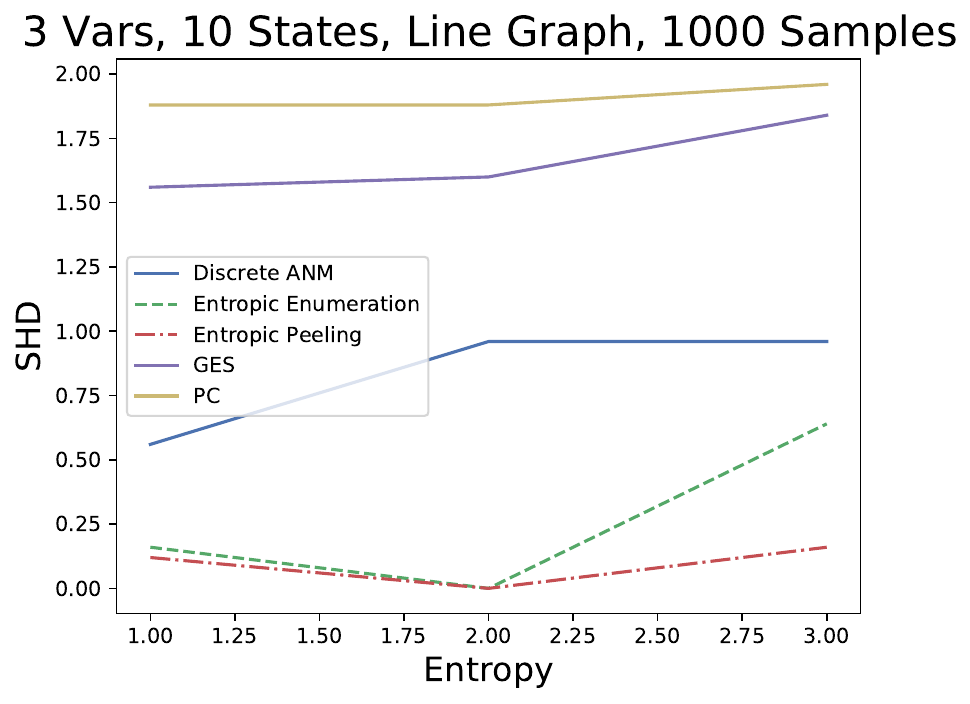}}
\caption{Performance of methods in the unconstrained setting on the line graph $X\rightarrow Y \rightarrow Z$: $25$ datasets are sampled for each configuration from the unconstrained model $X=f(\pa_X,E_X)$. The $x-$axis shows entropy of the exogenous noise. Entropic enumeration and peeling algorithms consistently outperform the ANM algorithm in all regimes.}
\label{fig:3_node_line_entropic}
\end{center}
\vskip -0.2in
\end{figure*}

\begin{figure*}[ht!]
\vskip 0.2in
\begin{center}
\subfigure[]{\label{fig:3_node_triangle_anm_100}\includegraphics[width=0.3\textwidth]{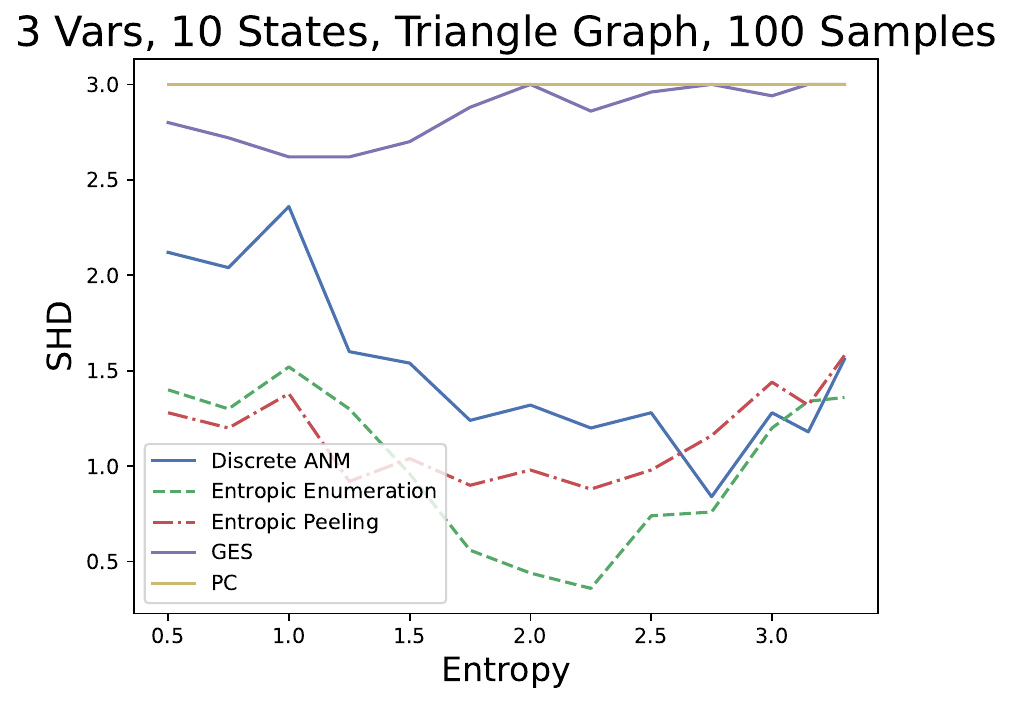}}
\subfigure[]{\label{fig:3_node_triangle_anm_1000}\includegraphics[width=0.3\textwidth]{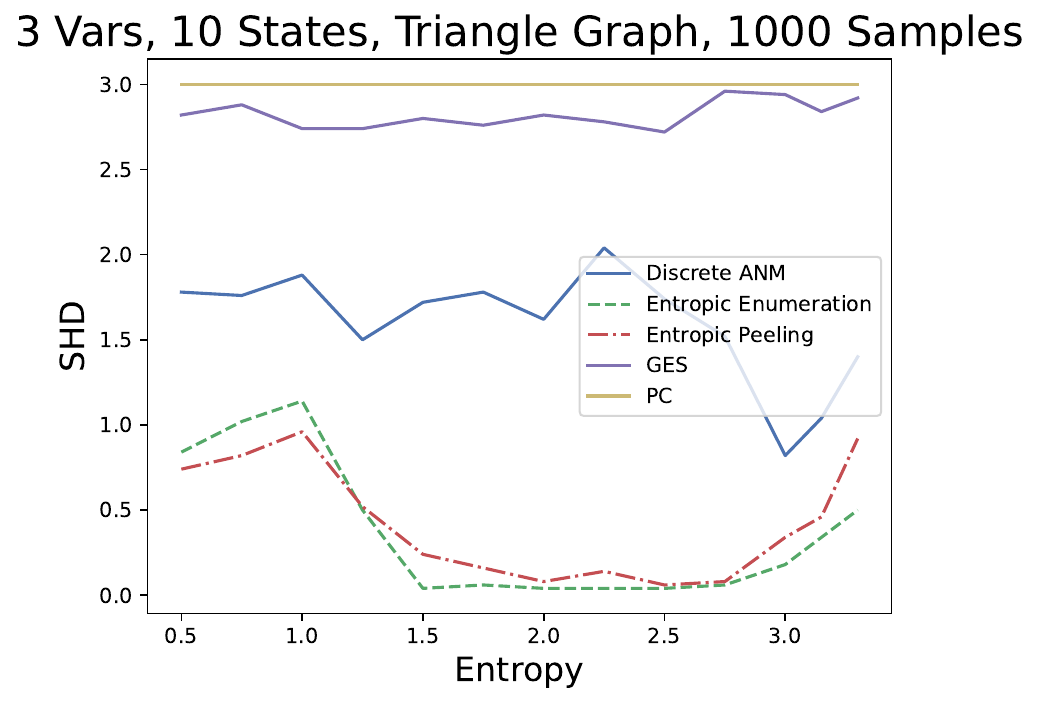}}
\subfigure[]{\label{fig:3_node_triangle_anm_50000}\includegraphics[width=0.3\textwidth]{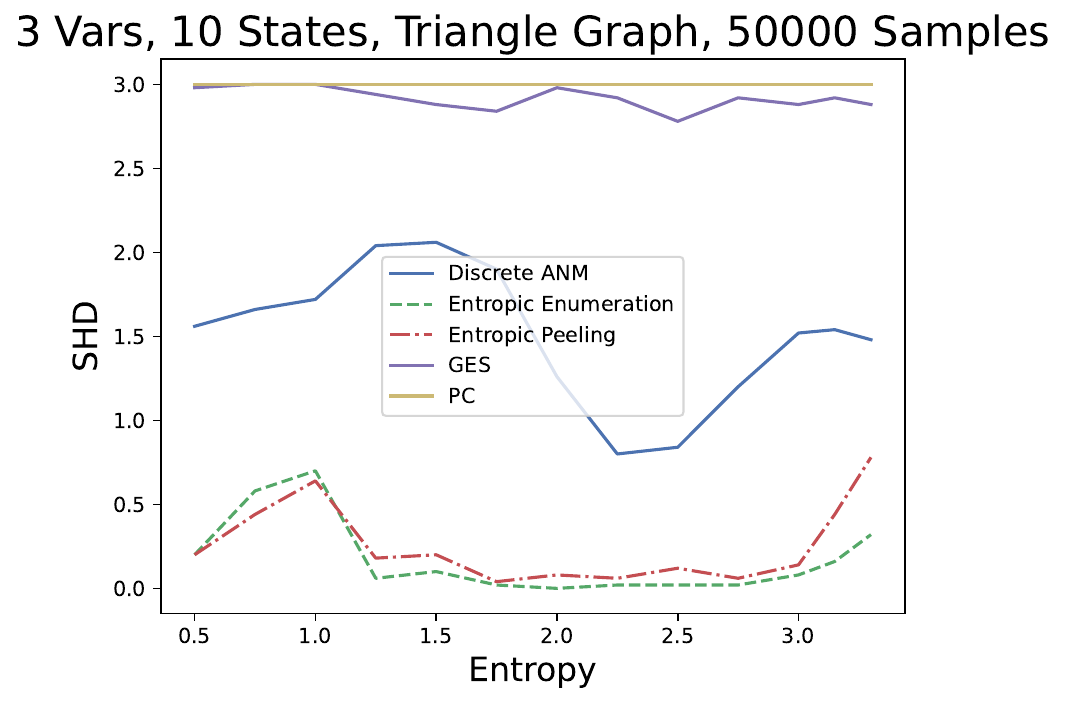}}
\caption{Performance of methods in the unconstrained setting in the triangle graph $X\rightarrow Y\rightarrow Z, X\rightarrow Z$: $50$ datasets are sampled for each configuration from the unconstrained model $X=f(\pa_X,E_X)$. The $x-$axis shows entropy of the exogenous noise. Entropic enumeration and peeling algorithms consistently outperform the ANM algorithm in almost all regimes. Note how unlike Figure \ref{fig:3_node_triangle_entropic_HES}, we do not fix the source to have high-entropy or treat the source differently than the other nodes.}
\label{fig:3_node_triangle_entropic}
\end{center}
\vskip -0.2in
\end{figure*}

\begin{figure*}[ht!]
\vskip 0.2in
\begin{center}
\subfigure[]{\label{fig:5_node_anm100}\includegraphics[width=0.3\textwidth]{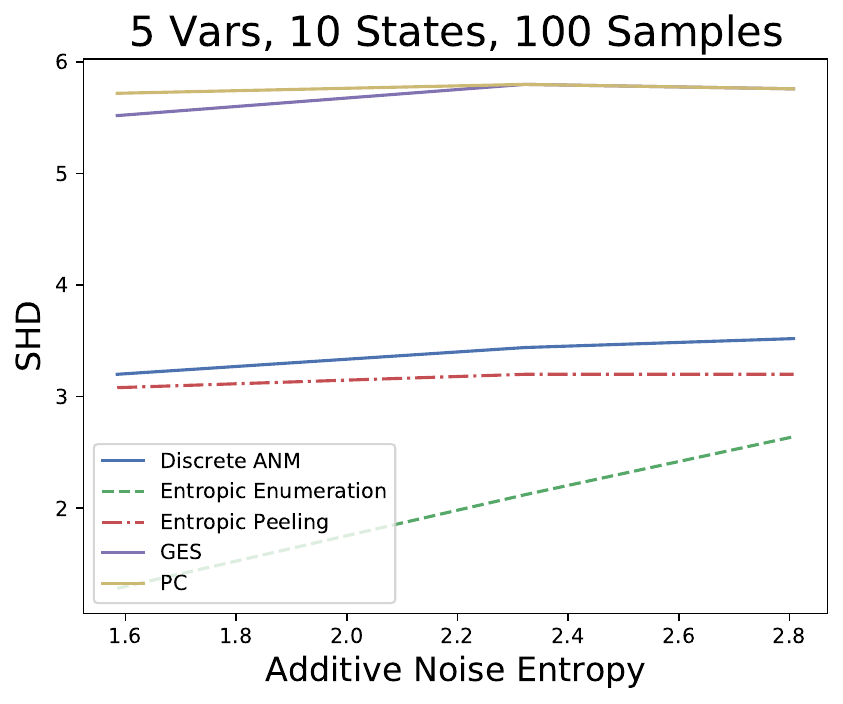}}
\subfigure[]{\label{fig:5_node_anm500}\includegraphics[width=0.3\textwidth]{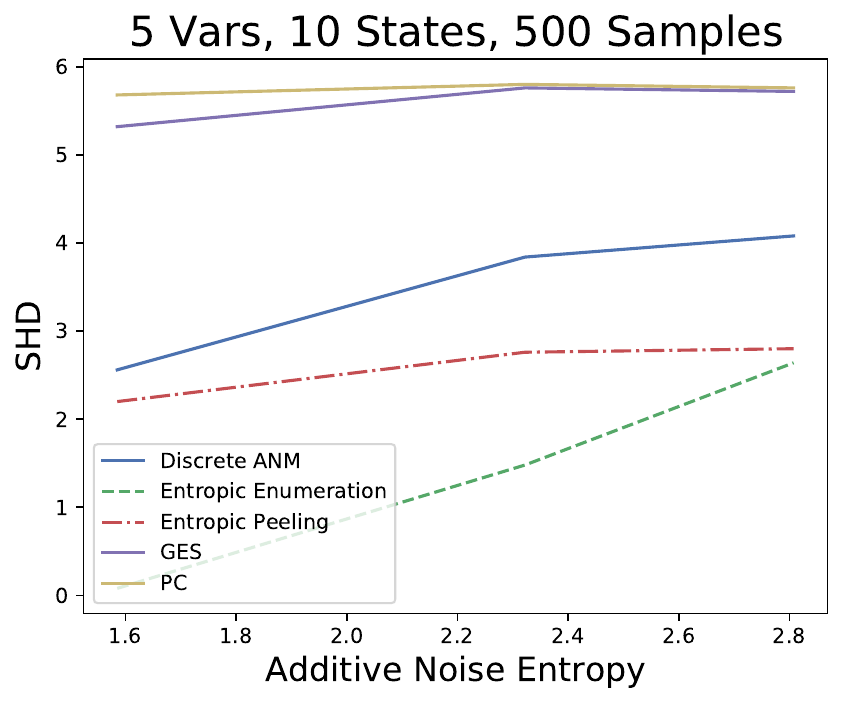}}
\subfigure[]{\label{fig:5_node_anm1000}\includegraphics[width=0.3\textwidth]{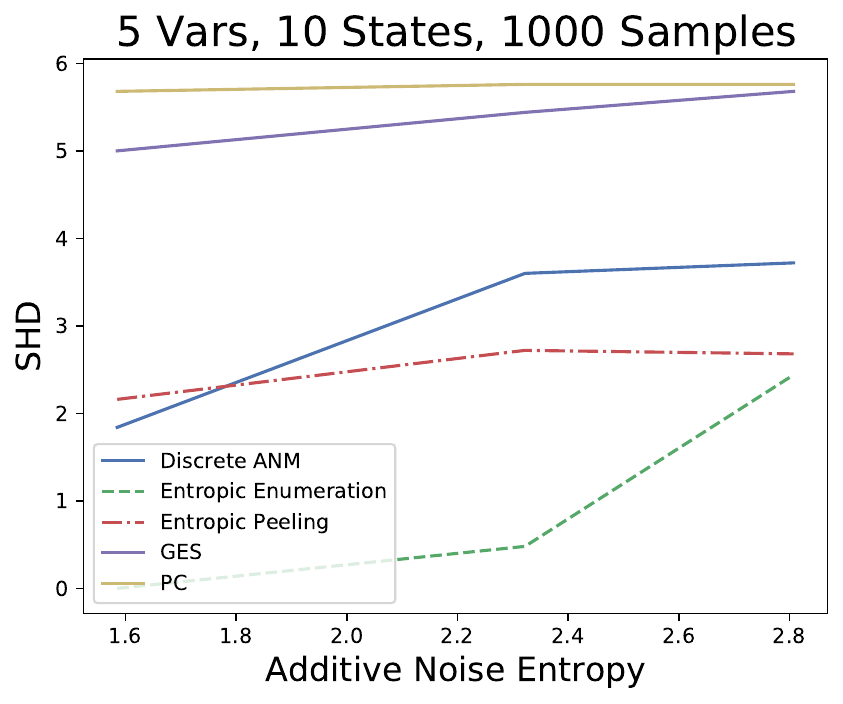}}
\caption{Performance of methods in the ANM setting on random $5-$node graphs: $25$ datasets are sampled for each configuration from the ANM $X=f(\pa_X) + N$. The $x-$axis shows entropy of the exogenous noise. Entropic enumeration outperforms consistently in all regimes.}
\label{fig:5_node_anm}
\end{center}
\vskip -0.2in
\end{figure*}
\begin{figure*}[ht!]
\vskip 0.2in
\begin{center}
\subfigure[]{\label{fig:5_node_entropic100}\includegraphics[width=0.3\textwidth]{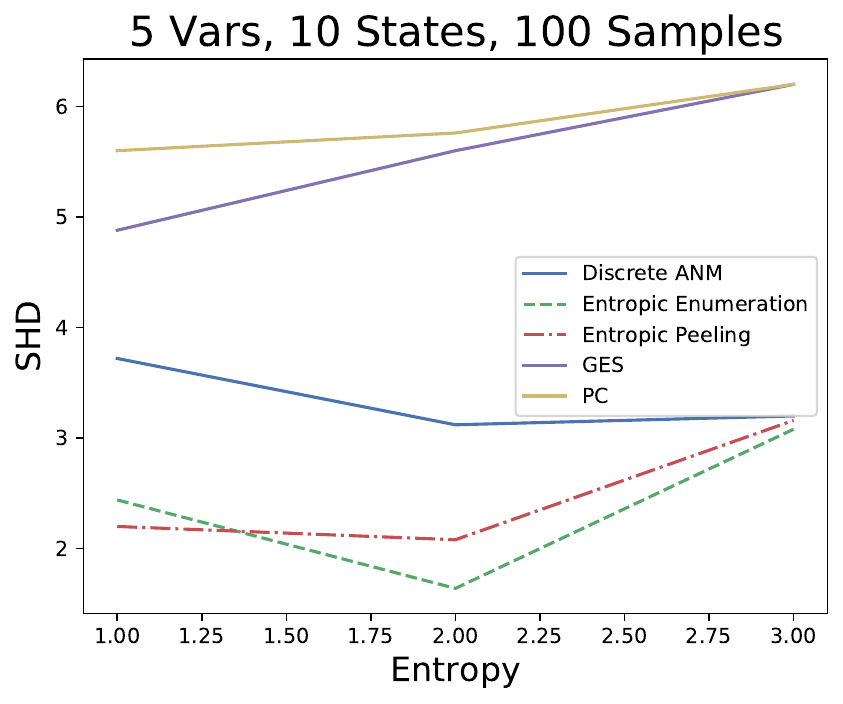}}
\subfigure[]{\label{fig:5_node_entropic500}\includegraphics[width=0.3\textwidth]{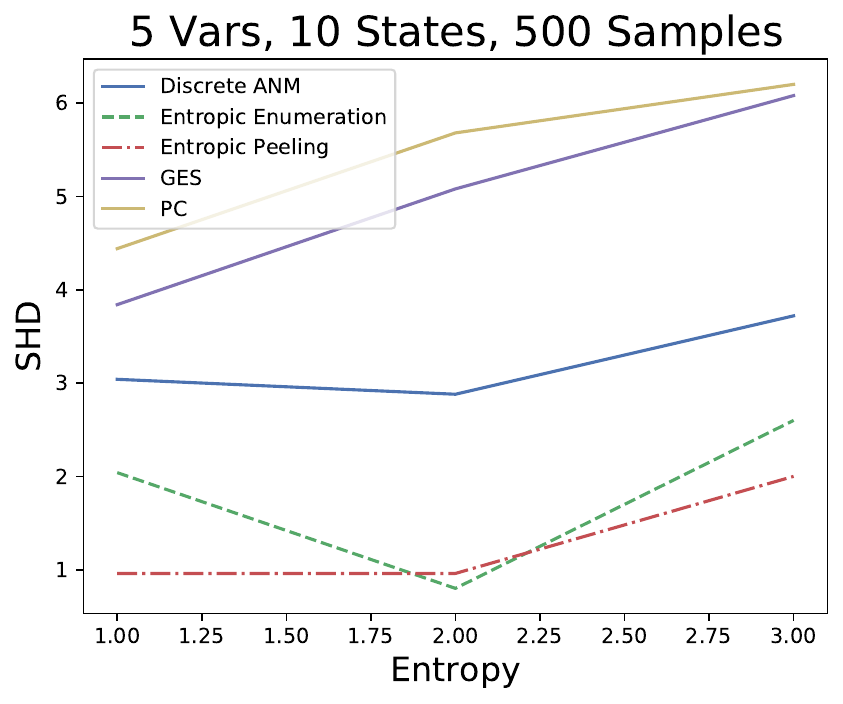}}
\subfigure[]{\label{fig:5_node_entropic1000}\includegraphics[width=0.3\textwidth]{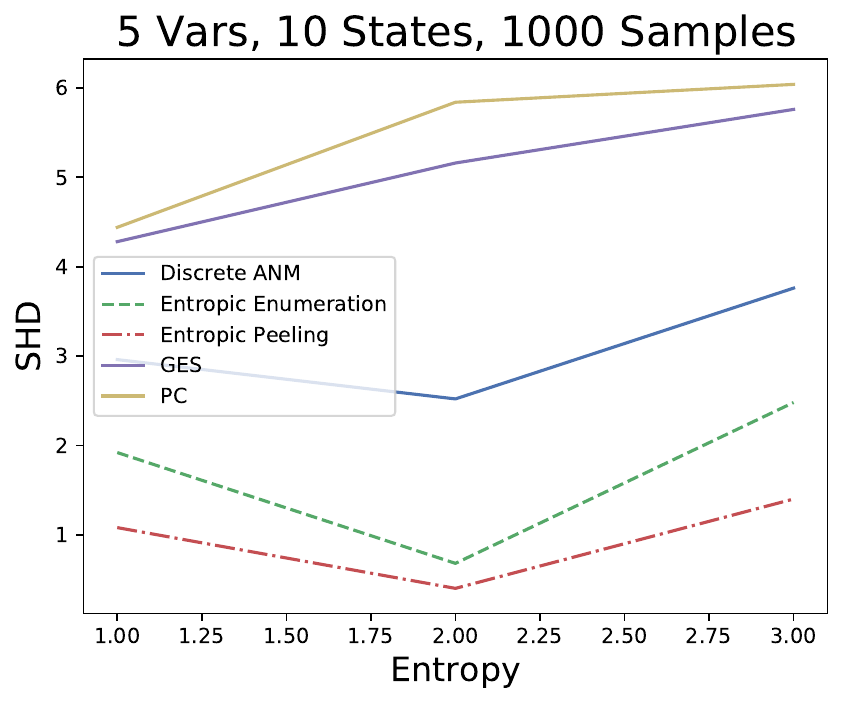}}
\caption{Performance of methods in the unconstrained setting on random $5-$node graphs: $25$ datasets are sampled for each configuration from a random graph from the unconstrained model $X=f(PA_X,N)$. The $x-$axis shows entropy of the additive noise. Entropic enumeration and peeling algorithms consistently outperform the ANM algorithm in all regimes.}
\label{fig:5_node_entropic}
\end{center}
\vskip -0.2in
\end{figure*}

\subsection{Entropy Measure for Peeling Algorithm}
In this section, we compare different versions of peeling algorithm, one that uses only the exogenous entropy and one that uses the total entropy in pairwise comparisons. We randomly sample exogenous distributions according to symmetric Dirichlet distribution, which is characterized by a single parameter $\alpha$. By varying $\alpha$ and with rejection sampling, we are able to generate distributions for the exogenous nodes $E$ such that $H(E)\leq \theta$ for some $\theta$. For each distribution, we then compare the structural Hamming distance of the output of our peeling algorithm with the true graph. The structural Hamming distance (SHD) is the number of edge modifications (insertions, deletions, flips) required to change one graph to another. Results are given in Figure \ref{fig:peeling_synthetic}. Let $H(E)$ and $H(\tilde{E})$ be the minimum exogenous entropy needed to generate $Y$ from $X$ and the minimum exogenous entropy needed to generate $X$ from $Y$, respectively. At each step of the algorithm for every pair $X,Y$, the red curve compares $H(X)$ with $H(Y)$, the blue curve compares $H(E)$ with $H(\tilde{E})$, and the green curve compares $H(X)+H(E)$ with $H(Y)+H(\tilde{E})$ and orients the edge based on the minimum. As expected, comparing exogenous entropies perform better than comparing observed variables' entropies in the low-entropy regime. Interestingly, we observe that comparing total entropies consistently performs much better than either. 
\begin{figure*}[ht!]
\vskip 0.2in
\begin{center}
\subfigure[Bivariate]{\label{fig:peeling_synthetic_2Nodes}\includegraphics[width=0.3\textwidth]{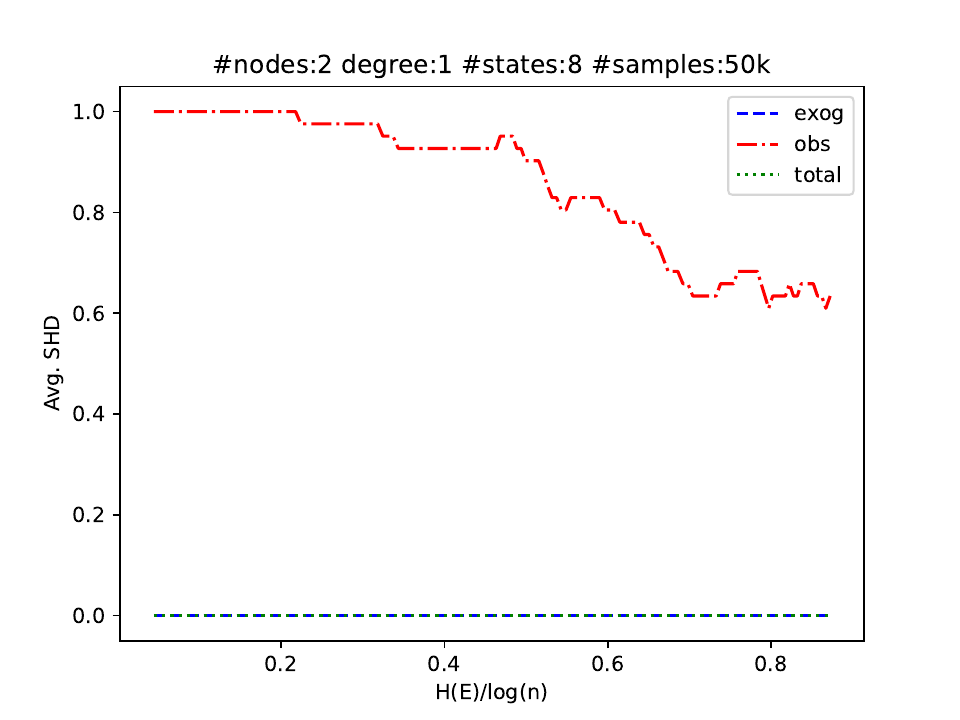}}
\subfigure[$3$-Node Complete Graph]{\label{fig:peeling_synthetic_3Nodes}\includegraphics[width=0.3\textwidth]{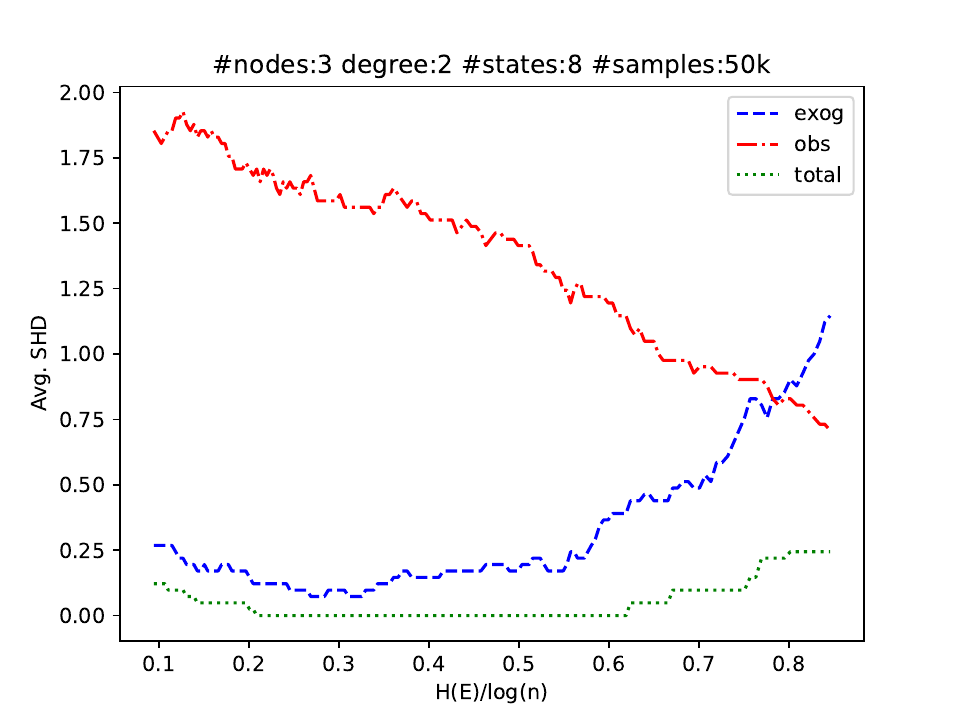}}
\subfigure[$4$-Node Complete Graph]{\label{fig:peeling_synthetic_4Nodes}\includegraphics[width=0.3\textwidth]{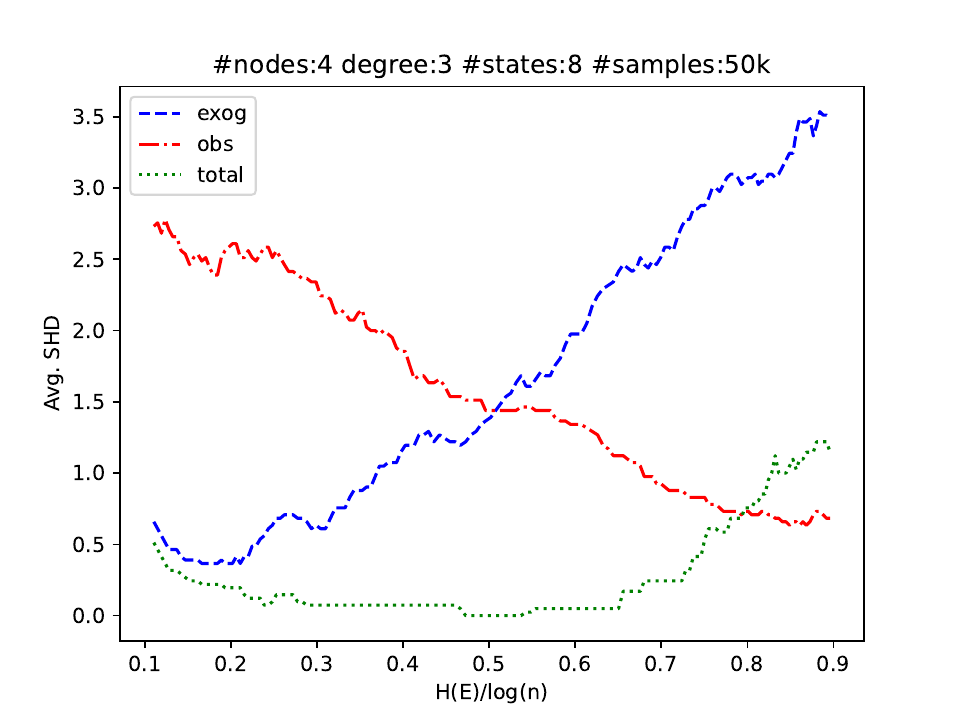}}
\caption{Average structural Hamming distance (SHD) of peeling algorithm on synthetic data for comparing \emph{i)} exogenous entropies (blue, dashed), \emph{ii)} entropies of observed variables (red, dotted-dashed) and \emph{iii) } total entropies (green, dotted) in line $11$ as the Oracle for Algorithm \ref{alg:general}. Randomly orienting all edges would result in an average SHD equal to half the number of edges (0.5, 1.5, and 3.0 for (a), (b), and (c), respectively).}
\label{fig:peeling_synthetic}
\end{center}
\vskip -0.2in
\end{figure*}

\subsection{Entropy Percentile of True Graph}
In this section, We test the hypothesis that \emph{when true exogenous entropies are small, the true causal graph is the DAG that minimizes the total entropy}. To test this, we find the minimum entropy needed to generate the joint distribution for every directed acyclic graph that is consistent with the true graph skeleton. We then look at the percentile of the entropy of the true graph. For example, if there are $5$ DAGs with less entropy than the true graph out of $100$ distinct DAGs, then the percentile is $1-5/100=0.95$.

\subsubsection{Synthetic Data}
Figure \ref{fig:permutation_synthetic} shows the results for various graphs. It can be seen that, especially for dense graphs, the true graph is the unique minimizer of total entropy needed to generate the joint distribution for a very wide range of entropy values. This presents exhaustive search as a practical algorithm for graphs with a small number of nodes (or generally, those for which the MEC is small). %
Our experiments, contrary to our theory, show that even when the number of nodes is the same as the number of states, entropic causality can be used for learning the graph.

\begin{figure*}[ht!]
\vskip 0.2in
\begin{center}
\subfigure[$4$-Node Graphs]{\label{fig:permutation_4Nodes}\includegraphics[width=0.3\textwidth]{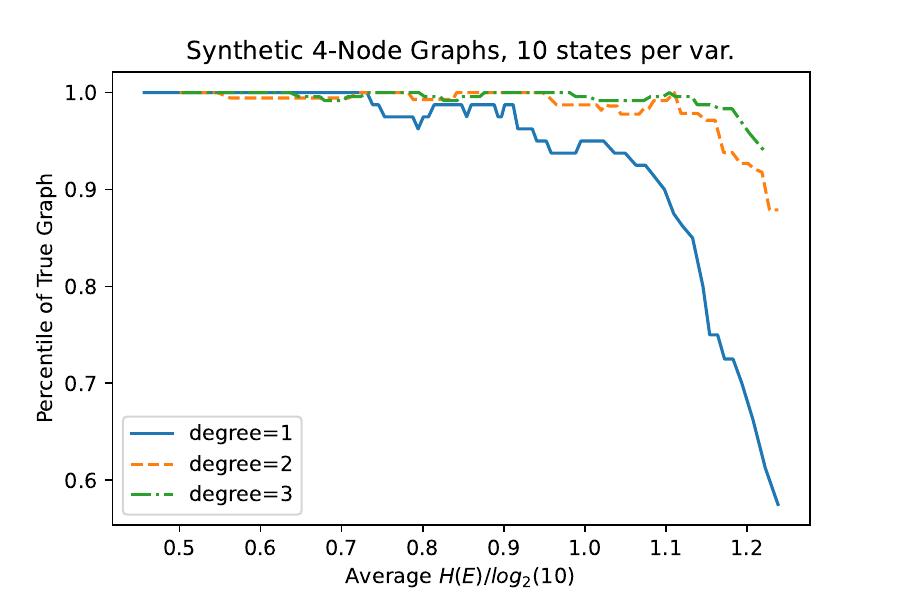}}
\subfigure[$5$-Node Graphs]{\label{fig:permutation_5Nodes}\includegraphics[width=0.3\textwidth]{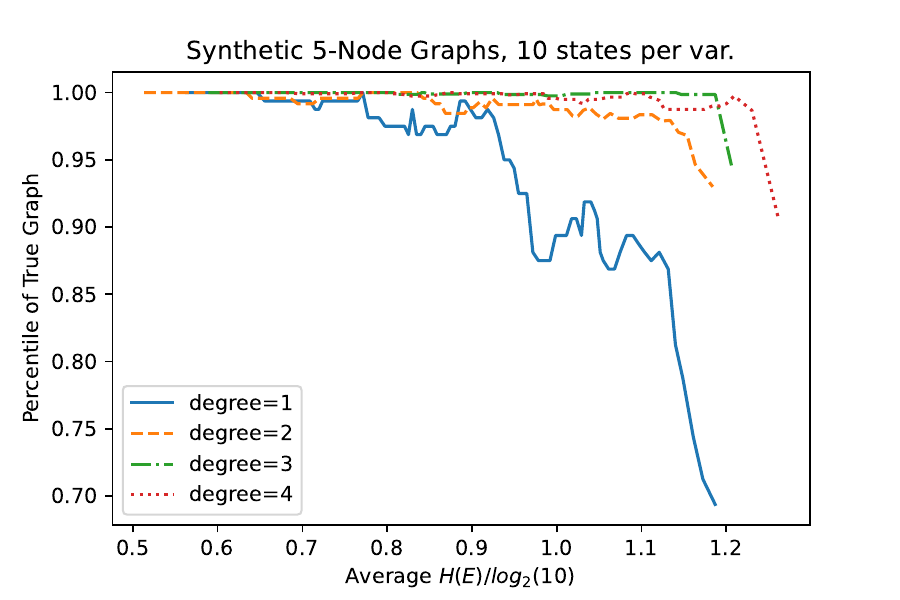}}
\subfigure[$10$-Node Graphs]{\label{fig:permutation_10Nodes}\includegraphics[width=0.3\textwidth]{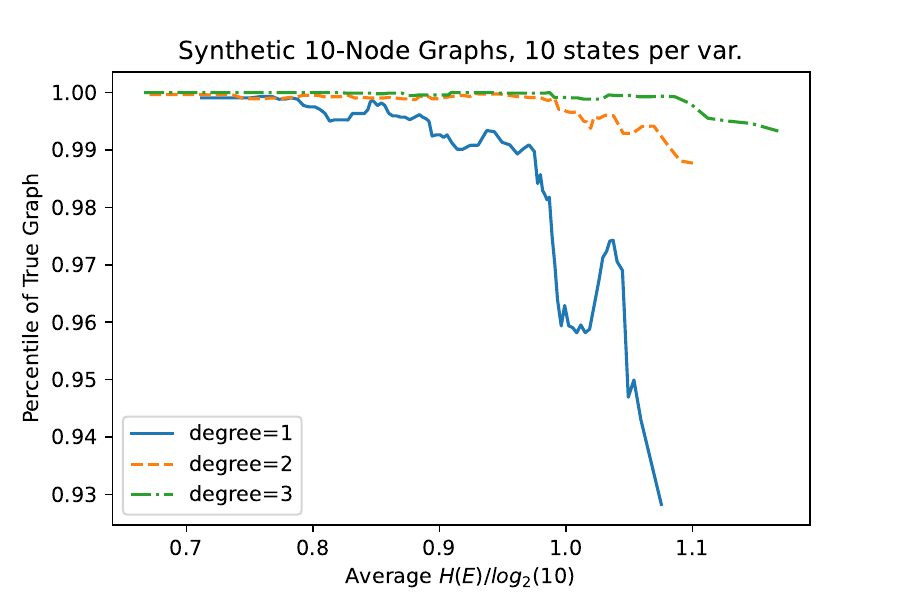}}
\caption{Percentile of the true graph's entropy compared to minimum entropy required to fit every other incorrect possible causal graph that is consistent with the skeleton (synthetic data).}
\label{fig:permutation_synthetic}
\end{center}
\vskip -0.2in
\end{figure*}

\subsubsection{Semi-synthetic Data}
\label{app:semi_synthetic}
\begin{figure*}[h!]
\vskip 0.2in
\begin{center}
\includegraphics[width=0.45\textwidth]{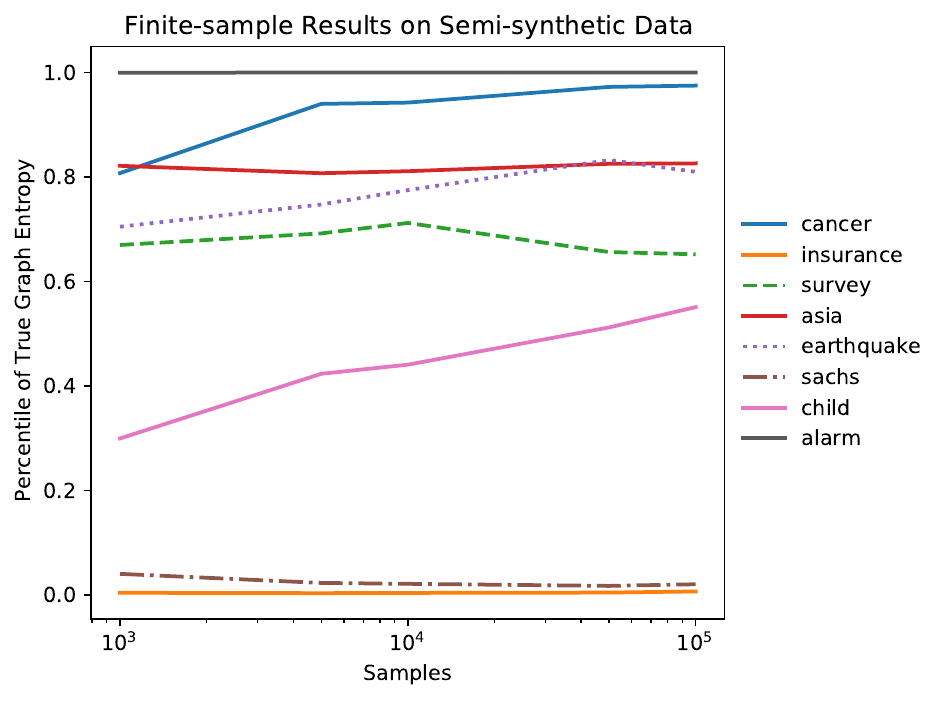}
\caption{Percentile of true graphs entropy compared to minimum entropy required to fit wrong causal graphs in semi-synthetic data from \emph{Bayesian Network Repository}~\cite{scutariLearning10}. }
\label{fig:real_data}
\end{center}
\vskip -0.2in
\end{figure*}

We use the \emph{bnlearn} repository\footnote{\url{https://www.bnlearn.com/bnrepository/}} which contains a selection of Bayesian network models curated from real data~\cite{scutariLearning10}. Each model in bnlearn contains the real causal DAG along with a .bif file that details the conditional distributions
that define the distribution. Using this, we generate some
number of samples (specified per experiment) and evaluate
a collection of causal graph discovery algorithms given the
graph skeleton (which can be learned using, e.g., Greedy
Equivalence Search) and generated samples as input. Using these models, we can generate any number of samples and test the accuracy of our algorithms. The datasets however are often binary which makes them less suitable for Algorithm \ref{alg:general}. We therefore limit our use of this data to test our hypothesis that the true graph has the smallest entropy among all graphs consistent with the skeleton. The results are given in Figure \ref{fig:real_data}. We resample datasets 25 times. For Alarm/Child/Insurance, the graph is too large to try all orientations, so we sample from random topological orderings until we have $10^4-1$ other unique orientations to compare against. Interestingly, for Alarm (46 edges), the true graph is often the entropy minimizer among the sampled orientations. In some cases, it seems the entropy of an orientation contains some real signal. Also interestingly (although not in line with our hypothesis), for Insurance (52 edges), the true graph is often the entropy maximizer among the sampled orientations. A deeper understanding of these phenomena would be interesting.